\documentclass{article}
\usepackage{iclr2027_conference,times}

\usepackage{microtype}
\usepackage{graphicx}
\usepackage{booktabs}
\usepackage{amsmath,amssymb,amsthm,mathtools}
\usepackage{float,array}
\usepackage[hidelinks]{hyperref}
\usepackage{url}

\newtheorem{theorem}{Theorem}[section]
\newtheorem{proposition}[theorem]{Proposition}
\newtheorem{lemma}[theorem]{Lemma}
\newtheorem{corollary}[theorem]{Corollary}
\newtheorem{remark}[theorem]{Remark}
\newtheorem{definition}[theorem]{Definition}
\DeclareMathOperator{\tr}{tr}

\newcommand{\R}{\mathbb{R}}
\newcommand{\W}{\mathbf{W}}
\newcommand{\U}{\mathbf{U}}
\newcommand{\E}{\mathbf{E}}

\newcommand{\supp}{\operatorname{supp}}
\newcommand{\norm}[1]{\left\lVert #1\right\rVert}

\title{Support Selection Beyond Smooth DAG\\
Exactness: Completion Geometry,\\
Score Margins, and Selective Certificates}

\author{
Rui Wu \qquad Zongyuan Chen \qquad
Hong Xie\thanks{Corresponding author: hongx87@ustc.edu.cn} \qquad
Defu Lian \qquad Enhong Chen\\
School of Computer Science and Engineering,\\
University of Science and Technology of China\\
\texttt{\{wurui22,chenzongyuan\}@mail.ustc.edu.cn}\\
\texttt{\{hongx87,liandefu,cheneh\}@ustc.edu.cn}
}

\iclrfinalcopy

\begin{document}
\maketitle
\pagestyle{plain}

\begin{abstract}
Smooth acyclicity constraints answer whether a weighted support is a DAG.
Structure learning asks a different question: which support change should be
made?  Existing analyses establish degeneracy for particular constraint
formulas, but do not isolate what follows from smooth exactness itself.  At a
DAG boundary, we show that minimal cycle completions generate a squarefree
monomial ideal containing every restricted Taylor jet of an exact
representation.  If the smallest completion has \(q\) edges, the first
possible response has order \(q\) for a vector residual and \(2q\) for a
nonnegative scalar.  Exponentially many constant-scale cyclic manifolds exhibit
the same lack of ranking away from the boundary for NOTEARS and DAGMA.
To connect this representation limit to optimization, we derive the exact
selection time for an isolated cycle.  When \(\Psi'(h)\asymp h^\nu\), the
feasibility-only time is
\(T_0(\varepsilon)=\Theta(\varepsilon^{-(2\nu+1)})\).  A score margin changes
the leading dynamics at scale \(T_0^{-1}\) for \(\nu>0\).  The endpoint
\(\nu=0\) has a logarithmic boundary layer: the unperturbed limit requires
\(\gamma T_0\log(1/\varepsilon)\to0\).  Controlled experiments check the law,
and a separation statistic computed without the generating graph predicts
selection time on 320 official NOTEARS/DAGMA trajectories (Spearman \(-0.52\)
and \(-0.66\), permutation \(p<10^{-4}\)).
The score margin is unknown in finite samples.  A parent-set confidence family
and forced-opposite queries certify skeleton and unshielded-collider labels
shared by every population optimum of a frozen score.  In a 320-run audit with
four frontends, four graph families, and four SEM regimes, every regret bound
covers an independent oracle-score audit.  None of 3,042 certified skeleton or
2,396 collider labels disagrees with the oracle-score optimum, although
\(4.4\%\) and \(5.5\%\), respectively, disagree with the generating graph.
The results separate three claims often conflated in continuous DAG learning:
DAG feasibility, score-based support selection, and causal identification.
\end{abstract}

\section{Introduction}

A weighted matrix represents a directed graph through its zero pattern.
Edge magnitudes determine statistical fit, whereas acyclicity changes only
when an entry crosses zero.  Continuous DAG learning must therefore make a
discrete support decision while optimizing over a Euclidean space.  A typical
estimator combines
\begin{equation}
\min_{\W}\;
\mathcal L(\W;X)+\lambda\,\mathcal R(\W)+\Psi(h(\W)),
\label{eq:joint-objective}
\end{equation}
where \(\mathcal L\) measures fit, \(\mathcal R\) promotes sparsity, and
\(h\) enforces acyclicity.  A score may rank competing graphs, and a
nonsmooth update or threshold may create zeros.  Exactness of \(h\), however,
only identifies the acyclic supports.  We ask what support-ranking information,
if any, follows from exactness itself.

NOTEARS supplies the smooth equality
\begin{equation}
h_{\exp}(\W)=\tr\!\left(\exp(\W\circ\W)\right)-d=0,
\label{eq:notears}
\end{equation}
where $\W\in\R^{d\times d}$ is a weighted adjacency matrix
\citep{zheng2018dags}.  Polynomial, log-determinant, and spectral alternatives
change the geometry of the optimization problem
\citep{yu2019dag,bello2022dagma,nazaret2024stable,zhang2025analytic}.  Their
far-field landscapes differ substantially, but their exactness statements make
the same promise: \(h(\W)=0\) characterizes DAG support.  That promise does
not imply that \(\nabla h\) distinguishes two deletions that both restore
feasibility.  Rescaling a tied or absent local signal cannot supply the missing
ranking.

The obstruction comes from support geometry.  At a DAG boundary, suppose a set \(F\) of
absent edges creates a cycle but no proper subset does.  A smooth exact
constraint vanishes on every coordinate face obtained by omitting one edge of
\(F\).  Those faces forbid low-order Taylor terms according to the number of
edges needed to complete a cycle.  Conditioning can change coefficients and
far-field behavior, but not a term excluded by this support geometry.

The boundary calculation alone does not describe ordinary-sized iterates.
Relabeling symmetry supplies the missing link: it creates
constant-weight cyclic manifolds on which the feasibility gradient ties
competing deletions.  A perturbation of size \(\varepsilon\) breaks the tie,
but the time to reach a fixed edge ratio diverges as \(\varepsilon^{-1}\) for
a linear feasibility force and \(\varepsilon^{-3}\) for a cold quadratic
penalty.  These are worst-case conditioning statements, not a claim that
observational data are usually symmetric.  They raise the statistical question
that closes the paper: when the score supplies the ranking, do finite data
resolve it?

The paper develops this argument in three steps.
\begin{itemize}
    \item Earlier degeneracy results are tied to a chosen entrywise power.  We
    define the cycle-completion ideal and its initial degree \(q_\W(F)\), which
    depend on how several absent edges jointly close a cycle.  Every restricted
    Taylor jet of a smooth exact representation lies in this ideal, forcing
    response order at least \(q\) for vector residuals and \(2q\) for
    nonnegative scalars.
    \item A local Taylor barrier need not govern finite-scale optimization.  We
    close this gap with \(2^{0.332d}\) constant-scale cyclic manifolds in the
    NOTEARS and DAGMA domains and an exact isolated-cycle hitting-time law.  The
    law identifies the score scale that changes support selection and exposes
    a logarithmic transition layer at \(\nu=0\).  Controlled flows and 320
    official trajectories test these predictions.
    \item Feasibility does not reveal whether data resolve the resulting score
    margin.  We construct a parent-set confidence family whose forced-opposite
    queries certify skeleton and collider labels shared by all population
    optima of a frozen score.  The audit reports both oracle-score agreement and
    disagreement with the generating graph, separating statistical selection
    from causal identification.
\end{itemize}

No Fears studies degeneracy for positive Hadamard-power constraints
\citep{wei2020fears}; the present order is determined instead by the joint
completion statistic \(q_\W(F)\).  DAGMA improves the far-field path
\citep{bello2022dagma}, while remaining subject to the local representation
limit.
Acyclic parameterizations, boundary-constrained formulations, and discrete
graph operations lie outside our assumptions and can break the symmetry
\citep{yu2021nocurl,massidda2024cosmo,gillot2022large,rey2026nonnegative}.
The first two results are therefore not global iteration lower bounds for
Eq.~\eqref{eq:joint-objective}.  The certificate is likewise conditional on a
training-frozen candidate family and bounded decomposable score; it certifies
score selection, not causal identification.  Appendix~\ref{app:closest-work}
gives a claim-by-claim comparison and the precise boundaries.

\section{Setup and Scope}

Let $\mathcal Z=\{\W\in\R^{d\times d}:\operatorname{diag}(\W)=0\}$, and let
$\mathcal D\subset\mathcal Z$ be the matrices whose nonzero support is a DAG.
We work on an open domain $\Omega\subseteq\mathcal Z$ containing the candidate
subspaces of interest; a proper domain accommodates log-determinant
constraints.

\begin{definition}[Exact representations]
A smooth map $H:\Omega\to\R^r$ is a \emph{signed/vector exact representation} if, for every $\W\in\Omega$,
\begin{equation}
H(\W)=0\quad\Longleftrightarrow\quad \W\in\mathcal D.
\label{eq:vector-exact}
\end{equation}
A smooth scalar $h:\Omega\to\R$ is \emph{nonnegative exact} if, for every $\W\in\Omega$,
\begin{equation}
h(\W)\ge0,\qquad h(\W)=0\quad\Longleftrightarrow\quad \W\in\mathcal D.
\label{eq:scalar-exact}
\end{equation}
\end{definition}

The scalar definition includes EXP, trace polynomials, and DAGMA on its domain;
the vector definition also permits signed systems.  Nonsmooth functions,
discrete order searches, acyclic parameterizations, and boundary-constrained
domains fall outside these assumptions.

A pathwise quantity has \emph{exact order $p$} if
$R(t)=|t|^pv+o(|t|^p)$ for nonzero $v$.  Our barriers lower-bound the first
possible order; exactness need not attain it.  Likelihood choice and sparsity
regularization affect which member of a Markov equivalence class a
differentiable program selects \citep{deng2024markov,jin2026revisiting};
the next four sections isolate how feasibility enters a local optimizer before
Section~\ref{sec:completion-margin} returns to finite-data score selection.

\section{Cycle-Completion Complexity}
\label{sec:completion-geometry}

Fix a DAG $\W\in\mathcal D$ and a finite set
$F=\{e_1,\ldots,e_m\}$ of absent off-diagonal coordinates.  Let $\E_s$ be a signed coordinate matrix at $e_s$.  For $G\in\{H,h\}$ define the local restriction
\begin{equation}
V_F=\left\{x\in\R^m:\W+\sum_{s=1}^m x_s\E_s\in\Omega\right\},
\qquad
G_F(x)=G\!\left(\W+\sum_{s=1}^m x_s\E_s\right),\quad x\in V_F.
\label{eq:restriction}
\end{equation}
Because $\Omega$ is open, $V_F$ is an open neighborhood of the origin.

\begin{definition}[Completion number]
\label{def:q}
The candidate-subspace completion number is
\begin{equation}
q_\W(F)=\min\left\{|S|:S\subseteq F,\ \supp(\W)\cup S\text{ is cyclic}\right\},
\label{eq:q}
\end{equation}
with $q_\W(F)=\infty$ if no subset completes a cycle.
\end{definition}

Geometrically, every coordinate face with fewer than $q_\W(F)$ candidates lies
in the exact zero set; only after $q_\W(F)$ coordinates are active can the
subspace leave the acyclic set.  At $\W=0$, this is the directed girth of
$(V,F)$.  The definition also covers overlapping cycles and nonempty base DAGs.

Candidate edges need not approach zero at the same rate.  For positive weights
$a=(a_1,\ldots,a_m)\in\R_{>0}^m$, define the weighted analogue
\begin{equation}
\tau_\W(F;a)=\min_{\substack{S\subseteq F:\ 
\supp(\W)\cup S\text{ cyclic}}}\sum_{e_s\in S}a_s.
\label{eq:tau}
\end{equation}
Along $x_s=c_st^{a_s}$ as $t\downarrow0$, $\tau_\W(F;a)$ is the cheapest
cycle-completion exponent.  At the empty graph it is weighted directed girth.

\begin{definition}[Cycle-completion ideal]
\label{def:completion-ideal}
The acyclic candidate supports form a restricted directed-subgraph complex
\citep{hultman2004directed},
\begin{equation}
\Delta_{\W,F}=\{S\subseteq[m]:\supp(\W)\cup\{e_s:s\in S\}
\text{ is acyclic}\}.
\end{equation}
Let $\mathcal C_{\rm min}(\W,F)$ collect its inclusion-minimal nonfaces.
Its Stanley--Reisner ideal is
\begin{equation}
I_{\W,F}=\left\langle x^C:C\in\mathcal C_{\rm min}(\W,F)\right\rangle
\subset\R[x_1,\ldots,x_m],
\label{eq:completion-ideal}
\end{equation}
where $x^C=\prod_{s\in C}x_s$.
\end{definition}

Thus $q_\W(F)$ and $\tau_\W(F;a)$ are the least ordinary and weighted degrees
in $I_{\W,F}$.  When this ideal is nonzero, call
$r_\W(F)=\max_{C\in\mathcal C_{\rm min}(\W,F)}|C|$ its \emph{completion
width}.  The two unweighted degrees can differ: $q_\W(F)$ controls the
earliest Taylor signal, while $r_\W(F)$ controls the worst local error bound.

\begin{theorem}[Cycle-completion jet ideal]
\label{thm:jet-ideal}
Let $G_F$ be the restriction of either exact representation in
Eqs.~\eqref{eq:vector-exact}--\eqref{eq:scalar-exact}, and suppose $G_F$ is
$C^k$.  Every component of its order-$k$ Taylor polynomial at zero belongs to
$I_{\W,F}$.  Conversely, the vector of minimal generators
\begin{equation}
M_{\W,F}(x)=(x^C)_{C\in\mathcal C_{\rm min}(\W,F)}
\label{eq:ideal-generator-vector}
\end{equation}
vanishes exactly on the acyclic candidate supports.
\end{theorem}

\begin{proof}
If $\supp(\alpha)\in\Delta_{\W,F}$, exactness makes $G_F$ identically zero
near the origin on that coordinate subspace, so $D^\alpha G_F(0)=0$.
Every surviving Taylor monomial therefore contains a minimal nonface and
belongs to $I_{\W,F}$.  The generator vector is zero exactly when the active
support contains no minimal nonface, which is precisely membership in
$\Delta_{\W,F}$.
\end{proof}

\section{Signed and Vector Representations}

The coordinate-face observation already constrains a vector representation.
No sign or nonnegativity assumption is needed: exactness alone forces the
low-order Taylor coefficients to disappear.

\begin{theorem}[Vector cycle-completion barrier]
\label{thm:vector}
Assume Eq.~\eqref{eq:vector-exact}, let $q=q_\W(F)<\infty$, and suppose $H$ is $C^{q-1}$ near $\W$.  Then
\begin{equation}
D^\alpha H_F(0)=0
\qquad\text{for every multi-index }|\alpha|\le q-1.
\label{eq:vector-jet}
\end{equation}
If $H$ is $C^q$, then for every $\U$ supported on $F$,
\begin{equation}
\norm{H(\W+t\U)}_2=O(|t|^q),
\qquad
\norm{D H_F(tu)}_{\rm op}=O(|t|^{q-1}).
\label{eq:vector-order}
\end{equation}
\end{theorem}

\begin{proof}
Consider a Taylor coefficient indexed by $\alpha$ and let $S=\supp(\alpha)$.  If $|\alpha|<q$, then $|S|<q$ and $\supp(\W)\cup S$ is acyclic.  Exactness makes $H_F$ identically zero near the origin in the coordinate subspace indexed by $S$, so the coefficient must vanish.  Taylor remainder bounds give Eq.~\eqref{eq:vector-order}.
\end{proof}

The restriction in Eq.~\eqref{eq:vector-order} matters.  The full Jacobian may contain directions outside $F$ that close an already existing path at lower order; the theorem controls the chosen candidate subspace.

\begin{corollary}[Weighted vector barrier]
\label{cor:weighted-vector}
Let $a\in\mathbb N^m$, $\tau=\tau_\W(F;a)<\infty$, $c_s\ne0$, and
$\W_a(t)=\W+\sum_s c_st^{a_s}\E_s$.  If $H$ is $C^{\tau-1}$, every derivative of $H(\W_a(t))$ below order $\tau$ vanishes at zero.  If $H$ is $C^\tau$, then
\begin{equation}
\norm{H(\W_a(t))}_2=O(|t|^\tau).
\end{equation}
\end{corollary}

The proof lifts edge $e_s$ to $a_s$ auxiliary coordinates whose product activates the edge, then applies the local coordinate-subspace jet lemma underlying Theorem~\ref{thm:vector}; Appendix~\ref{app:weighted} gives details.

The bound is sharp.  To see this directly, let $\mathcal C_d$ be the set of all
simple directed cycles and define
\begin{equation}
H_{\rm cyc}(\W)=\left(\prod_{(i,j)\in C}W_{ij}\right)_{C\in\mathcal C_d}.
\label{eq:cycle-vector}
\end{equation}
This polynomial vector is zero exactly on DAG support.  A minimal $q$-edge completion produces a component of order exactly $q$, and a minimum-weight completion produces order $\tau$.  Its representation size is
\begin{equation}
|\mathcal C_d|=\sum_{k=2}^d \binom{d}{k}(k-1)!.
\label{eq:cycle-count}
\end{equation}
Equation~\eqref{eq:cycle-vector} gives an explicit sharp representation, but
Theorem~\ref{thm:vector} does not require a representation to enumerate cycles.

Theorem~\ref{thm:jet-ideal} is the finite-jet form of the classical
Stanley--Reisner correspondence \citep{miller2005combinatorial}.
The same geometry also limits succinct vector residuals: uniform
$q$th-order sensitivity on an explicit channel family requires
$\Theta(d^2)$ outputs for fixed $q$.  This dimension statement concerns
uniform directional sharpness, not convergence of an arbitrary
lower-dimensional map; Appendix~\ref{sec:structural-consequences} states and
proves the near-tight bound.

\section{Nonnegative Scalar Representations}

The scalar case contains one additional restriction.  A vector-valued leading
term may change sign, whereas the leading Taylor polynomial of a nonnegative
scalar must itself be nonnegative.  This parity requirement doubles the first
possible order.

\begin{lemma}[Newton vertices]
\label{lem:newton}
Every vertex exponent of the Newton polytope of a globally nonnegative real polynomial is coordinatewise even, and its coefficient is positive \citep{reznick1978extremal}.
\end{lemma}

\begin{theorem}[Nonnegative scalar barrier]
\label{thm:scalar}
Assume Eq.~\eqref{eq:scalar-exact}, let $q=q_\W(F)<\infty$, and suppose $h$ is $C^{2q-1}$ near $\W$.  Then
\begin{equation}
D^\alpha h_F(0)=0\qquad\text{for every }|\alpha|\le2q-1.
\label{eq:scalar-jet}
\end{equation}
If $h$ is $C^{2q}$, then
\begin{equation}
h(\W+t\U)=O(|t|^{2q}),
\qquad \norm{D h_F(tu)}=O(|t|^{2q-1}).
\label{eq:scalar-order}
\end{equation}
\end{theorem}

\begin{proof}
Suppose the first nonzero homogeneous Taylor term $P_r$ has degree $r<2q$.  Nonnegativity of $h$ makes $P_r$ globally nonnegative.  Exactness on acyclic coordinate subspaces implies that every monomial of $P_r$ uses at least $q$ distinct candidates.  Choose a vertex exponent $\alpha$ of its Newton polytope.  Lemma~\ref{lem:newton} makes every nonzero $\alpha_s$ at least two, so
$r=|\alpha|\ge2|\supp(\alpha)|\ge2q$, a contradiction.
\end{proof}

\begin{corollary}[Newton-square geometry]
\label{cor:newton-square}
Let $P_r$ be the first nonzero homogeneous Taylor form of $h_F$.  Every
vertex exponent $\alpha$ of its Newton polytope satisfies
$x^\alpha\in I_{\W,F}^2$.  Hence, for every positive weight vector $a$,
the minimum $a$-weight at a Newton vertex is at least
$2\tau_\W(F;a)$.
\end{corollary}

\begin{proof}
Theorem~\ref{thm:jet-ideal} gives $P_r\in I_{\W,F}$.
Lemma~\ref{lem:newton} makes $\alpha$ even, while
$\supp(\alpha/2)=\supp(\alpha)$ is a nonface.  Thus
$x^{\alpha/2}\in I_{\W,F}$ and $x^\alpha\in I_{\W,F}^2$.
\end{proof}

For $a\in\mathbb N^m$, the same auxiliary-coordinate argument gives a weighted
$2\tau_\W(F;a)$ barrier with the required regularity made explicit.  If $h$ is
$C^{2\tau-1}$ near $\W$, every derivative of $h(\W_a(t))$ below order $2\tau$
vanishes; if $h$ is $C^{2\tau}$, then
$h(\W_a(t))=O(|t|^{2\tau})$.  For one minimal completion, a stronger local
factorization holds:
\begin{equation}
h_F(x)=\left(\prod_{s=1}^q x_s^2\right)\Psi(x),
\qquad \Psi(x)\ge0\quad\text{locally},
\label{eq:factorization}
\end{equation}
under $C^{2q}$ regularity.  The positive walk families used by standard
continuous methods attain this lower bound.

\begin{corollary}[Local order of positive walk constraints]
\label{cor:positive-walk-order}
Let $F$ be a minimal $q$-edge completion at a DAG $W$ and let $U$ activate
every edge in $F$ with a nonzero equal-scale coefficient.  Then NOTEARS EXP,
the usual positive trace polynomial, and DAGMA's log determinant while
$\rho(W\circ W)<s$ satisfy
\begin{equation}
h(W+tU)=\Theta(|t|^{2q}),\qquad
\norm{D h_F(tu)}_2=\Theta(|t|^{2q-1}).
\label{eq:positive-walk-order}
\end{equation}
The same conclusion holds for every convergent positive walk sum with a
positive coefficient at each simple-cycle length, and for smooth entrywise
squashes with a positive quadratic leading term.
\end{corollary}

The proof is Proposition~\ref{prop:walk-sharp} in Appendix~\ref{app:scalar}.
For DAGMA, $-\log\det(sI-A)+d\log s=\sum_{k\ge1}
\tr(A^k)/(ks^k)$ inside its M-matrix domain.  The log-determinant can therefore
improve coefficients and the landscape away from the boundary without
changing the topology-controlled order at the boundary itself.

\paragraph{Consequences for optimization.}
The corollary does not say that all smooth constraints optimize equally well.
Coefficient size and conditioning still matter, and the log-determinant path
used by DAGMA can improve both.  It says only that these changes do not remove
the local support-boundary order.  In practice, exact support may instead be
selected after optimization by thresholding
\citep{ng2024sober}, or enforced through absolute-value reformulations,
higher-order local search, acyclic parameterizations, discrete graph
operations, or boundary-constrained domains
\citep{wei2020fears,shridharan2025beta,yu2021nocurl,massidda2024cosmo,gillot2022large,rey2026nonnegative}.
Appendix~\ref{app:closest-work} compares these assumptions, and
Appendix~\ref{app:scalar} gives the scalar proofs.

\section{From Constraint Symmetry to the Score--Topology Crossover}
\label{sec:universality}

The far-field construction uses a small ordinary-coordinate gadget.  For
each $i\in[m]$, introduce nodes $a_i,b_i$, a candidate edge
$a_i\to b_i$, and fixed cross-edges $b_i\to a_j$ for every $i\ne j$.
The fixed graph is acyclic, while any two active candidates form a directed
four-cycle.  We take $t$ disjoint copies.  Appendix
\ref{sec:structural-consequences} shows that the same reduction has a broader
topological consequence: local candidate sections can realize any finite
homotopy type and illustrates the construction.  That result is not needed
for the selection theorem below.

Call $h$ \emph{square-symmetric} if
\begin{equation}
h(\W)=\Phi(\W\circ\W),
\qquad h(P\W P^\top)=h(\W)
\label{eq:square-symmetry}
\end{equation}
for every permutation matrix $P$.  NOTEARS, positive trace polynomials, and
DAGMA are square-symmetric on their domains.

\begin{theorem}[Far-field support-selection blindness]
\label{thm:symmetry-blindness}
For every integer $m\ge2$ and integer $t\ge1$, there is a base DAG on
$d=2mt$ nodes with $mt$ independent candidate entries and
\begin{equation}
\binom{m}{2}^{t}
\label{eq:blind-manifold-count}
\end{equation}
disjoint positive cyclic support manifolds.  Let $h:\Omega\to\R$ be a $C^1$
square-symmetric exact scalar, and suppose that $\Omega$ contains a
neighborhood of these manifolds.  On each manifold,
\begin{equation}
\nabla_{x_b}h(\W_0+x)\in\operatorname{span}\{x_b\},
\qquad b=1,\ldots,t,
\label{eq:block-radial-gradient}
\end{equation}
so the constraint gives zero first-order signal in every within-block
edge-redistribution direction.

If $\nabla h$ is locally Lipschitz, every scalarized flow
$\dot x=-\omega(x)\nabla h(\W_0+x)$ with locally Lipschitz scalar $\omega$,
when initialized on one of these manifolds, preserves it on the solution's
maximal interval of existence.  Positive active coefficients cannot reach
zero at a finite time in that interval, so the support remains cyclic whenever
the solution exists.  With $m=5$ there are
$10^t=2^{0.332\ldots d}$ such manifolds, and every active coefficient may be
fixed at $1/2$.  The resulting manifolds lie in the domains of NOTEARS and
DAGMA with $s=1$.
\end{theorem}

An automorphism ties the two active derivatives, while evenness zeros every
inactive derivative.  The flow claim then follows from uniqueness and
Gronwall's inequality.

\begin{proposition}[Full-cycle scalarization and score perturbation]
\label{prop:symmetry-time}
Let \(W(z)\) be supported on an isolated directed \(L\)-cycle with positive
edge weights \(z=(z_1,\ldots,z_L)\).  Suppose
\[
h(W(z))=\phi(p),\qquad p=\prod_{i=1}^Lz_i^2,\qquad
\phi(0)=0,\quad\phi'(0)>0,
\]
where \(\phi\) is \(C^1\) and \(\phi'>0\) on the relevant range.  Let
\(\Psi\) be \(C^1\), with \(\Psi'>0\) on that range and
\(\Psi'(s)=c s^\nu(1+o(1))\) as \(s\downarrow0\), where
\(c>0\) and \(\nu\ge0\).  Under
\(\dot z=-\nabla_z(\Psi\circ h)\), initialize two tracked edges as
\((x_0,y_0)=(a+\varepsilon,a-\varepsilon)\) and initialize the remaining
edges at \(b_j>a\).  Put
\[
C=\prod_{j=1}^{L-2}(b_j^2-a^2),\qquad
I_\nu(r)=\int_{r^2/(1-r^2)}^\infty
          [u(u+1)]^{-(\nu+1)}\,du .
\]
For fixed \(0<r<1\), let \(T_r(\varepsilon)\) be the first time \(y/x=r\).
Then every difference \(z_i^2-z_j^2\) is conserved and
\begin{equation}
T_r(\varepsilon)\sim
\frac{I_\nu(r)}
{4c[\phi'(0)C]^{\nu+1}(4a\varepsilon)^{2\nu+1}}.
\label{eq:symmetry-time}
\end{equation}
For a separate flow with a smooth score \(S\), allow arbitrary positive
tracked initial values with \(y_0/x_0>r\), and define the total logarithmic
deletion margin
\[
m_{\rm sc}(z)=
\frac{\partial_y[S+\Psi(h)]}{y}
-\frac{\partial_x[S+\Psi(h)]}{x}.
\]
If \(m_{\rm sc}(z)\ge\gamma_{\rm flow}>0\) until \(y/x=r\), then
\begin{equation}
T_r(\varepsilon)\le
\frac1{\gamma_{\rm flow}}\log\frac{y_0}{r x_0},
\label{eq:score-time}
\end{equation}
so a fixed score margin removes the divergence and can overturn an
\(O(\varepsilon)\) reversed initial ranking.  More precisely, add the smooth
edge score \(S_\gamma(z)=\gamma y^2/2\), let
\(T_r(\varepsilon,\gamma)\) be the new hitting time, and write
\(T_0(\varepsilon)=T_r(\varepsilon,0)\).  For any positive schedule
\(\gamma=\gamma(\varepsilon)\),
\begin{equation}
\left.
\begin{array}{ll}
\gamma T_0(\varepsilon)\to0, & \nu>0,\\
\gamma T_0(\varepsilon)\log(1/\varepsilon)\to0, & \nu=0
\end{array}
\right\}
\quad\Longrightarrow\quad
\frac{T_r(\varepsilon,\gamma)}{T_0(\varepsilon)}\to1,
\label{eq:topology-limited-regime}
\end{equation}
whereas, for every \(\nu\ge0\),
\begin{equation}
\gamma T_0(\varepsilon)\to\infty
\quad\Longrightarrow\quad
\frac{T_r(\varepsilon,\gamma)}{T_0(\varepsilon)}\to0.
\label{eq:score-limited-regime}
\end{equation}
Thus, when \(\nu>0\), the two regimes meet at the power scale
\begin{equation}
\gamma_c(\varepsilon)=T_0(\varepsilon)^{-1}
=\Theta(\varepsilon^{2\nu+1}).
\label{eq:score-topology-crossover}
\end{equation}
At \(\nu=0\), the scale
\([T_0(\varepsilon)\log(1/\varepsilon)]^{-1}
=\Theta(\varepsilon/\log(1/\varepsilon))\) is sufficient for asymptotically
negligible score perturbations, while
\(\gamma T_0\to\infty\) is sufficient for score domination.  In general a
logarithmic transition layer lies between these statements.
\end{proposition}

The endpoint logarithm is structural rather than an artifact of the proof.
For the critical limiting system, after a constant time rescaling,
\[
Q'=-4QV,\qquad V'=-4QV-2\eta V,
\]
the quantity \(V-Q-(\eta/2)\log Q\) is conserved.  There are schedules with
\(\eta\to0\) but \(\eta\log(1/\Delta)\to\infty\), where
\(\Delta=x_0^2-y_0^2\), for which the normalized hitting time tends to zero.
Thus the implication \(\gamma T_0\to0\Rightarrow T_r/T_0\to1\) is false at
\(\nu=0\) without additional entrance control.

The pathwise quantity \(\gamma_{\rm flow}\) and the completion-exchange margin
\(\kappa_{\mathcal C}\) introduced below measure different objects.  The
former lower-bounds a gradient ratio along one trajectory; the latter is a
discrete score gap between repairs.  Neither bounds the other without
additional assumptions on the score and path.

The conserved squared differences echo balancing invariants for homogeneous
models \citep{du2018algorithmic}.  Here the product is an exact DAG-cycle
restriction; the new consequence is the scalarization-dependent hitting-time
law and its score-margin counterpart, rather than the invariant by itself.
For an \(L\)-cycle,
\(\phi_{\rm EXP}(p)=L\sum_{\ell\ge1}p^\ell/(\ell L)!\), so
\(\phi_{\rm EXP}'(0)=1/(L-1)!\), whereas
\(\phi_{\rm DAGMA}(p)=-\log(1-p)\) and
\(\phi_{\rm DAGMA}'(0)=1\).  Thus the proposition applies exactly to both
constraints.  Linear feasibility and seeded ALM have \(\nu=0\), while a cold
quadratic penalty has \(\nu=1\), yielding the feasibility-only exponents
\(1,3,1\).  The logarithmic endpoint qualification applies to the first and
third scalarizations.

\begin{figure}[t]
    \centering
    \includegraphics[width=0.98\linewidth]{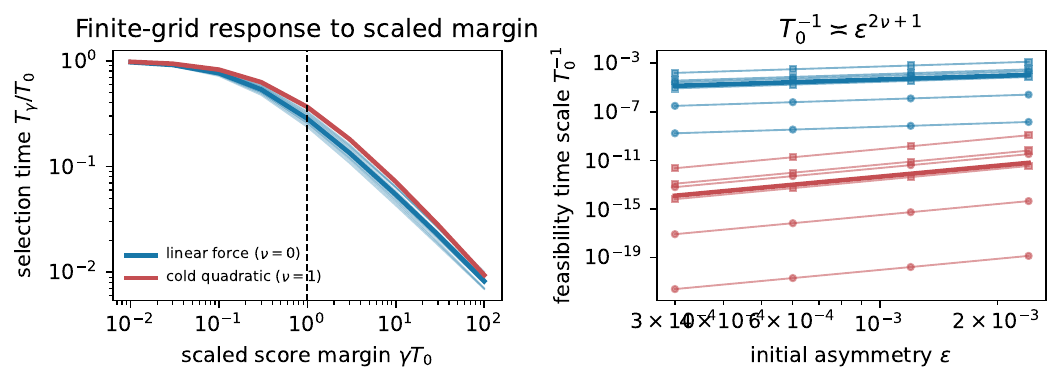}
    \caption{Finite-range score perturbation experiment.  After normalizing
    the margin by the feasibility-only time \(T_0\), 480 controlled
    NOTEARS/DAGMA flows separate around \(\gamma T_0=1\) on the plotted grid
    (left).  The fitted power scale is
    \(T_0^{-1}\asymp\varepsilon^{2\nu+1}\) (right).  This is the asymptotic
    crossover scale for \(\nu>0\); at the linear endpoint \(\nu=0\), the
    theorem contains the additional logarithmic entrance factor.}
    \label{fig:score-topology-phase}
\end{figure}

\paragraph{Numerical checks.}
The frozen grid supports each link in the argument.  Over 72 controlled
cycle settings, the largest exponent error in Eq.~\eqref{eq:symmetry-time} is
0.044, and direct matrix values and gradients agree with the closed form to
\(1.1\times10^{-15}\).  The 480-run phase experiment in Figure
\ref{fig:score-topology-phase} gives $T_\gamma/T_0\in[0.960,0.978]$ at
$\gamma T_0=0.01$, $[0.236,0.365]$ at $1$, and
$[0.0069,0.0102]$ at $100$; over this finite grid, the largest fitted
power-exponent error is 0.034.  These ratios check the normalization and the
finite-range phase change; for \(\nu=0\), they are not evidence for the
invalid fixed-\(\gamma T_0\) endpoint limit.  A population linear-Gaussian
audit selects the false back edge first
in all 216 dense flows, while 91.25\% of 720 finite-sample runs retain the
required score margin.  On 320 disjoint official NOTEARS/DAGMA trajectories,
first-stage separation correlates with selection time at
$-0.521/-0.664$ (Spearman, $p<10^{-4}$).  Appendix
\ref{app:score-topology-flow-audit} gives the full trajectories and Appendix
\ref{app:predictive-bridge} gives held-out errors and boundary cases.

\section{Interaction-Aware Score Certificates}
\label{sec:completion-margin}

Proposition~\ref{prop:symmetry-time} makes support selection depend on a score
margin that is unknown in finite samples.  Certifying that margin with
edge-additive deletion costs would be convenient, but would discard
interactions among parents of the same child.  We retain those interactions by
freezing candidate parent sets \(\mathcal P_j\), refitting each set on training
data, and evaluating bounded losses \(\ell_{j,S}\) on an independent holdout.
Let \(G_0\in\mathcal G(\mathcal P)\) be a reference DAG from any continuous,
constraint-based, or score-based procedure; the audit uses the holdout-score
minimizer.  For a candidate DAG \(G\), let
\[
Q(G)=\sum_j\{R_{j,\operatorname{pa}_G(j)}
                 +a_{j,\operatorname{pa}_G(j)}\}.
\]
Simultaneous paired intervals give, for every local alternative \(S\), a
lower bound \(g_{j,S}\) on its population score difference from the parent
set \(S_{0j}=\operatorname{pa}_{G_0}(j)\).  Put
\begin{align}
\Gamma(G)&=\sum_jg_{j,\operatorname{pa}_G(j)},
&L&=\min_G\Gamma(G),
&L_-&=\min_{G\ne G_0}\Gamma(G),
\label{eq:robust-dag-oracles}\\
\mathcal A_\alpha(G_0)&=
\{G\in\mathcal G(\mathcal P):\Gamma(G)\le0\}.
\label{eq:parent-set-confidence-family}
\end{align}
For a binary graph feature \(\varphi\), define
\begin{equation}
L_\varphi^{\rm opp}=
\min_{G:\varphi(G)\ne\varphi(G_0)}\Gamma(G).
\label{eq:forced-opposite-feature}
\end{equation}

\begin{theorem}[Interaction-aware selective feature certificate]
\label{thm:global-certificate}
Condition on the training sample.  Suppose the holdout observations are
i.i.d. and the paired intervals used to construct \(g_{j,S}\) cover all local
score differences simultaneously with probability at least \(1-\alpha\).
The reference $G_0$ may depend on this holdout.  Then every population
minimizer of \(Q\) belongs to
\(\mathcal A_\alpha(G_0)\), and
\begin{equation}
Q(G_0)-\min_{G\in\mathcal G(\mathcal P)}Q(G)\le-L.
\label{eq:global-regret-certificate}
\end{equation}
If \(L_->0\), \(G_0\) is the unique population minimizer.  If
\(L_\varphi^{\rm opp}>0\), every population minimizer shares the label
\(\varphi(G_0)\).  The feature statement applies to skeleton
adjacencies, directed edges, and unshielded colliders.
\end{theorem}

The parent-set score retains collider interactions.  Linear rows in the
parent-set integer program force skeleton and unshielded-collider labels, so
one opposite-label solve certifies each feature.  These features characterize
Markov equivalence.  Exact search is combinatorial; for edge-modular scores,
repairs reduce to transversals and a cycle-cover LP gives a polynomial-time
conservative certificate.  A bounded-parent audit runs the exact core and
predeclared feature queries through $d=200$, without claiming worst-case
tractability.  Appendices~\ref{app:completion-margin} and
\ref{app:global-certificate} give the modular specialization and main proof.

\begin{table}[t]
\centering
\scriptsize
\setlength{\tabcolsep}{3.0pt}
\caption{Unified 320-run pipeline with a shared top-$2d$ screen.  ``Oracle
parent'' uses 20,000 independent observations; labels are certified feature
percentages.}
\label{tab:interaction-aware-main}
\begin{tabular}{lrrrrrrr}
\toprule
Frontend & Screen & Front. & Modular & Parent & Oracle & Adj. & Collider \\
 & recall & SHD & SHD & SHD & parent SHD & labels & labels \\
\midrule
NOTEARS & 57.4 & 15.6 & 19.9 & 20.4 & 20.1 & 19.5 & 16.8 \\
DAGMA & 77.6 & 4.5 & 3.6 & 8.5 & 8.6 & 44.5 & 34.6 \\
SDCD & 53.4 & 39.3 & 24.1 & 21.1 & 21.0 & 10.5 & 4.4 \\
GOLEM & 77.2 & 8.1 & 4.0 & 8.5 & 8.5 & 44.4 & 36.3 \\
\bottomrule
\end{tabular}
\end{table}

Every regret bound covers the independent oracle-score audit across the 320
frontend--dataset pairs.  None of 3,042 certified skeleton or 2,396 collider
labels disagrees with its oracle-score optimum.  Parent-set SHD is 14.63,
compared with 14.56 for the 20,000-sample oracle parent-set solution, and has
correlation $-0.886$ with screen recall.  Finite holdout variance therefore does
not explain the weaker DAGMA/GOLEM parent-set results; the candidate family and
score target do.  Increasing the frozen screen from top-$2d$ to top-$6d$ raises
family recall only from $66.4\%$ to $68.5\%$ and worsens oracle-parent SHD from
14.56 to 15.67.  Modular, parent-set, and reported frontend SHDs average 12.91,
14.63, and 16.87.  Against the generating graph, $4.4\%$ of certified skeleton
labels and $5.5\%$ of collider labels disagree.  The certificate quantifies
uncertainty for its score target; it does not establish causal identification.
Appendix~\ref{app:global-certificate} gives stratified results and runtimes.

\section{Implications}

The analysis assigns separate roles to the terms in
Eq.~\eqref{eq:joint-objective}.  A smooth constraint can shape the far-field
landscape and restrict the search to DAGs, but its exactness does not rank
support changes.  When the ranking comes from data, the relevant quantity is a
score margin together with its uncertainty, not acyclicity residual alone.

This distinction leads to a concrete reporting rule.  Freeze a candidate
parent-set family, evaluate its score on independent data, report the skeleton
and collider labels shared by all interval-compatible DAGs, and leave the
remaining labels unresolved.  A modular cycle-cover relaxation exchanges
parent interactions for scalability.  The reference graph may come from
NOTEARS, PC, GES, or another screened search; continuous optimization is one
way to reduce the candidate family, not a source of statistical certification.

The guarantee remains conditional on the chosen screen and score.  It cannot
recover omitted edges, prove observational identifiability, or remove the
combinatorial cost of exact parent-set search.  These limits are part of the
main conclusion: DAG feasibility, score-based support selection, and causal
identification require different arguments.

\subsection*{AI use statement}
In this work, generative AI tools were used to assist with drafting and
editing portions of the manuscript for clarity, grammar, and readability.
They were not used to formulate the scientific claims, develop theoretical
results or proofs, design the methodology or experiments, or interpret the
experimental results. All AI-assisted text was reviewed, verified, and
revised by the authors. We take full responsibility for the final content
of this work.



\appendix
\paragraph{Appendix guide.}
The supplement is organized around the three claims in the main text.
Sections~\ref{sec:structural-consequences}--\ref{app:proofs} develop completion
geometry and its proofs.  The symmetry, crossover, and numerical audits follow,
after which Sections~\ref{app:completion-margin} and
\ref{app:global-certificate} give the two score certificates and their frozen
protocols.

\section{Additional Geometry and Consequences}
\label{sec:structural-consequences}

The main text uses only the completion order and the far-field symmetry
construction.  Two additional consequences describe the output dimension
needed for uniform sharpness and the topology of local candidate sections.

\subsection{Uniform sharpness and residual dimension}

Write $J_q(U)=D^qH(0)[U,\ldots,U]/q!$ for the degree-$q$ Taylor jet.  Call
$H$ \emph{directionwise $q$-sharp} on a family $U(z)$ if
$J_q(U(z))\ne0$ for every $z\ne0$ in that family.  Its
\emph{$q$-sharp residual complexity} $s_q(U)$ is the minimum output dimension
of a $C^q$ map $H:\mathcal Z\to\R^r$ that is globally exact and has this
property.

\begin{theorem}[Near-tight sensitivity--succinctness law]
\label{thm:dimension}
For every $3\le q\le d$, there is an explicit channel family $U_q$ with
\begin{equation}
M_q(d)=\left\lfloor\frac{(d-q+2)^2}{4}\right\rfloor
\label{eq:funnel-dimension}
\end{equation}
independent $q$-cycle channels satisfying
\begin{equation}
M_q(d)\le s_q(U_q)\le M_q(d)+1.
\label{eq:near-tight-dimension}
\end{equation}
For $q=2$, there is a family with one channel per unordered node pair and
\begin{equation}
\binom d2\le s_2(U_2)\le\binom d2+1.
\label{eq:dimension-lower}
\end{equation}
Consequently, $s_q(U_q)=\Theta(d^2)$ for every fixed $q$.  Below either
lower bound, some cyclic mixture has zero degree-$q$ jet; the upper bound is
globally exact and directionwise sharp.
\end{theorem}

The lower bound is a rank argument on a layered funnel.  The upper bound
collects the selected cycle products and appends one globally exact scalar.
Section~\ref{app:dimension} gives the construction and proof.  The theorem
requires sharpness in every channel mixture; it is not a convergence lower
bound for arbitrary lower-dimensional residuals.

\subsection{Coordinate topology}

\begin{figure}[ht]
    \centering
    \includegraphics[width=0.63\linewidth]{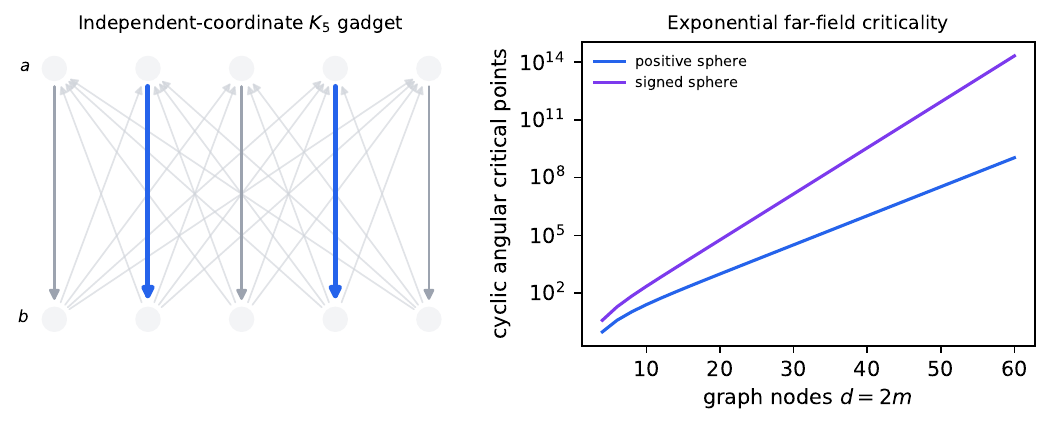}
    \caption{Ordinary-coordinate construction.  A $K_5$ gadget realizes all
    pairwise cyclic candidate supports (left), producing exponentially many
    cyclic angular critical points under disjoint union (right).}
    \label{fig:dag-germ-universality}
\end{figure}

The positive normalized candidate section is
\begin{equation}
\mathcal A^+_{\W,F}
=\left\{x\ge0:\|x\|_2=1,\
\W+\sum_{e\in F}x_e\E_e\in\mathcal D\right\}
\cong|\Delta_{\W,F}|.
\label{eq:positive-dag-section}
\end{equation}

\begin{theorem}[Coordinate DAG-germ universality]
\label{thm:coordinate-universality}
For every finite simplicial complex $K$, there are a base DAG $\W_0$ and
independently parameterized absent adjacency entries $F$ such that
\begin{equation}
\Delta_{\W_0,F}\cong\operatorname{sd}K.
\label{eq:coordinate-universality}
\end{equation}
Consequently $\mathcal A^+_{\W_0,F}$ is homeomorphic to $|K|$.
\end{theorem}

The construction represents the vertices of an independence complex by
candidate adjacency entries.  This differs from the full acyclic-subgraph
complex, whose homotopy type is restricted \citep{hultman2004directed}:
conditioning on a base DAG and then selecting candidate coordinates changes
the local topology.

\begin{corollary}[Exponential local homology]
\label{cor:exponential-homology}
There is a family with $d=10t$ graph nodes and $5t$ independent candidates
such that
\begin{equation}
\Delta_{\W_t,F_t}\simeq\bigvee_{4^t}S^{t-1},
\qquad
\sum_j\widetilde\beta_j(\Delta_{\W_t,F_t})
=4^t=2^{d/5}.
\label{eq:exponential-homology}
\end{equation}
\end{corollary}

\begin{corollary}[Exponential angular criticality]
\label{cor:angular-criticality}
On one $K_m$ gadget, the positive and signed candidate unit spheres contain at
least $2^m-m-1$ and $3^m-2m-1$ distinct cyclic critical points, respectively.
All corresponding DAGMA points lie in its $s=1$ M-matrix domain.
\end{corollary}

The proof uses the fact that every finite complex becomes an independence
complex after barycentric subdivision \citep{ehrenborg2006topology}; the DAG
contribution is an ordinary-adjacency realization of that complex.

\subsection{Metric consequences}
\label{sec:metric-consequences}

Derivative order also controls how reliably a residual measures distance to
the DAG set.  Write
\begin{equation}
\mathcal D_F=\left\{x\in V_F:
\W+\sum_sx_s\E_s\in\mathcal D\right\}.
\end{equation}
A local H\"older error bound with exponent \(\theta>0\) is an inequality
\begin{equation}
\operatorname{dist}(x,\mathcal D_F)
\le \kappa\norm{R_F(x)}^\theta
\label{eq:holder-error-bound}
\end{equation}
for all \(x\) near zero and some \(\kappa>0\); this is also called H\"older
metric subregularity \citep{li2012holder}.

\begin{theorem}[Sharp topological error-bound exponents]
\label{thm:holder-barrier}
Suppose \(q_\W(F)<\infty\) and let \(r=r_\W(F)\).  For any \(C^r\)
signed/vector exact residual \(H_F\), every exponent in
Eq.~\eqref{eq:holder-error-bound} satisfies \(\theta\le1/r\).  For any
\(C^{2r}\) nonnegative scalar exact residual \(h_F\), every exponent satisfies
\(\theta\le1/(2r)\).  Both bounds are attained: the simple-cycle product
vector has optimal exponent \(1/r\), while the positive walk constraints in
Corollary~\ref{cor:positive-walk-order} have optimal exponent \(1/(2r)\).
Consequently, squaring these sharp residuals gives optimal exponents
\(1/(2r)\) and \(1/(4r)\).
\end{theorem}

\begin{proof}[Proof sketch]
A largest minimal completion gives necessity along \(t\mathbf1_C\).
Conversely, zeroing one smallest coordinate from each minimal completion gives
a hitting set and
\(\operatorname{dist}(x,\mathcal D_F)\le
\sqrt{|\mathcal C_{\rm min}|}\norm{M_{\W,F}(x)}_2^{1/r}\).
Cycle products dominate \(M_{\W,F}\) and positive walk scalars dominate its
squared norm; see Proposition~\ref{prop:general-holder}.
\end{proof}

Two further invariants describe the size of neighborhoods rather than their
worst direction.  Let
\begin{align}
b_\W(F)
&=m-\max_{S\in\Delta_{\W,F}}|S|, \nonumber\\
\lambda_\W(F)
&=\min_{\substack{z\ge0\\\sum_{s\in C}z_s\ge1,\ 
 C\in\mathcal C_{\rm min}}}\sum_{s=1}^m z_s .
\label{eq:volume-invariants}
\end{align}
Thus \(b\) is the minimum integral completion cover (the height of
\(I_{\W,F}\)), while \(\lambda\) is its fractional relaxation.

\begin{theorem}[Completion-volume geometry]
\label{thm:volume-geometry}
For a sufficiently small box \(B_\rho=[-\rho,\rho]^m\) and
\(0\le\varepsilon\le\rho\), the normalized \(\ell_\infty\) tube volume is
exactly
\begin{equation}
\sum_{S\in\Delta_{\W,F}}
(1-\varepsilon/\rho)^{|S|}(\varepsilon/\rho)^{m-|S|};
\label{eq:tube-main}
\end{equation}
its small-\(\varepsilon\) exponent is \(b\).  If
\(\mathcal V_R(\eta)=
\operatorname{vol}\{x\in B_\rho:\|R(x)\|\le\eta\}\), then
\begin{equation}
\lim_{\eta\downarrow0}\frac{\log\mathcal V_{H_{\rm cyc}}(\eta)}
{\log\eta}=\lambda,\qquad
\lim_{\eta\downarrow0}\frac{\log\mathcal V_{h_f}(\eta)}
{\log\eta}=\frac{\lambda}{2}
\label{eq:volume-main}
\end{equation}
for the simple-cycle vector and every locally convergent positive-walk scalar
\(h_f\).  Moreover, \(b/r\le\lambda\le b\).
\end{theorem}

The logarithmic substitution \(|x_s|=e^{-u_s}\) turns a generator sublevel set
into the fractional-cover inequalities
\(\sum_{s\in C}u_s\ge\log(1/\eta)\).  LP primal and dual solutions give
matching exponential bounds; Section~\ref{app:completion-volume} also proves
the required residual comparisons.  The formula for \(\lambda\) is the
classical monomial-ideal threshold \citep{howald2001multiplier}; its
completion-graph interpretation and separation from \(q,r,b\) are specific
to this setting.

Thus a smooth nonnegative exact scalar is never Lipschitz metrically
subregular at such a support boundary, and a vector residual is not Lipschitz
metrically subregular when \(r\ge2\).  It also makes finite smooth penalties
non-exact against first-order score descent
(Corollary~\ref{cor:finite-penalty}).

Smooth scalarization adds a second loss of sensitivity.  On paths where a
vector residual is \(q\)-sharp and a nonnegative scalar residual has exact
order \(2q\), the common squared-penalty constructions obey
\begin{equation}
\underbrace{H}_{q}
\quad\longrightarrow\quad
\underbrace{\tfrac12\norm{H}^2}_{2q}
\quad\text{and}\quad
\underbrace{h}_{2q}
\quad\longrightarrow\quad
\underbrace{\tfrac12 h^2}_{4q}.
\label{eq:hierarchy}
\end{equation}
The subscripts are exact orders under sharpness and lower bounds otherwise.
Increasing an augmented-Lagrangian penalty changes coefficients, not these
missing derivatives; a residual-sized multiplier also leaves a cold start at
the squared order.  Section~\ref{app:numerical-consequences} gives the full
sensitivity--smoothness frontier, ALM statements, and finite-precision
calculations.

\subsection{Statistical consequences}
\label{sec:inference}

The derivative orders in the main paper also determine the null scale of a
residual evaluated at a noisy estimator.  For fixed \(d\), let \(F\) contain
all coordinates absent from a fixed DAG \(W\), and suppose
\begin{equation}
\sqrt n(\widehat W_n-W)\ \Rightarrow\ Z
\label{eq:root-n-estimator}
\end{equation}
for a nondegenerate Gaussian matrix \(Z\).  This assumption is deliberately
modular: it can arise from randomized interventions, temporal information, or
another model that identifies the directed parameter.

\begin{theorem}[Topology-indexed delta law]
\label{thm:statistical-scaling}
Let \(q=q_W(F)<\infty\) and consider a locally convergent positive-walk
constraint
\begin{equation}
h_f(A)=\sum_{k\ge1}c_k\tr((A\circ A)^k),\qquad c_k>0.
\end{equation}
Then
\begin{equation}
n^q h_f(\widehat W_n)\Rightarrow
P_{W,q}(Z_F),\qquad
P_{W,q}(z)=
\sum_{\substack{C\in\mathcal C_{\rm min}(W,F)\\|C|=q}}
a_C\prod_{e\in C}z_e^2,
\label{eq:scalar-null-law}
\end{equation}
where every \(a_C\) is positive.  Under
\(W_n=W+U/\sqrt n\), the limit is \(P_{W,q}(Z_F+U_F)\).
\end{theorem}

Theorem~\ref{thm:statistical-scaling} is a higher-order delta-method calculation
at a singular restriction, closely related to singular Wald asymptotics and inference for irregular
polynomial constraints \citep{drton2016wald,sturma2024testing}.  The completion
ideal supplies the part that is specific to acyclicity: the degree, the
monomials that survive, and their positive coefficients.  Fluctuations of the
fixed active edges enter one order later.
Section~\ref{app:statistical-desingularization} gives the proof.

\begin{corollary}[No universal scalar tolerance]
\label{cor:no-universal-tolerance}
For a rule that accepts when \(h_f(\widehat W_n)\le c n^{-a}\),
\begin{equation}
\Pr(\mathrm{accept})\longrightarrow
\begin{cases}
0,&a>q,\\
\Pr\{P_{W,q}(Z_F)\le c\},&a=q,\\
1,&a<q.
\end{cases}
\label{eq:tolerance-trichotomy}
\end{equation}
Thus one exponent \(a\) cannot be nondegenerately calibrated at two DAG
strata with different completion numbers.  Moreover, if one completion
complex has minimal generators of sizes \(q<r\), the limit in
Eq.~\eqref{eq:scalar-null-law} retains only size-\(q\) generators.  A local
shift on a disjoint size-\(r\) generator is invisible to the shortest-order
scalar limit.
\end{corollary}

\section{Proofs of Completion-Order Results}
\label{app:proofs}

All derivatives below are taken in the Euclidean space
$\mathcal Z=\{W\in\R^{d\times d}:\operatorname{diag}(W)=0\}$.  For a
multi-index $\alpha$, $\operatorname{supp}(\alpha)=\{s:\alpha_s>0\}$.

\subsection{Coordinate-subspace vanishing}

We first record the elementary mechanism behind the vector result.

\begin{lemma}[Local coordinate-subspace jet]
\label{lem:coordinate-jet}
Let $V\subseteq\R^m$ be an open neighborhood of zero and let
$G:V\to\R^r$ be $C^{k}$.  Suppose that, for every coordinate set $S$ with
$|S|\le k$, $G$ is zero on a neighborhood of zero in
$V\cap\R^S$.
Then
\begin{equation}
D^\alpha G(0)=0\qquad\text{for every }|\alpha|\le k.
\end{equation}
\end{lemma}

\begin{proof}
Fix $\alpha$ with $|\alpha|\le k$ and set
$S=\operatorname{supp}(\alpha)$.  Since $|S|\le|\alpha|\le k$, the
local restriction $G_S$ obtained by setting all coordinates outside $S$ to
zero is identically zero near the origin.  The derivative $D^\alpha G(0)$
is the corresponding derivative of $G_S$ and is therefore zero, component
by component.
\end{proof}

\paragraph{Proof of Theorem~\ref{thm:vector}.}
For every $S\subseteq F$ with $|S|<q_W(F)$, the support
$\operatorname{supp}(W)\cup S$ is acyclic.  Exactness therefore makes
$H_F$ zero on a neighborhood of the origin in the coordinate subspace
indexed by $S$.
Lemma~\ref{lem:coordinate-jet}, with $k=q-1$, proves
Eq.~\eqref{eq:vector-jet}.  If $H$ is $C^q$, the multivariate Taylor
formula with integral remainder gives, uniformly for $u$ in a compact set,
\begin{equation}
\|H_F(tu)\|_2\le C|t|^q\|u\|_2^q.
\end{equation}
Apply the same formula to $D H_F$.  Its derivatives through order $q-2$
are derivatives of $H_F$ through order $q-1$ and vanish at zero, giving
\begin{equation}
\|D H_F(tu)\|_{\rm op}\le C'|t|^{q-1}.
\end{equation}
This is the Jacobian of the restriction to $F$.  A coordinate outside $F$
may close an existing path with fewer edges, so the same order need not hold
for the full Jacobian. \qed

The theorem gives a lower bound on the first possible nonzero order, not an
equality.  Smooth nonanalytic exact maps may be flat to all orders, and
signed components can cancel on a particular path.

\subsection{Weighted paths}
\label{app:weighted}

We prove both weighted statements by a clone-variable lift, using the local
coordinate-subspace argument directly rather than treating clones as graph
edges.  Let $a_s\in\mathbb N$, introduce variables
$z_{s,1},\ldots,z_{s,a_s}$, and define
\begin{equation}
\Gamma(z)=W+\sum_{s=1}^m c_s
\left(\prod_{j=1}^{a_s}z_{s,j}\right)E_s.
\label{eq:clone-map}
\end{equation}
Original edge $e_s$ is active exactly when all of its clones are nonzero.
If $\operatorname{supp}(W)$ together with the active original edges is
cyclic, the support of $z$ contains at least
\begin{equation}
\min_{S:\operatorname{supp}(W)\cup S\text{ cyclic}}
\sum_{e_s\in S}a_s=\tau_W(F;a)
\end{equation}
clone coordinates.  Consequently, every clone-coordinate subspace of
dimension below $\tau$ maps to DAG support.  Conversely, activating all
clones of a minimizing set $S$ uses exactly $\tau$ clone coordinates and
maps to cyclic support.  Thus $\tau$ is exactly the coordinate-subspace
completion number of the lifted composite.

Let $V_\Gamma=\{z:\Gamma(z)\in\Omega\}$, an open neighborhood of zero, and
set $G=H\circ\Gamma$.  If $H$ is $C^{\tau-1}$, then so is $G$, and $G$ is
zero near the origin on every clone-coordinate subspace of dimension below
$\tau$.  Lemma~\ref{lem:coordinate-jet} makes all derivatives of $G$ below
order $\tau$ vanish.  If $H$ is $C^\tau$, Taylor's formula also gives
$G(z)=O(\|z\|^\tau)$.  On the diagonal clone path $z_{s,j}=t$,
Eq.~\eqref{eq:clone-map} is precisely $W_a(t)$, proving
Corollary~\ref{cor:weighted-vector}.

For the scalar statement set $g=h\circ\Gamma$.  Assume first that $h$ is
$C^{2\tau-1}$.  The function $g$ is locally nonnegative and is zero on every
clone-coordinate subspace of dimension below $\tau$.  If its first nonzero
homogeneous Taylor term had degree $r<2\tau$, the proof of
Theorem~\ref{thm:scalar} would apply verbatim: every monomial would use at
least $\tau$ clones, while a Newton vertex of a nonnegative polynomial has
even exponents, forcing $r\ge2\tau$.  Hence all derivatives below
$2\tau$ vanish.  Under $C^{2\tau}$ regularity, Taylor's formula yields
\begin{equation}
h(W_a(t))=O(|t|^{2\tau}).
\label{eq:weighted-scalar-order}
\end{equation}
The same one-sided asymptotic bounds for positive rational exponents follow
after multiplying all exponents by a common denominator and
reparameterizing $t\downarrow0$; derivative statements at zero require the
resulting path to possess the corresponding regularity.

\subsection{A sharp vector representation}

We verify exactness and count the components of
Eq.~\eqref{eq:cycle-vector}.

\begin{proposition}[Simple-cycle vector]
\label{prop:cycle-vector}
$H_{\rm cyc}(W)=0$ if and only if $\operatorname{supp}(W)$ is acyclic.
For a minimal $q$-edge completion, its restricted residual and Jacobian
orders are exactly $q$ and $q-1$, respectively.  Along a weighted path the
residual order is exactly $\tau$, provided the coefficients on a cheapest
completion are nonzero.
\end{proposition}

\begin{proof}
Every finite directed graph is cyclic if and only if it contains a simple
directed cycle.  The component indexed by such a cycle is nonzero exactly
when all its edges are present, proving exactness.  A cycle created by a
minimal completion must use all $q$ completion edges; otherwise a proper
subset would already be cyclic.  Its component is a nonzero fixed-edge
coefficient times the product of those $q$ candidate coordinates.  This
has equal-scale order $q$, and differentiating in a completion coordinate
has order $q-1$.  No component has lower order by Theorem~\ref{thm:vector}.
The same argument chooses a cycle whose completion cost is $\tau$ on a
weighted path.
\end{proof}

There are $\binom{d}{k}$ choices of nodes for a length-$k$ cycle.  Fixing
the smallest selected node as the first node leaves $(k-1)!$ directed
orders, each representing one cyclic-rotation class.  Summing over
$k=2,\ldots,d$ proves Eq.~\eqref{eq:cycle-count}.

\subsection{Uniform sharpness and residual dimension}
\label{app:dimension}

We prove Theorem~\ref{thm:dimension}.  Let
\begin{equation}
P_q(U)=\frac1{q!}D^qH(0)[U,\ldots,U]
\end{equation}
be the homogeneous degree-$q$ Taylor jet of $H$ at the empty graph.

Fix $3\le q\le d$ and set $n=d-q+2$.  Partition $n$ endpoint nodes into
disjoint sets $S,T$ of sizes $a=\lfloor n/2\rfloor$ and
$b=\lceil n/2\rceil$.  The remaining $q-2$ nodes form singleton layers
$v_1,\ldots,v_{q-2}$.  Let $U(z)$ contain all forward funnel edges
\begin{equation}
s\to v_1,\quad v_1\to\cdots\to v_{q-2},\quad v_{q-2}\to t
\end{equation}
with weight one for $s\in S,t\in T$, and let the reverse edge $t\to s$
have weight $z_{st}$.  The forward graph is acyclic.  Every nonzero reverse
channel completes the simple $q$-cycle
$s\to v_1\to\cdots\to v_{q-2}\to t\to s$; hence $tU(z)$ is cyclic for
every $t\ne0$ and $z\ne0$.

Any cyclic support in this funnel must contain one reverse edge and all
$q-1$ edges of its associated forward path.  A degree-$q$ Taylor monomial
with any other support is therefore acyclic and its coefficient vanishes by
exactness on that coordinate subspace.  Total degree $q$ also forces exponent
one on each cycle edge.  It follows that
\begin{equation}
P_q(U(z))=\sum_{s\in S,t\in T}v_{st}z_{st}=Lz,
\qquad L:\R^{ab}\to\R^r.
\end{equation}
Directionwise $q$-sharpness requires $Lz\ne0$ for every $z\ne0$, so $L$ is
injective and
\begin{equation}
r\ge ab=\left\lfloor\frac{(d-q+2)^2}{4}\right\rfloor.
\end{equation}
Conversely, if $r<ab$, rank--nullity gives $z\ne0$ with $P_q(U(z))=0$.
Theorem~\ref{thm:vector} removes every lower jet, so
$H(tU(z))=o(|t|^q)$ even though exactness requires it to be nonzero for all
sufficiently small $t\ne0$.

It remains to prove the upper bound.  Let $C_q(W)\in\R^{ab}$ collect the
cycle products associated with the selected funnel channels, and define
\begin{equation}
H_q^+(W)=\left(C_q(W),h_{\exp}(W)\right)\in\R^{ab+1}.
\label{eq:sharp-upper-construction}
\end{equation}
The EXP component is zero exactly on DAGs, so $H_q^+$ is globally exact.
On the funnel family, $C_q(tU(z))=t^qz$ up to fixed nonzero forward-edge
coefficients.  Hence its degree-$q$ jet is nonzero for every $z\ne0$,
proving $s_q(U_q)\le ab+1$.

For $q=2$, label nodes by a total order.  Let $U(z)_{ij}=1$ for every
$i<j$, and let $U(z)_{ji}=z_{ij}$.  Any two directed edges form a cycle if
and only if they are the two opposite orientations of one unordered pair.
Thus exactness eliminates every quadratic Taylor monomial except
$W_{ij}W_{ji}$.  Consequently
\begin{equation}
P_2(U(z))=\sum_{i<j}v_{ij}z_{ij}=Lz.
\end{equation}
Directionwise second-order sharpness makes
$L:\R^{\binom d2}\to\R^r$ injective, proving
$r\ge\binom d2$.  Conversely, if $r<\binom d2$, rank--nullity supplies a
nonzero $z$ with $P_2(U(z))=0$.  The support of $tU(z)$ contains a two-cycle
for every $t\ne0$, so exactness still requires $H(tU(z))\ne0$; the first two
Taylor jets vanish and $H(tU(z))=o(t^2)$.
For the upper bound, collect the $\binom d2$ products
$W_{ij}W_{ji}$ and append $h_{\exp}(W)$ as in
Eq.~\eqref{eq:sharp-upper-construction}.  This map is globally exact and its
quadratic jet on the pair-channel family is the coefficient vector $z$.

\subsection{Nonnegative scalar representations}
\label{app:scalar}

We give a proof of the Newton-polytope fact used in the main text and then
prove the scalar barrier in full.

\begin{proof}[Proof of Lemma~\ref{lem:newton}]
Let $P(x)=\sum_\beta c_\beta x^\beta$ be a nonzero globally nonnegative
polynomial and let $\alpha$ be a vertex of its Newton polytope.  A vector
$w\in\R^m$ uniquely exposes $\alpha$; changing its sign if necessary, take
$\langle w,\alpha\rangle<\langle w,\beta\rangle$ for every other exponent.
For a sign vector $s\in\{-1,1\}^m$, evaluate $P$ on the monomial curve
$x_i=s_it^{w_i}$.  Although coordinates with $w_i<0$ diverge as
$t\downarrow0$, the polynomial remains nonnegative.  Unique exposure gives
\begin{equation}
t^{-\langle w,\alpha\rangle}P(s_1t^{w_1},\ldots,s_mt^{w_m})
\longrightarrow c_\alpha s^\alpha.
\end{equation}
The limit is nonnegative for every $s$.  If some $\alpha_i$ were odd,
flipping $s_i$ would flip the nonzero limit.  Thus $\alpha$ is
coordinatewise even, and taking $s=\mathbf 1$ gives $c_\alpha>0$.
\end{proof}

\begin{proof}[Proof of Theorem~\ref{thm:scalar}]
Suppose, for contradiction, that the Taylor jet of $h_F$ has a first
nonzero homogeneous term $P_r$ with $r<2q$.  Local nonnegativity of $h_F$
is sufficient to make $P_r$ globally nonnegative: for fixed $x$ and
$t\downarrow0$,
\begin{equation}
t^{-r}h_F(tx)\longrightarrow P_r(x)\ge0.
\end{equation}
Every coordinate subspace supported on fewer than $q$ candidates is
acyclic, so the restriction of $h_F$, and hence of $P_r$, to that subspace
is zero.  Therefore every monomial of $P_r$ uses at least $q$ distinct
coordinates.  Choose any vertex exponent $\alpha$ of the Newton polytope
of $P_r$.  Lemma~\ref{lem:newton} makes every nonzero $\alpha_s$ at least
two.  Hence
\begin{equation}
r=|\alpha|\ge2|\operatorname{supp}(\alpha)|\ge2q,
\end{equation}
a contradiction.  Thus all derivatives through order $2q-1$ vanish.
Under $C^{2q}$ regularity, Taylor remainder bounds give
$h_F(tu)=O(|t|^{2q})$.  Applying the same argument to $D h_F$ gives the
restricted gradient order $O(|t|^{2q-1})$.
\end{proof}

Nonnegativity is essential for the factor of two.  For $d=2$, the signed
scalar $W_{12}W_{21}$ is exact and has order two at the empty graph, whereas
every smooth nonnegative scalar exact representation has order at least
four there.

\subsection{Minimal-completion factorization}

Let $F=\{e_1,\ldots,e_q\}$ itself be a minimal completion and write
$\Phi(x)=h(W+\sum_sx_sE_s)$ on a sufficiently small box contained in
$V_F$.  If $x_s=0$, only a proper subset of $F$ is active and exactness
gives $\Phi(x)=0$.  For fixed $x_{-s}$, the
nonnegative differentiable function $u\mapsto\Phi(x_{-s},u)$ has a minimum
at zero, so
\begin{equation}
\Phi(x_{-s},0)=0,
\qquad \partial_s\Phi(x_{-s},0)=0.
\label{eq:double-vanishing}
\end{equation}

The integral form of the second-order Hadamard lemma is
\begin{equation}
f(u,z)=u^2\int_0^1(1-r)\partial_{uu}f(ru,z)\,dr
\label{eq:hadamard-second}
\end{equation}
whenever $f(0,z)=\partial_uf(0,z)=0$.  We apply it inductively.  After
$k-1$ steps, write
\begin{equation}
\Phi(x)=\left(\prod_{s<k}x_s^2\right)\Phi_{k-1}(x),
\end{equation}
where $\Phi_{k-1}$ is nonnegative and $C^{2q-2(k-1)}$.  On the dense set
where the preceding coordinates are nonzero, $\Phi_{k-1}(x)=0$ whenever
$x_k=0$; continuity extends this identity to the divided hyperplanes.
Because $\Phi_{k-1}$ is nonnegative and differentiable, its $x_k$ derivative
also vanishes there.  Equation~\eqref{eq:hadamard-second} extracts another
factor $x_k^2$ and leaves a nonnegative $C^{2q-2k}$ quotient.  The
$C^{2q}$ assumption permits all $q$ steps and yields
\begin{equation}
\Phi(x)=\left(\prod_{s=1}^qx_s^2\right)\Psi(x)
\end{equation}
with continuous $\Psi$.  Off the coordinate hyperplanes the squared
product is positive and exactness gives $\Phi>0$, so $\Psi>0$ there;
continuity gives $\Psi\ge0$ everywhere locally.

\begin{proposition}[Sharpness for nonnegative walk sums]
\label{prop:walk-sharp}
Let $A=W\circ W$ and
$h_f(W)=\sum_{k\ge1}c_k\tr(A^k)$, where $c_k\ge0$ and the coefficient for
every simple-cycle length is positive.  Suppose the series converges in a
neighborhood of a minimal $q$-edge completion to a $C^{2q}$ function.  Then
\begin{equation}
h_f(W+tU)=C|t|^{2q}+o(|t|^{2q}),\qquad
\norm{D(h_f)_F(tu)}_2=\Theta(|t|^{2q-1})
\end{equation}
for some $C>0$.
\end{proposition}

\begin{proof}
Choose a simple cycle in the completed support.  Minimality forces it to use
all $q$ candidate edges.  Its cyclic rotations contribute a positive
constant times the product of the $q$ squared candidate coordinates to the
appropriate trace power.  All other walk-sum terms are nonnegative, and
the same expansion is zero exactly on DAG support.  Thus
$h_f(W+tU)\ge c|t|^{2q}$ for some $c>0$ and all sufficiently small nonzero
$t$.  Theorem~\ref{thm:scalar} and $C^{2q}$ regularity give
$h_f(W+tU)=C|t|^{2q}+o(|t|^{2q})$; the lower bound forces $C>0$.
Let $P_{2q}$ be the first homogeneous Taylor term, so $P_{2q}(u)=C>0$.
The restricted gradient has leading term $t^{2q-1}\nabla P_{2q}(u)$.
Euler's identity gives
$\langle\nabla P_{2q}(u),u\rangle=2qP_{2q}(u)=2qC\ne0$, so the gradient
norm is bounded below by a positive multiple of $|t|^{2q-1}$.  Its Taylor
expansion gives the matching upper bound.
\end{proof}

This proposition covers EXP and the usual positive trace polynomial.  The
same proof covers an entrywise squash satisfying $s(0)=0$, $s(x)>0$ for
$x\ne0$, and $s(x)=\gamma x^2+o(x^2)$ with $\gamma>0$.  The log-determinant
series has the same positive walk expansion inside its convergence domain.

\subsection{General completion complexes and exact error bounds}

\begin{proposition}[Generator domination and optimal exponents]
\label{prop:general-holder}
Let $\mathcal C_{\rm min}=\mathcal C_{\rm min}(W,F)$ be nonempty,
$N=|\mathcal C_{\rm min}|$, and
$r=\max_{C\in\mathcal C_{\rm min}}|C|$.  On a sufficiently small candidate
box,
\begin{equation}
\operatorname{dist}(x,\mathcal D_F)
\le \sqrt N\,\|M_{W,F}(x)\|_2^{1/r}.
\label{eq:generator-error-bound}
\end{equation}
Moreover, there are constants $c_1,c_2>0$ such that
\begin{align}
\|(H_{\rm cyc})_F(x)\|_2
&\ge c_1\|M_{W,F}(x)\|_2, \label{eq:cycle-dominates-generator}\\
h_f\left(W+\sum_sx_sE_s\right)
&\ge c_2\|M_{W,F}(x)\|_2^2 \label{eq:walk-dominates-generator}
\end{align}
for every positive walk scalar in Proposition~\ref{prop:walk-sharp}.
Consequently, the two families have optimal error-bound exponents $1/r$
and $1/(2r)$, respectively.
\end{proposition}

\begin{proof}
For each $C\in\mathcal C_{\rm min}$, choose
$i(C)\in\arg\min_{s\in C}|x_s|$ and set every selected coordinate to zero.
This deletion set meets every minimal nonface, so the remaining candidate
support lies in $\Delta_{W,F}$.  If $\|x\|_\infty\le1$, the distance to this
particular acyclic point satisfies
\begin{align}
\operatorname{dist}(x,\mathcal D_F)^2
&\le \sum_{C\in\mathcal C_{\rm min}}\min_{s\in C}|x_s|^2\\
&\le \sum_{C\in\mathcal C_{\rm min}}|(x^C)|^{2/|C|}
\le N\|M_{W,F}(x)\|_2^{2/r}.
\end{align}
The middle inequality is the minimum--geometric-mean inequality.  The last
uses $|C|\le r$ and $|x^C|\le1$, proving
Eq.~\eqref{eq:generator-error-bound}.

For every minimal completion $C$, choose a simple directed cycle in
$\operatorname{supp}(W)\cup C$.  Minimality forces this cycle to use every
candidate in $C$; its other edges have fixed nonzero weights in $W$.  The
corresponding component of $H_{\rm cyc}$ is therefore $b_Cx^C$ with
$b_C\ne0$.  Taking the minimum of finitely many $|b_C|$ proves
Eq.~\eqref{eq:cycle-dominates-generator}.  The same selected cycles occur
with positive coefficients and squared edge weights in $h_f$, and all other
walk terms are nonnegative.  Their finite minimum gives
Eq.~\eqref{eq:walk-dominates-generator}.

Finally choose $C_\star\in\mathcal C_{\rm min}$ with $|C_\star|=r$ and set
$x=t\mathbf1_{C_\star}$.  Every other minimal generator contains a coordinate
outside $C_\star$, while acyclicity requires at least one coordinate of
$C_\star$ to be zero.  Hence
\begin{equation}
\operatorname{dist}(x,\mathcal D_F)=|t|,
\qquad \|M_{W,F}(x)\|_2=|t|^r.
\end{equation}
Proposition~\ref{prop:walk-sharp}, restricted to $C_\star$, gives scalar
order $2r$.  Thus neither exponent can be increased.
\end{proof}

\subsection{Coordinate tubes and residual sublevel volume}
\label{app:completion-volume}

The completion ideal also controls two volume notions that are not determined
by $q$ or $r$.  Let $m=|F|$ and define
\begin{align}
b_W(F)
&=m-\max_{S\in\Delta_{W,F}}|S| \nonumber\\
&=\min\{|T|:T\cap C\ne\varnothing
          \text{ for every }C\in\mathcal C_{\rm min}\},
\label{eq:completion-height}\\
\lambda_W(F)
&=\min_{\substack{z\in\R_+^m\\
                  \sum_{s\in C}z_s\ge1,\ C\in\mathcal C_{\rm min}}}
                  \sum_{s=1}^m z_s .
\label{eq:fractional-completion-cover}
\end{align}
The first quantity is the height of $I_{W,F}$ and the minimum number of
candidate edges whose removal from the full candidate graph makes it
acyclic.  The second is the fractional cover number of the
minimal-completion hypergraph.  In the standard algebraic terminology,
Eq.~\eqref{eq:fractional-completion-cover} is the Newton-polyhedron formula
for the log-canonical threshold of this monomial ideal
\citep{howald2001multiplier}.

\begin{proposition}[Exact coordinate-tube law]
\label{prop:tube-law}
Choose $\rho>0$ such that $B_\rho=[-\rho,\rho]^m\subset V_F$.  For
$0\le\varepsilon\le\rho$,
\begin{align}
&\frac{\operatorname{vol}\{x\in B_\rho:
  \operatorname{dist}_\infty(x,\mathcal D_F)\le\varepsilon\}}
 {(2\rho)^m} \nonumber\\
&\qquad =
\sum_{S\in\Delta_{W,F}}
\left(1-\frac{\varepsilon}{\rho}\right)^{|S|}
\left(\frac{\varepsilon}{\rho}\right)^{m-|S|}.
\label{eq:exact-tube-law}
\end{align}
In particular, its small-$\varepsilon$ exponent is $b_W(F)$, and the
leading coefficient after normalization by
$(\varepsilon/\rho)^{b_W(F)}$ is the number of maximum-cardinality faces of
$\Delta_{W,F}$.
\end{proposition}

\begin{proof}
For $x\in B_\rho$, set
$S_\varepsilon(x)=\{s:|x_s|>\varepsilon\}$.  If some $y\in\mathcal D_F$
satisfies $\|x-y\|_\infty\le\varepsilon$, then every coordinate in
$S_\varepsilon(x)$ is nonzero in $y$.  Hence
$S_\varepsilon(x)\subseteq\supp(y)\in\Delta_{W,F}$ and, by downward
closure, $S_\varepsilon(x)\in\Delta_{W,F}$.  Conversely, if
$S_\varepsilon(x)$ is a face, retaining those coordinates and setting all
others to zero produces an element of $\mathcal D_F$ within distance
$\varepsilon$.

For a uniform point in $B_\rho$, the probability that
$S_\varepsilon(x)=S$ is
$(1-\varepsilon/\rho)^{|S|}(\varepsilon/\rho)^{m-|S|}$.
The events are disjoint, so summing over the faces proves
Eq.~\eqref{eq:exact-tube-law}.  Its smallest power of $\varepsilon$ is
$m-\max_{S\in\Delta_{W,F}}|S|=b_W(F)$; all coefficients at that power are
positive.
\end{proof}

\begin{theorem}[Residual tolerance-volume exponent]
\label{thm:completion-volume}
Let $M=M_{W,F}$ be the minimal-generator vector and
\[
V_M(\eta)=\operatorname{vol}
\{x\in B_\rho:\|M(x)\|_\infty\le\eta\}.
\]
Then
\begin{equation}
\lim_{\eta\downarrow0}\frac{\log V_M(\eta)}{\log\eta}
=\lambda_W(F).
\label{eq:generator-volume-exponent}
\end{equation}
Moreover, on a sufficiently small candidate box there are positive constants
$c_i,C_i$ such that
\begin{align}
c_1\|M(x)\|_\infty
&\le\|(H_{\rm cyc})_F(x)\|_2
\le C_1\|M(x)\|_\infty, \label{eq:cycle-two-sided}\\
c_2\|M(x)\|_\infty^2
&\le h_f\left(W+\sum_sx_sE_s\right)
\le C_2\|M(x)\|_\infty^2             \label{eq:walk-two-sided}
\end{align}
for every locally uniformly convergent positive-walk scalar in
Proposition~\ref{prop:walk-sharp}.  Consequently, the sublevel-volume
exponents of the cycle vector and positive-walk scalar are respectively
$\lambda_W(F)$ and $\lambda_W(F)/2$.
\end{theorem}

\begin{proof}
Fixed coordinate rescaling changes only constants, so first take $\rho=1$.
By symmetry it suffices to work on $[0,1]^m$.  Put
$x_s=e^{-u_s}$ and $L=\log(1/\eta)$.  The coordinates $u_s$ have independent
unit-exponential density, and $\|M(x)\|_\infty\le\eta$ is equivalent to
\begin{equation}
u\ge0,\qquad \sum_{s\in C}u_s\ge L
\quad\text{for every }C\in\mathcal C_{\rm min}.
\label{eq:log-cover-event}
\end{equation}
Let $z^\star$ solve Eq.~\eqref{eq:fractional-completion-cover}.
The orthant event $u_s\ge Lz^\star_s$ for all $s$ is contained in
Eq.~\eqref{eq:log-cover-event} and has probability
$e^{-L\sum_sz^\star_s}=e^{-\lambda L}$.

For the reverse bound, linear-programming duality gives $y_C\ge0$ satisfying
\begin{equation}
\sum_Cy_C=\lambda,\qquad
\sum_{C\ni s}y_C\le1\quad\text{for every }s.
\end{equation}
On the event in Eq.~\eqref{eq:log-cover-event},
\[
\sum_su_s
\ge\sum_su_s\sum_{C\ni s}y_C
=\sum_Cy_C\sum_{s\in C}u_s
\ge\lambda L.
\]
(The first inequality uses $u_s\ge0$.)  Since $\sum_su_s$ has a
Gamma$(m,1)$ distribution,
\begin{equation}
e^{-\lambda L}
\le \Pr\{\|M(x)\|_\infty\le e^{-L}\}
\le e^{-\lambda L}\sum_{j=0}^{m-1}\frac{(\lambda L)^j}{j!}.
\label{eq:volume-sandwich}
\end{equation}
Orthant factors and the fixed radius $\rho$ do not affect the logarithmic
limit.  Dividing logarithms by $\log\eta=-L$ proves
Eq.~\eqref{eq:generator-volume-exponent}.

The lower inequalities in Eqs.~\eqref{eq:cycle-two-sided} and
\eqref{eq:walk-two-sided} were proved in
Proposition~\ref{prop:general-holder}.  For the cycle-vector upper bound,
every nonzero restricted simple-cycle product contains a minimal completion
$C$ among its candidate edges.  On a box with $\|x\|_\infty\le1$, its
absolute value is at most a fixed base-edge coefficient times $|x^C|$.
There are finitely many simple cycles, which gives $C_1$.

For the scalar upper bound, every nonzero closed-walk monomial likewise
contains a minimal completion and is bounded by a fixed coefficient times
$|x^C|^2$.  More explicitly, choose a slightly larger candidate polydisc
whose closure remains in the convergence domain and assign each walk to one
minimal completion dividing its monomial.  Absolute convergence on the larger
polydisc implies absolute convergence of the finitely many termwise-divided
series on the smaller box.  Their suprema sum to a finite $C_2$.  The
two-sided comparisons sandwich each residual sublevel set between generator
sublevel sets with constant-rescaled tolerances.
Equation~\eqref{eq:generator-volume-exponent} then yields $\lambda$ for the
vector and $\lambda/2$ for the scalar.
\end{proof}

\begin{remark}[A neighborhood invariant, not one origin ray]
An optimal fractional cover may have zero coordinates.  The associated
valuation then approaches a higher-dimensional acyclic support stratum inside
$B_\rho$ rather than the all-zero candidate point.  This is intended:
$\lambda$ describes the volume of a neighborhood of the full local coordinate
arrangement $\mathcal D_F\cap B_\rho$, whereas $q$ and $r$ are witnessed by
paths into the distinguished base point.
\end{remark}

\begin{corollary}[Four scales of the completion ideal]
\label{cor:four-scales}
The invariants satisfy
\begin{equation}
\frac{b_W(F)}{r_W(F)}\le\lambda_W(F)\le b_W(F).
\label{eq:cover-gap}
\end{equation}
Here $q$ is the first possible jet degree, $r$ controls the worst H\"older
error-bound exponent, $b$ is the coordinate-tube codimension, and $\lambda$
is the residual tolerance-volume exponent.
\end{corollary}

\begin{proof}
The indicator of an integral completion cover is feasible in
Eq.~\eqref{eq:fractional-completion-cover}, so $\lambda\le b$.  Conversely,
given any feasible $z$, the set
$T=\{s:z_s\ge1/r\}$ meets every minimal completion: otherwise some
$C$ of size at most $r$ would satisfy $\sum_{s\in C}z_s<1$.
Thus $b\le|T|\le r\sum_sz_s$.  Minimize over $z$.
\end{proof}

\subsection{Statistical desingularization}
\label{app:statistical-desingularization}

\paragraph{Proof of Theorem~\ref{thm:statistical-scaling}.}
Write $A=W\circ W$.  Every term of
$\tr(A^k)$ is a positive closed-walk monomial.  After restricting to the
absent coordinates $F$, its candidate degree is twice the number of candidate
edges traversed with multiplicity.  Theorem~\ref{thm:jet-ideal} rules out
candidate degree below $2q$.  At degree $2q$, the candidate support of a
closed walk contains a cycle completion of size at least $q$, so it has
exactly $q$ distinct candidate edges and traverses each once.  Removing any
one of them breaks all cycles in that support; otherwise a smaller completion
would occur.  Its candidate factor is therefore
$\prod_{e\in C}x_e^2$ for a size-$q$ minimal completion $C$.

Conversely, every minimal completion $C$ contains a simple directed cycle in
$\supp(W)\cup C$.  The cycle uses every edge of $C$, since otherwise a proper
subset would complete a cycle.  Its fixed-edge product is nonzero, and the
coefficient $c_k$ at its length is positive.  Summing all degree-$2q$ closed
walks therefore gives
\begin{equation}
h_{f,F}(x)=
\sum_{\substack{C\in\mathcal C_{\rm min}(W,F)\\|C|=q}}
a_C\prod_{e\in C}x_e^2+o(\|x\|^{2q}),
\qquad a_C>0.
\label{eq:appendix-statistical-initial-form}
\end{equation}
Let $Z_n=\sqrt n(\widehat W_n-W)$.  Substitution into
Eq.~\eqref{eq:appendix-statistical-initial-form} yields
\begin{equation}
n^qh_f(\widehat W_n)=P_{W,q}((Z_n)_F)+o_p(1).
\end{equation}
The continuous mapping theorem gives Eq.~\eqref{eq:scalar-null-law}.
The fluctuations on fixed active edges perturb the coefficients $a_C$ by
$o_p(1)$ and hence disappear by Slutsky's theorem.  Under the local sequence,
$\sqrt n(\widehat W_n-W)=Z_n+U+o_p(1)$, proving the shifted claim.

\paragraph{Proof of Corollary~\ref{cor:no-universal-tolerance}.}
The acceptance event is
\begin{equation}
n^qh_f(\widehat W_n)\le c n^{q-a}.
\end{equation}
The right side tends to zero, $c$, or infinity according as $a>q$, $a=q$,
or $a<q$.  The limiting polynomial is positive almost surely because its
coefficients are positive and a nondegenerate Gaussian has no zero
coordinate.  This proves Eq.~\eqref{eq:tolerance-trichotomy}.  The
degree-$2q$ initial form contains only size-$q$ generators.  A local shift
supported on a disjoint size-$r$ generator with $r>q$ leaves this initial form
unchanged, which proves the final statement.  Disjointness is sufficient, not
necessary; without it, a shifted coordinate may also belong to a shortest
generator.

\section{Symmetry and the Score--Topology Crossover}
\label{app:universality-proofs}

\subsection{Coordinate realization and homology}

\paragraph{Proof of Theorem~\ref{thm:coordinate-universality}.}
Let \(L=\operatorname{sd}K\).  Its vertices are the nonempty faces of \(K\),
and its simplices are chains under inclusion.  In particular, \(L\) is flag:
a set of vertices is a face exactly when every pair is comparable.  Let \(H\)
be the graph whose vertices are those of \(L\), with an edge between two
incomparable faces.  Then
\begin{equation}
\operatorname{Ind}(H)=L.
\label{eq:appendix-independence-subdivision}
\end{equation}

For each \(i\in V(H)\), create graph nodes \(a_i,b_i\) and candidate edge
\(e_i:a_i\to b_i\).  For every \(\{i,j\}\in E(H)\), add fixed edges
\(b_i\to a_j\) and \(b_j\to a_i\).  Every fixed edge goes from the
\(b\)-layer to the \(a\)-layer, so the fixed graph \(W_0\) is acyclic.

If \(e_i,e_j\) are active for \(\{i,j\}\in E(H)\), then
\begin{equation}
a_i\to b_i\to a_j\to b_j\to a_i
\label{eq:appendix-four-cycle}
\end{equation}
is a directed cycle.  Conversely, every directed cycle must alternate between
a candidate edge \(a_i\to b_i\) and a fixed edge \(b_i\to a_j\).  Any such
fixed edge certifies \(\{i,j\}\in E(H)\), while the cycle certifies that both
candidates are active.  Thus an active candidate set is acyclic if and only
if it is independent in \(H\).  Equation~\eqref{eq:appendix-independence-subdivision}
proves
\(\Delta_{W_0,F}\cong\operatorname{sd}K\).  Radial projection identifies the
positive normalized coordinate faces with the usual geometric realization,
and barycentric subdivision preserves that realization.

\paragraph{Proof of Corollary~\ref{cor:exponential-homology}.}
Let \(H_t=\bigsqcup_{b=1}^tK_5\) and apply the preceding construction.
There are \(5t\) candidates and \(10t\) graph nodes.  Independence complexes
turn disjoint graph unions into simplicial joins:
\begin{equation}
\operatorname{Ind}(H_t)=\operatorname{Ind}(K_5)^{*t}.
\end{equation}
The complex \(\operatorname{Ind}(K_5)\) is a set of five points, so its only
nonzero reduced homology has rank four in degree zero.  The reduced join
formula
\begin{equation}
\widetilde H_r(A*B)
\cong\bigoplus_{i+j=r-1}
\widetilde H_i(A)\otimes\widetilde H_j(B)
\end{equation}
shows inductively that the \(t\)-fold join has reduced homology of rank
\(4^t\) in degree \(t-1\) and zero in all other degrees.  Equivalently it is
homotopy equivalent to a wedge of \(4^t\) copies of \(S^{t-1}\).

\subsection{Far-field symmetry and invariant flows}

\paragraph{Proof of Theorem~\ref{thm:symmetry-blindness}.}
Use \(t\) disjoint coordinate gadgets for \(K_m\).  Write the candidate
restriction as
\begin{equation}
g(x)=h\!\left(W_0+\sum_{b=1}^t\sum_{i=1}^m
x_{bi}E_{bi}\right).
\label{eq:appendix-candidate-restriction}
\end{equation}
For each block choose a pair \(S_b=\{i_b,j_b\}\) and define
\begin{equation}
\mathcal M_S=
\left\{x:\ x_{b i_b}=x_{b j_b}=a_b>0,\ 
x_{b\ell}=0\ \text{for }\ell\notin S_b,\ b\in[t]\right\}.
\label{eq:appendix-blind-manifold}
\end{equation}
Each block contains the four-cycle in
Eq.~\eqref{eq:appendix-four-cycle}, hence every point of
\(\mathcal M_S\) is cyclic.  The theorem's domain assumption ensures that
the restriction and its derivatives are defined on a neighborhood of every
such manifold.

Because \(h(W)=\Phi(W\circ W)\), \(g\) is even in every candidate
coordinate.  Therefore
\begin{equation}
\partial_{b\ell}g(x)=0
\qquad(\ell\notin S_b,\ x\in\mathcal M_S).
\label{eq:appendix-inactive-zero}
\end{equation}
The node permutation exchanging simultaneously
\(a_{i_b}\leftrightarrow a_{j_b}\) and
\(b_{i_b}\leftrightarrow b_{j_b}\) is an automorphism of the \(K_m\)
gadget.  It fixes every point of \(\mathcal M_S\).  Relabeling invariance of
\(h\) therefore gives
\begin{equation}
\partial_{b i_b}g(x)=\partial_{b j_b}g(x).
\label{eq:appendix-active-tie}
\end{equation}
Equations~\eqref{eq:appendix-inactive-zero}--\eqref{eq:appendix-active-tie}
say exactly that the block gradient is proportional to \(x_b\).
There are \(\binom m2\) choices in each block, and distinct choices have
different algebraic supports, proving Eq.~\eqref{eq:blind-manifold-count}.

Now suppose \(\nabla g\) and the scalar function \(\omega\) are locally
Lipschitz.  The vector field \(-\omega\nabla g\) satisfies the same
inactive-zero and active-tie identities, hence is tangent to
\(\mathcal M_S\).  On that manifold the common
amplitudes solve a locally Lipschitz system
\begin{equation}
\dot a_b=F_b(a),\qquad F_b(0,a_{-b})=0.
\label{eq:appendix-amplitude-flow}
\end{equation}
On any compact time interval there is \(L<\infty\) with
\(|F_b(a)|\le L|a_b|\).  Gronwall's inequality yields
\begin{equation}
a_b(t)\ge a_b(0)e^{-Lt}>0
\end{equation}
whenever \(a_b(0)>0\), on every compact subinterval of the solution's maximal
interval of existence.  No active coefficient reaches zero at a finite time
in that interval, and every block retains its four-cycle while the solution
exists.

For \(m=5\), \(d=10t\) and the number of manifolds is \(10^t\).
At \(a_b=1/2\), all active coefficients are fixed away from zero.  For
\(A=W\circ W\), the restriction of \(A^2\) to the two active \(a\)-nodes has
eigenvalues \(\pm a_b^2\).  Hence
\(\rho(A)^2=\rho(A^2)=a_b^2\), so \(\rho(A)=a_b=1/2<1\), strictly inside
DAGMA's \(s=1\) domain.

\paragraph{Proof of Proposition~\ref{prop:symmetry-time}.}
Put \(p=\prod_{i=1}^Lz_i^2\).  On the positive orthant,
\begin{equation}
\partial_{z_i}h=2p\phi'(p)/z_i,
\qquad
\dot z_i=-2\Psi'(\phi(p))p\phi'(p)/z_i.
\label{eq:full-cycle-flow}
\end{equation}
It follows immediately that
\begin{equation}
\frac{d}{dt}z_i^2
=-4\Psi'(\phi(p))p\phi'(p)
\label{eq:equal-square-speed}
\end{equation}
is independent of \(i\).  Hence every \(z_i^2-z_j^2\) is conserved.
In particular, \(\Delta=x^2-y^2=4a\varepsilon\).
Conservation of squared-norm differences is a familiar balancing mechanism
in homogeneous gradient flows \citep{du2018algorithmic}; the calculation
below uses the exact cycle product to obtain the sharp DAG-selection time.

Write \(u=y^2\) and
\[
d_{j,\varepsilon}=b_j^2-(a-\varepsilon)^2,\qquad
p_\varepsilon(u)=u(u+\Delta)
                  \prod_{j=1}^{L-2}(u+d_{j,\varepsilon}).
\]
Equation~\eqref{eq:equal-square-speed} becomes
\begin{equation}
\dot u=-4\Psi'(\phi(p_\varepsilon(u)))
           p_\varepsilon(u)\phi'(p_\varepsilon(u)).
\label{eq:full-cycle-u-flow}
\end{equation}
The event \(y/x=r\) occurs at
\(u_r=r^2\Delta/(1-r^2)\).  With \(u=\Delta v\), separation of variables
gives
\begin{equation}
\begin{split}
T_r(\varepsilon)
=\frac{\Delta}{4}
\int_{v_r}^{(a-\varepsilon)^2/\Delta}
\frac{dv}{
\Psi'(\phi(p_\varepsilon(\Delta v)))
p_\varepsilon(\Delta v)
\phi'(p_\varepsilon(\Delta v))},\\
v_r=\frac{r^2}{1-r^2}.
\end{split}
\label{eq:full-cycle-selection-integral}
\end{equation}
For fixed \(v\),
\[
p_\varepsilon(\Delta v)
=\Delta^2v(v+1)
  \prod_{j=1}^{L-2}(\Delta v+d_{j,\varepsilon}),
\qquad
d_{j,\varepsilon}\longrightarrow b_j^2-a^2.
\]
Let \(\alpha=\phi'(0)\) and
\(C=\prod_j(b_j^2-a^2)\).  The assumptions on \(\Psi'\) and \(\phi'\) imply
\[
\Psi'(\phi(p))p\phi'(p)
=c\alpha^{\nu+1}p^{\nu+1}(1+o(1)).
\]
After multiplying Eq.~\eqref{eq:full-cycle-selection-integral} by
\(\Delta^{2\nu+1}\), its integrand therefore converges pointwise to
\[
\frac{1}{
4c(\alpha C)^{\nu+1}[v(v+1)]^{\nu+1}}.
\]

For completeness, this limit does not interchange an uncontrolled moving
endpoint.  Choose a small fixed \(u_0>0\).  On
\(u\in[0,u_0]\), the asymptotic relations above hold with two-sided uniform
bounds, and the normalized integrand is dominated by a constant multiple of
\([v(v+1)]^{-(\nu+1)}\), which is integrable on
\([v_r,\infty)\).  On \(u\in[u_0,(a-\varepsilon)^2]\), all factors in
Eq.~\eqref{eq:full-cycle-u-flow} are bounded away from zero, so that part of
the unnormalized hitting time is \(O(1)\) and vanishes after multiplication
by \(\Delta^{2\nu+1}\).  Dominated convergence now yields
\[
\lim_{\varepsilon\downarrow0}
\Delta^{2\nu+1}T_r(\varepsilon)
=\frac{I_\nu(r)}{4c[\alpha C]^{\nu+1}}.
\]
Substitution of \(\Delta=4a\varepsilon\) proves
Eq.~\eqref{eq:symmetry-time}.

For the score statement, the full flow satisfies
\begin{equation}
\frac{d}{dt}\log\frac{y}{x}
=-\left(
\frac{\partial_y[S+\Psi(h)]}{y}
-\frac{\partial_x[S+\Psi(h)]}{x}
\right)
=-m_{\rm sc}(z).
\label{eq:log-ratio-margin}
\end{equation}
If \(m_{\rm sc}(z)\ge\gamma_{\rm flow}\) until the stopping event, integration gives
Eq.~\eqref{eq:score-time}.  This argument does not assume that the preferred
edge is initially smaller.

It remains to prove the crossover claim for
\(S_\gamma(z)=\gamma y^2/2\).  Put
\[
A(p)=\Psi'(\phi(p))p\phi'(p).
\]
The squared-weight equations become
\begin{equation}
\frac{d}{dt}x^2=-4A(p),\qquad
\frac{d}{dt}y^2=-4A(p)-2\gamma y^2,
\label{eq:score-crossover-square-flow}
\end{equation}
while every unpenalized cycle edge has the same squared speed as \(x\).
Consequently,
\begin{equation}
\frac{d}{dt}\log\frac{y}{x}
=-\gamma-2A(p)\left(\frac1{y^2}-\frac1{x^2}\right)
\le-\gamma.
\label{eq:score-crossover-log-flow}
\end{equation}
This proves
\[
T_r(\varepsilon,\gamma)
\le\gamma^{-1}\log\frac{y_0}{rx_0},
\]
and therefore Eq.~\eqref{eq:score-limited-regime} whenever
\(\gamma T_0\to\infty\).

For the other regime, let \(\Delta=x_0^2-y_0^2=4a\varepsilon\),
write \(q=x^2\), \(v=y^2\), and note that every other squared cycle weight
equals \(q+d_{j,\varepsilon}\) for a constant
\(d_{j,\varepsilon}\to b_j^2-a^2>0\).  On the bottleneck scale
\[
q=\Delta Q,\qquad v=\Delta V,qquad
\tau=\Delta^{2\nu+1}t,
\]
Eq.~\eqref{eq:score-crossover-square-flow} has, on compact subsets of the
positive \((Q,V)\) quadrant, the uniform limit
\begin{align}
\frac{dQ}{d\tau}
&=-4c[\phi'(0)C]^{\nu+1}(QV)^{\nu+1},
\label{eq:score-crossover-limit-q}\\
\frac{dV}{d\tau}
&=-4c[\phi'(0)C]^{\nu+1}(QV)^{\nu+1}
  -2\eta V,
\label{eq:score-crossover-limit-v}
\end{align}
where \(\eta=\gamma/\Delta^{2\nu+1}\).  The pure-flow calculation above
shows that \(T_0\Delta^{2\nu+1}\) converges to a finite positive constant.
Thus \(\gamma T_0\to0\) is equivalent, up to a convergent positive factor,
to \(\eta\to0\).

The entrance from \(Q,V\to\infty\) is where the cases \(\nu>0\) and
\(\nu=0\) differ.  Fix \(M\), and let \(\tau_M\) be the first scaled time at
which \(Q=M\).  When \(\nu>0\), integration of the exact power-potential
identity bounds the score-induced change of \(Q-V\) before \(\tau_M\) by
\(O(\eta M^{-2\nu})\), uniformly in the initial height.  Before the target
ratio is reached, the entrance time is bounded by a constant multiple of
\(M^{-(2\nu+1)}\).  Hence the entrance state converges to \((M,M-1)\) for
every fixed \(M\), and the time omitted above that section vanishes uniformly
as \(M\to\infty\).

At \(\nu=0\), write the exact scaled equations before the section as
\[
Q'=-4K_\varepsilon QV,\qquad
V'=-4K_\varepsilon QV-2\eta V,
\]
where the scalar-speed assumptions give \(K_\varepsilon\ge k>0\) for all
sufficiently small \(\varepsilon\).  With \(D=Q-V\),
\[
D'=2\eta V,
\qquad
-(\log Q)'=4K_\varepsilon V\ge4kV.
\]
Since \(D(0)=1\), integration to \(\tau_M\) gives the entrance estimate
\begin{equation}
0\le D(\tau_M)-1
\le \frac{\eta}{2k}\log\frac{Q(0)}{M}.
\label{eq:critical-log-entrance}
\end{equation}
Here \(Q(0)=x_0^2/\Delta=\Theta(\Delta^{-1})\).  The condition
\(\gamma T_0\log(1/\varepsilon)\to0\) is therefore equivalent at this scale
to \(\eta\log(1/\Delta)\to0\), and
Eq.~\eqref{eq:critical-log-entrance} again sends the entrance state to
\((M,M-1)\).  Moreover, while \(V/Q\ge r^2\),
\(-Q'\ge4kr^2Q^2\), so \(\tau_M\le(4kr^2M)^{-1}\).  This entrance envelope
vanishes as \(M\to\infty\).

On the compact corridor below \(Q=M\), uniform convergence of the vector
fields, continuous dependence of solutions, and transversality of the target
ratio give convergence to the \(\eta=0\) hitting time.  Taking first
\(\varepsilon\downarrow0\) and then \(M\to\infty\) proves
Eq.~\eqref{eq:topology-limited-regime}.  For \(\nu>0\), combining this result
with Eq.~\eqref{eq:symmetry-time} gives the power scale in
Eq.~\eqref{eq:score-topology-crossover}.

The additional logarithm at \(\nu=0\) cannot be removed.  In the constant
critical system \(K_\varepsilon\equiv1\), direct differentiation gives the
first integral
\[
V-Q-\frac{\eta}{2}\log Q=\text{constant}.
\]
For example, take \(\Delta_n=\exp(-(n+1)^2)\) and
\(\eta_n=(n+1)^{-1}\).  Then \(\eta_n\to0\), but
\(\eta_n\log(1/\Delta_n)\to\infty\).  The first integral forces the target
section to be reached at \(Q_{\rm hit,n}\to\infty\); since
\(T_n\le(4r^2Q_{\rm hit,n})^{-1}\), the normalized hitting time tends to
zero rather than to the positive feasibility-only limit.  This proves the
claimed endpoint failure under the weaker condition \(\gamma T_0\to0\).

Finally, on an isolated \(L\)-cycle, the only closed walks have lengths
\(\ell L\), and
\(\operatorname{tr}((W\circ W)^{\ell L})=Lp^\ell\).  Hence
\[
\phi_{\rm EXP}(p)
=L\sum_{\ell\ge1}\frac{p^\ell}{(\ell L)!},
\qquad
\phi_{\rm EXP}'(0)=\frac{1}{(L-1)!}.
\]
The same cycle block satisfies
\(\det(I-W\circ W)=1-p\), so
\(\phi_{\rm DAGMA}(p)=-\log(1-p)\) and
\(\phi_{\rm DAGMA}'(0)=1\).  The linear scalarization has
\(\Psi'(s)=1\), the cold quadratic penalty has \(\Psi'(s)=s\), and a seeded
ALM term \(\Psi(s)=\mu s+\rho s^2/2\) with \(\mu>0\) has
\(\Psi'(s)\sim\mu\). \qed

\paragraph{Proof of Corollary~\ref{cor:angular-criticality}.}
Use one \(K_m\) gadget.  For every \(S\subseteq[m]\) of size \(k\ge2\), set
\begin{equation}
x_i=\frac{1}{\sqrt{k}}\mathbf 1\{i\in S\}.
\end{equation}
All active coordinates are related by automorphisms and all inactive
derivatives vanish by evenness.  This is the elementary finite-group case of
the symmetric-criticality principle \citep{palais1979symmetric}.  Thus
\(\nabla g(x)=\lambda x\), the Lagrange equation on the unit sphere.
Every such support is cyclic, and there are
\(\sum_{k=2}^m\binom mk=2^m-m-1\) of them.

Giving each active coordinate an arbitrary sign preserves the Lagrange
equation because \(g\) is coordinatewise even.  This gives
\begin{equation}
\sum_{k=2}^m\binom mk2^k=3^m-2m-1
\end{equation}
signed critical points.  For DAGMA, with \(A=W\circ W\), the active block of
\(A^2\) is \(k^{-1}(J_k-I_k)\).  Thus
\(\rho(A)=\sqrt{(k-1)/k}<1\).

\subsection{Score--topology flow audit}
\label{app:score-topology-flow-audit}

Figure~\ref{fig:full-score-flow-audit} reports the full-matrix checks that
complement the normalized crossover experiment in the main text.  The
population score has a positive pathwise margin in every audited setting and
selects the false back edge even when that edge starts larger.  Sampling can
reverse this margin, which is why the finite-data experiment does not attain
perfect selection.  This distinction motivates the confidence set in
Section~\ref{sec:completion-margin}: the controlled score isolates the
dynamical scale, while the holdout certificate asks whether an empirical
score resolves it.

\begin{figure}[H]
    \centering
    \includegraphics[width=0.88\textwidth]{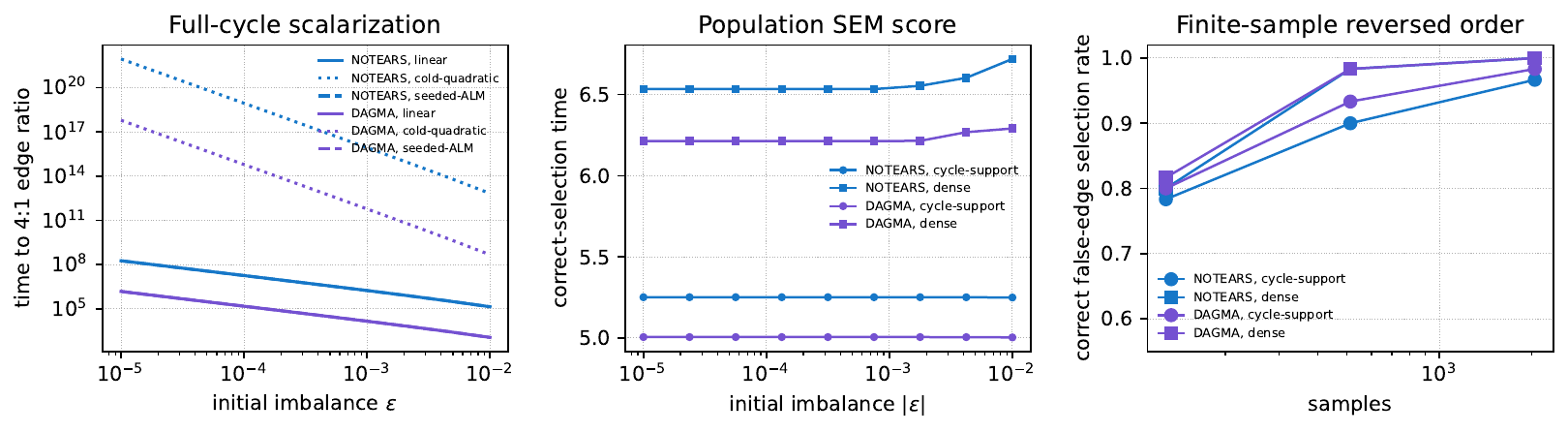}
    \caption{Full-matrix checks of Proposition~\ref{prop:symmetry-time}.
    Feasibility-only linear and seeded-ALM flows scale as
    \(\varepsilon^{-1}\), while the cold quadratic scales as
    \(\varepsilon^{-3}\) (left).  A population linear-SEM score removes the
    divergence for reversed initial rankings and dense flows (center).  The
    same event is not guaranteed from an empirical covariance, although its
    frequency increases with sample size (right).}
    \label{fig:full-score-flow-audit}
\end{figure}

\subsection{Hard-threshold support stability}
\label{app:threshold-proof}

\begin{proposition}[Stability of hard-threshold support]
\label{prop:threshold-externalization}
For $\theta>0$, define
\begin{equation}
\mathcal S_\theta(W)=\{(i,j):i\ne j,\ |W_{ij}|>\theta\},\qquad
m_\theta(W)=\min_{i\ne j}\bigl||W_{ij}|-\theta\bigr|.
\label{eq:threshold-margin}
\end{equation}
The support $\mathcal S_\theta(W)$ is unchanged for every off-diagonal
perturbation $\|\Delta\|_\infty<m_\theta(W)$; conversely, for every
$\eta>0$ some $\Delta$ with
$\|\Delta\|_\infty\le m_\theta(W)+\eta$ changes it.  Moreover, a positive
diagonal rescaling preserves algebraic support but need not preserve
$\mathcal S_\theta(W)$.
\end{proposition}

\begin{proposition}[Scale invariance of profile-score event paths]
\label{prop:profile-invariance}
For centered data $X$, define
\begin{equation}
\Delta_{i\to j}(P)=\tfrac12\log
\frac{{\rm RSS}_j(P\setminus\{i\})}{{\rm RSS}_j(P)}.
\label{eq:profile-deletion}
\end{equation}
Under $X'=XD$ with positive diagonal $D$, every positive-RSS deletion score
is unchanged.  Conditional on the same starting algebraic support, repeatedly
deleting the minimum-score edge returned by a complete cycle oracle, with
fixed tie breaking, produces the same event path in both units.
\end{proposition}

We prove Proposition~\ref{prop:threshold-externalization}.  For every
off-diagonal coordinate, the reverse triangle inequality gives
\begin{equation}
\bigl||W_{ij}+\Delta_{ij}|-|W_{ij}|\bigr|
\le |\Delta_{ij}|<m_\theta(W)
\le\bigl||W_{ij}|-\theta\bigr|.
\label{eq:threshold-reverse-triangle}
\end{equation}
Thus $|W_{ij}+\Delta_{ij}|-\theta$ has the same sign as
$|W_{ij}|-\theta$ for every coordinate, and the strict threshold support is
unchanged.

Conversely, choose a coordinate attaining the finite minimum
$m=m_\theta(W)$.  If $|W_{ij}|>\theta$, perturb opposite to its sign by
$m+\eta'$; if $|W_{ij}|\le\theta$, perturb along its sign (choosing either
sign at zero) by $m+\eta'$, where $0<\eta'\le\eta$ is small enough to avoid
crossing zero in the first case.  The chosen coordinate crosses the strict
threshold and the perturbation norm is at most $m+\eta$.  Hence $m$ is the
supremum of open $\ell_\infty$ balls on which support is constant.  Taking
one coefficient equal to $\theta$ shows $\inf_Wm_\theta(W)=0$.

For the scaling claim, let $D=\operatorname{diag}(d_1,\ldots,d_d)$ with
$d_i>0$.  Then
\begin{equation}
(D^{-1}WD)_{ij}=\frac{d_j}{d_i}W_{ij}.
\label{eq:diagonal-rescale}
\end{equation}
Every multiplier is positive, so exact zero support, directed cycles, and
acyclicity are preserved.  For any selected nonzero coefficient, however,
$d_j/d_i$ can be any positive number.  Choosing it above or below
$\theta/|W_{ij}|$ moves that coefficient across the fixed threshold.  This
proves non-invariance.  If standardization is prescribed and repeated after
rescaling, this particular source of variation disappears within that fixed
pipeline.  The proposition makes the narrower statement that a numerical
cutoff depends on the chosen coordinate convention, and that its robustness
within any such convention is governed by $m_\theta(W)$.

We next prove Proposition~\ref{prop:profile-invariance}.  Let $P$ be a parent
set for node $j$, and let $\Pi_P$ denote orthogonal projection onto the column
span of $X_P$.  Under $X'=XD$, the response becomes $X'_j=d_jX_j$, while the
design becomes $X'_P=X_PD_P$.  Because $D_P$ is invertible,
$\operatorname{col}(X'_P)=\operatorname{col}(X_P)$ and hence the two designs
have the same projection operator.  Therefore
\begin{equation}
 {\rm RSS}'_j(P)
 =\|(I-\Pi_P)d_jX_j\|_2^2
 =d_j^2{\rm RSS}_j(P).
\label{eq:profile-rss-rescale}
\end{equation}
The same factor $d_j^2$ multiplies the numerator and denominator of
Eq.~\eqref{eq:profile-deletion}, so every positive-RSS deletion score is
unchanged.  Conditional on one fixed algebraic support, diagonal rescaling
also preserves every directed cycle.  Thus the cycle oracle returns the same
candidate edges.  With deterministic tie breaking, equality of all candidate
scores implies equality of the first deletion.  Induction gives the same
support and candidate set after every later deletion, hence the same event
path and final certified DAG.  This proves the proposition.  The conditioning
is essential: the claim does not assert that a statistical frontend produces
the same support under rescaling.

\subsection{Relation to prior work and alternative assumptions}
\label{app:closest-work}

Table~\ref{tab:closest-work} compares the assumptions and conclusions of the
closest results.  The algebraic, topological, and inferential tools listed in
the middle column are classical.  The right column records what follows after
they are tied to cycle-completion geometry or to the ordinary-coordinate
symmetry construction.

\begin{table}[p]
\centering
\scriptsize
\renewcommand{\arraystretch}{0.98}
\caption{Closest prior results and the boundary of the present claims.}
\label{tab:closest-work}
\begin{tabular}{p{0.19\textwidth}p{0.28\textwidth}p{0.40\textwidth}}
\toprule
Work & Established there & Distinction here \\
\midrule
\citet{hultman2004directed}
& The simplicial complex of acyclic edge sets and its homotopy type.
& We condition on a fixed base DAG and restrict to candidate coordinates.
These local sections realize arbitrary finite homotopy types, despite the
restricted homotopy type of the full complex. \\
\addlinespace
\citet{ehrenborg2006topology}
& Every finite complex becomes an independence complex after barycentric
subdivision.
& This fact is an ingredient in an ordinary-adjacency reduction from graph
independence complexes to local DAG sections; the reduction yields the
homology and gradient consequences used here. \\
\addlinespace
\citet{howald2001multiplier}
& Multiplier ideals and log-canonical thresholds of monomial ideals from
their Newton polyhedra.
& The classical fractional-cover formula, together with an elementary
sublevel-volume sandwich, identifies the continuous DAG tolerance exponent. \\
\addlinespace
\citet{drton2016wald,sturma2024testing}
& Nonstandard inference for singular polynomial restrictions and bootstrap
tests for many irregular constraints.
& Theorem~\ref{thm:statistical-scaling} supplies the problem-specific part:
the exact singular degree and initial polynomial determined by DAG completion
topology. \\
\addlinespace
\citet{wei2020fears}
& Gradient degeneracy for positive polynomial constraints; absolute-value KKT conditions and local search.
& We cover arbitrary smooth exact maps on candidate subspaces, identify the topology-dependent $q$ versus $2q$ orders, and bound vector output dimension.  Their local search is prior art for the hybrid category. \\
\addlinespace
\citet{zhang2025analytic}
& Positive-coefficient analytic functions that yield exact DAG constraints,
closure rules for constructing new constraints, and efficient evaluation.
& These families furnish examples to which the representation-level lower
orders and edge-ranking obstructions apply on the candidate restrictions
considered here. \\
\addlinespace
\citet{bello2022dagma}
& An exact M-matrix log-determinant constraint, invexity, a
reachability-dependent Hessian, and central-path optimization.
& We preserve its Euclidean invexity result.  We prove exact $2q$ local order
and exponentially many nonzero cyclic strata with zero tangential
edge-ranking signal; these are not Euclidean stationary points. \\
\addlinespace
\citet{du2018algorithmic}
& Squared-norm differences are conserved by gradient flow in homogeneous
multi-layer models.
& On an exact DAG-cycle restriction, this invariant reduces the flow to a
one-dimensional product dynamics and yields the
scalarization-dependent \(\varepsilon^{-(2\nu+1)}\) DAG-selection law and the
score-margin escape condition. \\
\addlinespace
\citet{shridharan2025beta}
& $\beta$th-order coordinate derivatives and local search for the fixed family $h(|W|^\beta)$.
& The unavoidable order is instead the joint completion statistic $q_W(F)$ for any smooth exact map; it can grow beyond fixed $\beta$.  We also separate vector from nonnegative scalar maps and prove a dimension lower bound. \\
\addlinespace
\citet{ng2024sober}
& Empirical and optimization consequences of sparsity, nonconvexity, and thresholding.
& We prove that no positive continuous tolerance certifies support and quantify the local detection order before a threshold is applied. \\
\addlinespace
\citet{yu2021nocurl,massidda2024cosmo,gillot2022large}
& Direct acyclic parameterization/orientation and continuous--feedback-arc-set hybrid optimization.
& These methods lie outside a fixed smooth exact scalar.  The invariant-flow
result explains why a nonsmooth, discrete, or score-asymmetric mechanism can
supply support-selection information absent from the feasibility gradient. \\
\addlinespace
\citet{rey2026nonnegative}
& A smooth constraint without the Hadamard square and a benign population landscape under nonnegative edge weights.
& Sparse DAGs lie on the boundary of the nonnegative orthant, so the open signed-neighborhood assumption used here does not apply. \\
\bottomrule
\end{tabular}
\end{table}

The bounds apply to local Taylor information in an exact representation near a
DAG support boundary.  Algorithms that use higher derivatives, change the
parameterization, modify support discretely, or stop before exact support fall
outside the corresponding assumptions.

\subsection{Scalarization and augmented Lagrangians}

For $p\ge1$, define $P_p(W)=\norm{H(W)}_2^p$.

\begin{proposition}[Sensitivity--smoothness frontier]
\label{prop:frontier}
If $H(W+tU)=t^qv+o(t^q)$ with $v\ne0$, then
$P_p(W+tU)=\norm{v}_2^p|t|^{pq}+o(|t|^{pq})$.  As a function of the
residual, $z\mapsto\norm z_2^p$ is nonsmooth at zero for $p=1$, is $C^1$
but has non-Lipschitz gradient for $1<p<2$, and has locally Lipschitz
gradient for $p\ge2$.
\end{proposition}

\begin{corollary}[Augmented-Lagrangian hierarchy]
\label{cor:alm}
For the vector ALM
$Q_{\lambda,\rho}=\lambda^\top H+\rho\norm{H}^2/2$, the zero-multiplier
initialization vanishes to order at least $2q$ along an equal-scale
completion; on a $q$-sharp path this order is exact, while a generic nonzero
multiplier exposes order $q$.  For a nonnegative scalar ALM
$Q_{\mu,\rho}=\mu h+\rho h^2/2$, the corresponding lower bounds are $4q$
for $\mu=0$ and $2q$ for $\mu\ne0$, with equality when $h$ has exact order
$2q$.  In these sharp cases, derivative orders are one lower.
\end{corollary}

\begin{corollary}[Order under residual-scaled multipliers]
\label{cor:endogenous}
Let $R(t)=O(t^p)$ and $DR(t)=O(t^{p-1})$.  If
$\lambda(t)=O(t^p)$ and $\rho=O(1)$, then the primal gradient of
$\lambda^\top R+\rho\norm{R}^2/2$, treating the multiplier as fixed within
the primal step, is $O(t^{2p-1})$.  Any fixed number of updates
$\lambda\leftarrow\lambda+\rho R$ from zero, with bounded penalties and
residual-comparable iterates, preserves this lower bound.  Recovering order
$p-1$ requires a multiplier that does not vanish with the local residual and
is not orthogonal to its leading coefficient.
\end{corollary}

\paragraph{Residual-power frontier.}
If $H(t)=t^qv+o(t^q)$, continuity of the norm and homogeneity give
\begin{equation}
\|H(t)\|_2^p=|t|^{pq}\|v+o(1)\|_2^p
=\|v\|_2^p|t|^{pq}+o(|t|^{pq}).
\end{equation}
For $z\ne0$,
$\nabla\|z\|_2^p=p\|z\|_2^{p-2}z$.  At zero, the norm is nonsmooth for
$p=1$; for $1<p<2$ this gradient extends continuously but its derivative
scales as $\|z\|^{p-2}$ and is unbounded; for $p\ge2$ it is locally
Lipschitz.  These are residual-space statements.  Composition with a
particular high-order $H$ can add pathwise regularity, but it cannot provide
a uniform Lipschitz guarantee over arbitrary residual directions.

Assume a sharp path $H(t)=t^qv+o(t^q)$ with $v\ne0$.  Then
\begin{equation}
\frac12\|H(t)\|_2^2=\frac12\|v\|_2^2t^{2q}+o(t^{2q}).
\end{equation}
More generally, Theorem~\ref{thm:vector} always yields the lower-order
bound $O(t^{2q})$.  For
$Q_{\lambda,\rho}=\lambda^\top H+\rho\|H\|_2^2/2$, zero initialization
$\lambda=0$ therefore removes the order-$q$ term.  A multiplier with
$\lambda^\top v\ne0$ restores it; this is the meaning of ``generic'' in
Corollary~\ref{cor:alm}.

For a sharp nonnegative scalar $h(t)=at^{2q}+o(t^{2q})$, $a>0$,
\begin{equation}
Q_{\mu,\rho}(t)=\mu a t^{2q}+\frac\rho2a^2t^{4q}
+\mu\,o(t^{2q})+o(t^{4q}).
\end{equation}
Thus $\mu=0$ has order $4q$, while every nonzero $\mu$ has order $2q$.
Differentiation lowers each sharp order by one.  If the representation has
higher-than-minimal order, these statements remain lower bounds on the first
possible nonzero term.

\paragraph{Residual-generated multipliers.}
Let $R(t)=O(t^p)$ and $DR(t)=O(t^{p-1})$.  During a primal update the
multiplier is fixed, so
\begin{equation}
\nabla_tQ_{\lambda,\rho}
=DR(t)^\top\lambda+\rho DR(t)^\top R(t).
\end{equation}
If $\lambda=O(t^p)$ and $\rho=O(1)$, both terms are
$O(t^{2p-1})$.  Starting from zero, a fixed number of updates
$\lambda_{s+1}=\lambda_s+\rho_sR(t_s)$ with $t_s=\Theta(t)$ and bounded
$\rho_s$ leaves $\lambda_s=O(t^p)$ by induction.  An order-$p-1$ term can
appear only if the multiplier has a nonvanishing leading scale and nonzero
inner product with the leading residual coefficient.  Substituting $p=q$
for a vector representation and $p=2q$ for a nonnegative scalar proves
Corollary~\ref{cor:endogenous}.  A penalty schedule that diverges as
$t\to0$ can counteract the coefficient, but then does so by ill-conditioning,
not by restoring a missing Taylor term.

\subsection{Finite precision and restricted correction time}
\label{app:numerical-consequences}

If $|R(t)|\le C|t|^p$, then $|R(t)|<\varepsilon$ whenever
\begin{equation}
|t|<r_\varepsilon:=\left(\frac{\varepsilon}{C}\right)^{1/p}.
\label{eq:visibility-radius}
\end{equation}
This interval follows directly from the stated bound.  Its numerical value
depends on the formulation-specific constant $C$ and on the comparison rule.
For fixed $C$ and $0<\varepsilon<C$, the radius increases to one as
$p\to\infty$.

For the gradient-flow calculation, assume explicitly that
\begin{equation}
\phi'(t)=ap t^{p-1}+o(t^{p-1}),\qquad a>0, p\ge2, t\downarrow0.
\end{equation}
For sufficiently small $t$, $\phi'(t)$ is bounded above and below by
positive constants times $t^{p-1}$.  Separating variables in
$\dot t=-\phi'(t)$ gives
\begin{equation}
T(\varepsilon)\asymp\int_\varepsilon^{t_0}t^{-(p-1)}\,dt,
\end{equation}
and therefore
\begin{equation}
T(\varepsilon)=
\begin{cases}
\Theta(\log(t_0/\varepsilon)), & p=2,\\
\Theta(\varepsilon^{-(p-2)}), & p>2.
\end{cases}
\label{eq:gradient-flow}
\end{equation}
Thus the sharp vector least-squares and scalar zero-multiplier ALM paths have
one-dimensional correction exponents $2q-2$ and $4q-2$, respectively.

\begin{corollary}[Finite smooth penalties are not locally exact]
\label{cor:finite-penalty}
Let $F$ be a minimal cycle completion at a DAG $W$, let $U$ activate every
edge in $F$, and let $\mathcal L$ be $C^1$ with
$D\mathcal L(W)[U]<0$.  For any finite $\gamma\ge0$, $W$ is not a local
minimizer of
\begin{equation}
\mathcal L(\widetilde W)+\gamma P(\widetilde W)
\end{equation}
when $P=h$ is a smooth nonnegative exact scalar or
$P=\frac12\|H\|_2^2$ is a squared smooth exact vector residual.
\end{corollary}

\begin{proof}
Along $W+tU$ with $t\downarrow0$,
\begin{equation}
\mathcal L(W+tU)=\mathcal L(W)+tD\mathcal L(W)[U]+o(t).
\end{equation}
The scalar and squared-vector barriers give $P(W+tU)=O(t^{2q})=o(t)$.
The negative linear score term therefore dominates every finite multiple of
the penalty for all sufficiently small $t>0$.
\end{proof}

\subsection{Universal first- and second-order consequences}

\begin{proposition}[Gradient and Hessian degeneracy]
\label{prop:universal-degeneracy}
Let $h$ be differentiable, locally nonnegative, and exact.  At every DAG
$W$, $\nabla h(W)=0$.  If $h$ is twice differentiable, define
\begin{equation}
\mathcal S_W=\operatorname{span}\{E_{ij}:i\ne j,\ j\not\rightsquigarrow i
\text{ in }\operatorname{supp}(W)\}.
\end{equation}
Then $\mathcal S_W\subseteq\ker\nabla^2h(W)$ and
\begin{equation}
\dim\ker\nabla^2h(W)\ge d(d-1)-r(W)\ge\frac{d(d-1)}2,
\end{equation}
where $r(W)$ is the number of reachable ordered pairs.  At $W=0$ the
Hessian is zero.
\end{proposition}

\begin{proof}
A feasible DAG is a local minimizer of $h$, so Fermat's condition gives the
zero gradient.  If $j\not\rightsquigarrow i$, adding only edge $i\to j$
cannot create a cycle.  Hence
$t\mapsto h(W+tE_{ij})$ is identically zero and
$\langle E_{ij},\nabla^2h(W)E_{ij}\rangle=0$.  The Hessian at a local
minimum is positive semidefinite; for such an operator, zero quadratic form
implies membership in its kernel.  There are $d(d-1)-r(W)$ safe
coordinates, and a DAG has at most one reachable orientation for each
unordered pair, so $r(W)\le d(d-1)/2$.
\end{proof}

The proposition explains why LICQ and MFCQ fail for a scalar exact equality
at every feasible DAG.  It is weaker than the completion theorem when
$q>1$, but makes the topology dependence of the Hessian explicit.

\begin{proof}[Proof that tolerance is not a support certificate]
Let a DAG $W_0$ contain a directed path $j\rightsquigarrow i$ and add the
back edge $i\to j$ with weight $t$.  Every $t\ne0$ gives cyclic support,
whereas $W_t\to W_0$.  Exactness gives $h(W_t)>0$ and continuity gives
$h(W_t)\to h(W_0)=0$.  Therefore every positive tolerance accepts some
cyclic support.  Starting from an empty graph and shrinking all edges of a
cycle gives the same conclusion without requiring a nonempty base DAG.
\end{proof}

\section{Sharpness and Exact Constructions}
\label{app:bld}

This construction separates overflow control from low-order sensitivity.
Let $s(x)=x^2/(1+x^2)$, define $S(W)_{ij}=s(W_{ij})$ for $i\ne j$ and
$S(W)_{ii}=0$, and set
\begin{equation}
C(W)=\frac\kappa d S(W),\qquad
h_{\rm BLD}(W)=-\left(\frac d\kappa\right)^2
\log\det(I-C(W)),\quad 0<\kappa<1.
\end{equation}

\begin{proposition}[Bounded log-determinant representation]
For fixed $d$ and $0<\kappa<1$, $h_{\rm BLD}$ is $C^\infty$, nonnegative,
exact, and globally bounded.  Its gradient and Hessian are globally bounded,
so its gradient is globally Lipschitz.  Nevertheless, it obeys the same
$2q$ and $2\tau$ barriers as every nonnegative scalar representation.
\end{proposition}

\begin{proof}
Every entry of $S$ lies in $[0,1)$, so the maximum row and column sums of
the nonnegative matrix $C$ are below $\kappa$.  Thus $\rho(C)<\kappa<1$ and
\begin{equation}
-\log\det(I-C)=\sum_{r\ge1}\frac{\tr(C^r)}r.
\label{eq:bld-series}
\end{equation}
Every term is nonnegative.  A DAG makes $C$ nilpotent and every trace zero;
a directed cycle contributes a positive product to an appropriate trace.
This proves exactness.  Moreover
$\tr(C^r)\le d\kappa^r$, so the unscaled series is at most
$-d\log(1-\kappa)$.

Let $R=(I-C)^{-1}$.  The Neumann series gives
$\|R\|_1,\|R\|_\infty,\|R\|_2\le(1-\kappa)^{-1}$.  The first two
derivatives of $s$ are globally bounded.  For perturbations $U,V$,
\begin{align}
Dh[U]&=p\tr(R\,DC[U]),\\
D^2h[U,V]&=p\tr(R\,DC[V]R\,DC[U])
          +p\tr(R\,D^2C[U,V]),
\end{align}
where $p=(d/\kappa)^2$.  Standard Frobenius inequalities therefore give
global bounds for the gradient and Hessian.  The final claim follows from
Theorem~\ref{thm:scalar} and Eq.~\eqref{eq:weighted-scalar-order}; bounded
derivatives change coefficients but not the missing Taylor orders.
\end{proof}

\section{Numerical Verification of Geometry and Dynamics}
\label{app:diagnostics}

\subsection{Equal-scale hierarchy}

\begin{table}[t]
\centering
\small
\caption{Topology-dependent order hierarchy. The last column is the maximum absolute log--log slope error over $q=1,\ldots,5$.}
\label{tab:order-hierarchy}
\begin{tabular}{lccc}
\toprule
Formulation & Value order & Derivative order & Max. error \\
\midrule
Signed/vector residual $H$ & $q$ & $q-1$ & 0.000 \\
Vector least squares $\|H\|_2^2/2$ & $2q$ & $2q-1$ & 0.046 \\
Nonnegative scalar $h$ & $2q$ & $2q-1$ & 0.046 \\
Scalar ALM at $\mu=0$ & $4q$ & $4q-1$ & 0.046 \\
\bottomrule
\end{tabular}
\end{table}

For each $q\in\{1,\ldots,5\}$, the diagnostic uses
$d=\max(3,q+1)$ nodes.  It fixes the first $d-q$ edges of the path
$0\to1\to\cdots\to d-1$ at weight $0.7$.  The remaining path edges and
the closing edge $d-1\to0$ are the $q$ completion coordinates; the closing
edge has negative sign to ensure that signed weights are exercised.  Every
proper subset is acyclic.

The scale grid is $\operatorname{geomspace}(0.08,0.30,8)$.  We fit the
least-squares slope of $\log R(t)$ against $\log t$ on the four smallest
positive finite values.  The exact simple-cycle vector is enumerated for
$d\le6$.  Its Jacobian is computed by forward automatic differentiation
and then restricted to the $q$ completion coordinates before taking the
Frobenius norm.  EXP and all gradients use JAX float64 automatic
differentiation.  The 320 hierarchy rows contain eight scales for every
$(q,\text{method})$ pair.  A separate 160-row check applies the same base
DAGs to EXP, the trace polynomial, DAGMA log determinant, and BLD on
$\operatorname{geomspace}(0.10,0.30,8)$.

\begin{table}[h]
\centering
\small
\caption{Observed scalar orders for $q=1,\ldots,5$. Every family is predicted to have order $2q$.}
\label{tab:scalar-family}
\begin{tabular}{lcc}
\toprule
Constraint & Observed orders ($q=1,\ldots,5$) & Max. error \\
\midrule
EXP & 2.00, 4.00, 6.00, 8.00, 10.01 & 0.007 \\
Trace polynomial & 2.00, 4.00, 6.00, 8.00, 10.00 & 0.005 \\
DAGMA log-det & 2.00, 4.00, 6.00, 8.00, 10.00 & 0.004 \\
BLD (bounded) & 1.97, 3.93, 5.90, 7.87, 9.98 & 0.130 \\
\bottomrule
\end{tabular}
\end{table}

\begin{figure}[t]
    \centering
    \includegraphics[width=0.82\textwidth]{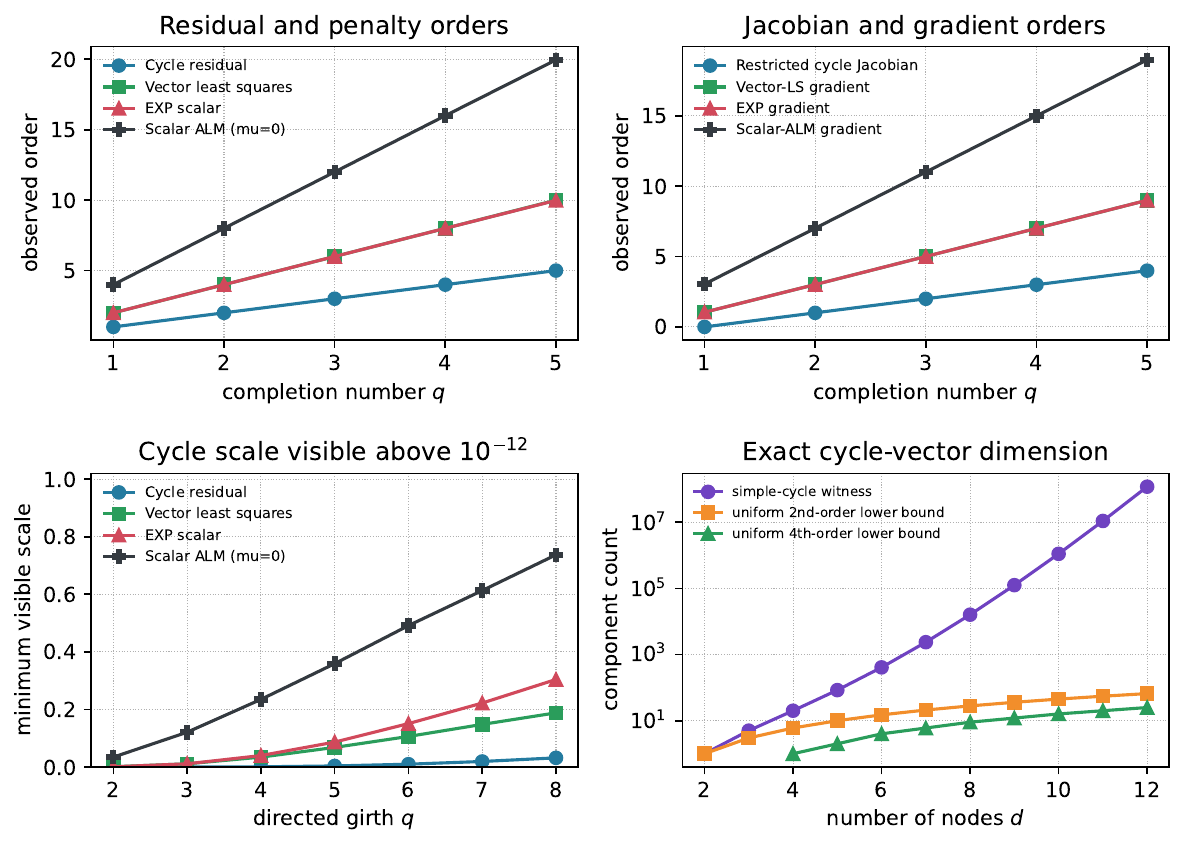}
    \caption{Empirical order hierarchy and its consequences.  Top:
    residual/penalty and restricted Jacobian/gradient orders.  Bottom left:
    smallest isolated-cycle scale visible above $10^{-12}$ in float64.
    Bottom right: exact cycle-vector size and uniform-sharpness lower bounds.}
    \label{fig:hierarchy}
\end{figure}

\subsection{Conditioned random-subspace census}
\label{app:random-census}

We test whether the computed completion number predicts local order away from
the controlled families.  For each $q\in\{1,2,3,4\}$ and 20 seeds, the
generator chooses a random directed cycle on $d=8$ nodes, assigns exactly $q$
of its edges to the candidate set, and places the remainder in the base DAG.
It then proposes additional random base edges, retaining one only if the base
remains acyclic and $q_\W(F)$ is unchanged.  It similarly adds two distractor
candidate edges only when $q_\W(F)$ remains unchanged.  Base weights are drawn
uniformly in magnitude from $[0.45,0.90]$, candidate coefficients from
$[0.70,1.30]$, and signs are independent.

As an independent check, the implementation exhaustively enumerates all
candidate subsets and recomputes $q_\W(F)$ before evaluation.  The resulting
instances contain 5--10 base edges, 3--6 candidate edges, and 1--5 active
simple cycles.  For eight scales in
$\operatorname{geomspace}(0.006,0.05,8)$, we evaluate the norm of all active
simple-cycle products and the positive trace series for EXP.  Slopes use the
four smallest scales, exactly as in the controlled diagnostics.  The 1,280 raw
rows yield 160 instance/method fits.  The predicted orders $q$ and $2q$ hold
with maximum absolute error $0.0033$ and median error $4.9\times10^{-15}$.

\begin{figure}[H]
    \centering
    \includegraphics[width=0.82\textwidth]{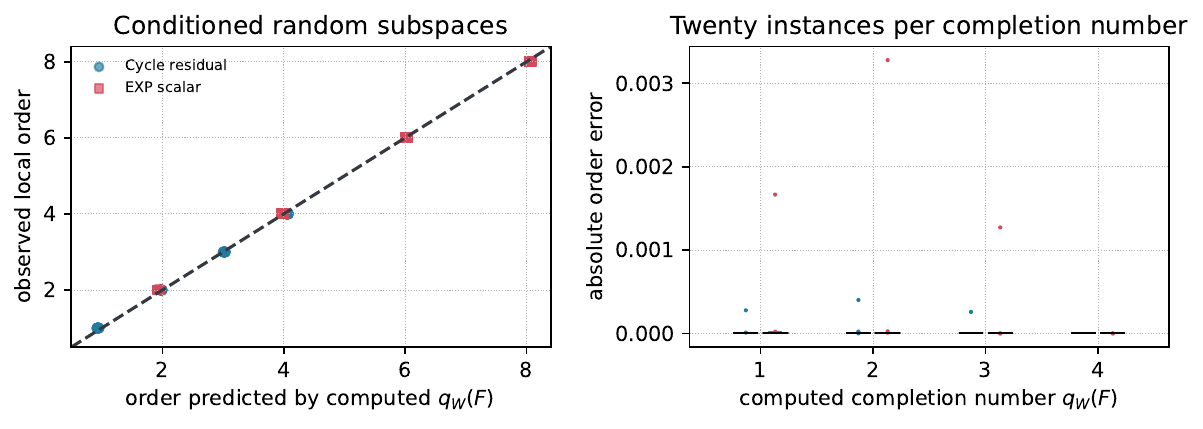}
    \caption{Conditioned random-subspace census.  Left: every observed order
    lies on the prediction from the independently recomputed $q_\W(F)$.  Right:
    absolute errors over 20 instances per completion number.}
    \label{fig:random-census}
\end{figure}

\subsection{Completion-volume audit}
\label{app:completion-volume-audit}

The script \texttt{paper\_iclr/explore\_completion\_invariants.py} performs a
separate finite falsification of Theorem~\ref{thm:volume-geometry}.  For every
random candidate complex it enumerates all supports, extracts the minimal
completions, solves both the integral and fractional cover problems, and
evaluates Eq.~\eqref{eq:exact-tube-law} on 30 values of $\varepsilon$.  It
also normalizes an optimal fractional-cover vector $z^\star$ so that
$\min_C\sum_{s\in C}z^\star_s=1$ and evaluates the logarithmic path
$|x_s|=t^{z^\star_s}$.  Along this path, the volume-density scale is
$t^{\sum_sz^\star_s}=t^\lambda$, while the generator and squared-generator
scales are $t$ and $t^2$.

\begin{table}[H]
\centering
\caption{Completion-volume falsification.  ``Ray'' lists the fitted vector
and scalar exponents.  The tournaments are genuine fractional
feedback-cover gaps, not arbitrary hypergraphs.}
\label{tab:completion-volume-audit}
\begin{tabular}{lrrrrrr}
\toprule
Suite & Cases & $q<r$ & $q$ & $r$ & $(b,\lambda)$ & Ray \\
\midrule
Random DAG subspaces & 597 & 440 & -- & -- & 23 distinct & exact \\
$d=7$ tournament & 1 & 1 & 3 & 7 & $(5,4.5)$ & $(4.5,2.25)$ \\
$d=8$ tournament & 1 & 1 & 3 & 8 & $(7,6.667)$ & $(6.667,3.333)$ \\
\bottomrule
\end{tabular}
\end{table}

For the random suite, the maximum error between the fitted tube slope and
$b$ is $1.7\times10^{-4}$; the valuation-ray errors for $\lambda$ and
$\lambda/2$ are below $10^{-8}$.  The $d=7$ tournament has 21 arcs and 64
simple cycles.  Its minimum feedback-arc cover has size five, whereas the LP
relaxation is $4.5$.  The $d=8$ example has 28 arcs and 288 cycles.
On the same rays, direct positive-series evaluations give orders
$(1.000000,1.000001,2.000001,2.000003)$ for the generator, cycle vector,
EXP, and DAGMA in the $d=7$ case; the maximum error is $1.5\times10^{-4}$
across both tournaments.

We additionally draw scrambled Sobol points in $[0,1]^m$ and directly count
generator sublevel events.  At practical thresholds this gives slopes
$3.219$ and $4.594$ for the two tournaments, below the asymptotic values
$4.5$ and $6.667$.  This is expected rather than a contradictory estimate:
Eq.~\eqref{eq:volume-sandwich} permits polynomial factors in
$\log(1/\eta)$, and the high-dimensional events become too rare before those
factors are negligible.  We report this failed finite-scale extrapolation to
distinguish the theorem's logarithmic limit from a claim about moderate
tolerances.  Neither Sobol result is used to select a parameter or establish
the theorem.

\subsection{Weighted overlapping cycles}

The weighted graph has candidate cycles
\begin{align}
C_A&:0\to1\to2\to3\to0,\\
C_B&:0\to1\to4\to3\to0.
\end{align}
They share edges $0\to1$ and $3\to0$.  The four exponent patterns give
cycle costs A/B of $4/4$, $5/8$, $6/8$, and $8/5$.  The final pattern
switches the cheapest cycle.  Eight scales follow
$\operatorname{geomspace}(0.12,0.30,8)$.  A generic weighted
shortest-directed-cycle routine computes $\tau$ from the graph rather than
reading either displayed cycle cost.  There are 128 raw weighted rows.

\begin{table}[t]
\centering
\small
\caption{Weighted overlapping cycles. Costs A/B are the two cycle costs and $\tau$ is their minimum. The final row switches the cheapest cycle.}
\label{tab:weighted-hierarchy}
\begin{tabular}{rrrrrr}
\toprule
A/B & $\tau$ & $\|H\|$ & $\|H\|^2/2$ & EXP & EXP$^2/2$ \\
\midrule
4/4 & 4 & 4.00 & 8.00 & 8.00 & 16.00 \\
5/8 & 5 & 5.00 & 10.00 & 10.00 & 20.00 \\
6/8 & 6 & 6.00 & 12.00 & 12.00 & 24.00 \\
8/5 & 5 & 5.00 & 10.00 & 10.00 & 20.00 \\
\bottomrule
\end{tabular}
\end{table}

\subsection{Subcritical compression diagnostic}
\label{app:compression}

For $d=10$ and $q=2,\ldots,6$, let $C_q(W)\in\R^{M_q(10)}$ collect the
selected cycle products from Theorem~\ref{thm:dimension}.  For a matrix
$A\in\R^{r\times M_q(10)}$ with orthonormal rows, we evaluate the globally
exact residual
\begin{equation}
H_A(W)=\left(A C_q(W),\ h_{\exp}(W)\right).
\label{eq:compressed-residual}
\end{equation}
Because the EXP component is globally exact, $H_A(W)=0$ if and only if $W$
is a DAG, regardless of the projection.  Along the channel family,
$C_q(tU(z))=t^qz$ and $h_{\exp}(tU(z))=\Theta(t^{2q})$.  A generic unit $z$
therefore gives $\norm{H_A(tU(z))}_2=\Theta(t^q)$, whereas any nonzero
$z\in\ker A$ gives $\Theta(t^{2q})$.  Rank--nullity supplies such a direction
whenever $r<M_q(10)$.

We draw 20 random orthogonal maps at each of ranks
$\lfloor M_q/4\rfloor$, $\lfloor M_q/2\rfloor$, and $M_q-2$.  After adding
the fallback, every total residual dimension remains strictly below $M_q$.
An SVD gives the null direction and an independent normalized Gaussian vector
gives the generic direction.  Eight local scales follow
$\operatorname{geomspace}(0.02,0.12,8)$, producing 4,800 raw rows.  EXP is
evaluated by its positive trace series rather than subtracting $d$, and the
validator checks that it is positive on every cyclic mixture.  The diagnostic
is a global-exactness stress test, not a proposal to scale the dense EXP
fallback to large graphs.  The largest absolute fitted-order error is $0.0031$.

\begin{table}[ht]
\centering
\small
\caption{Subcritical random compression on the $d=10$ channel families of Theorem~\ref{thm:dimension}. Observed orders are medians over three ranks and 20 seeds.}
\label{tab:compression}
\begin{tabular}{rrrr}
\toprule
$q$ & Channels $M_q(10)$ & Generic & Null direction \\
\midrule
2 & 45 & 2.000 & 4.002 \\
3 & 20 & 3.000 & 6.000 \\
4 & 16 & 4.000 & 8.000 \\
5 & 12 & 5.000 & 10.000 \\
6 & 9 & 6.000 & 12.000 \\
\bottomrule
\end{tabular}
\end{table}

\begin{figure}[H]
    \centering
    \includegraphics[width=0.82\textwidth]{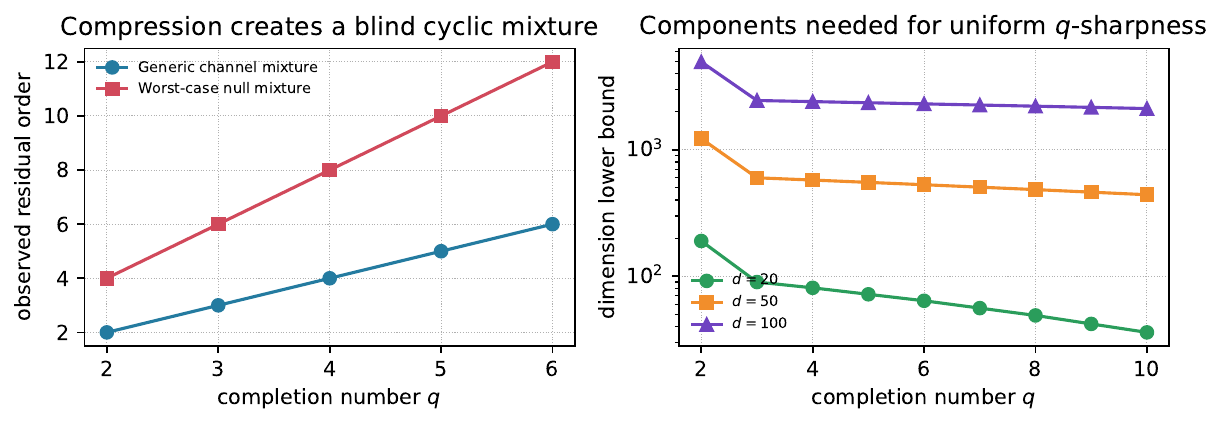}
    \caption{Subcritical compression creates a worst-case cyclic mixture with
    doubled residual order (left). The funnel lower bound is quadratic in $d$
    for each fixed $q$ (right).}
    \label{fig:compression}
\end{figure}

\subsection{Visibility and cycle-vector size}

For each isolated cycle length $q=2,\ldots,8$, the diagnostic evaluates 500
scales from $10^{-4}$ to $0.99$ and records the first residual at least
$10^{-8}$, $10^{-12}$, and $10^{-16}$.  We report the grid-based value so that
EXP cancellation and float64 evaluation are included.  Cycle-vector dimensions use the
closed count in Eq.~\eqref{eq:cycle-count} for $d=2,\ldots,12$.

\begin{table}[ht]
\centering
\small
\caption{Smallest isolated-cycle scale whose residual exceeds $10^{-12}$ in float64. Lower is better.}
\label{tab:visibility}
\begin{tabular}{rcccc}
\toprule
$q$ & $\|H\|$ & $\|H\|^2/2$ & EXP & EXP$^2/2$ \\
\midrule
2 & 0.0001 & 0.0012 & 0.0010 & 0.0345 \\
4 & 0.0010 & 0.0345 & 0.0400 & 0.2350 \\
6 & 0.0100 & 0.1064 & 0.1510 & 0.4913 \\
8 & 0.0321 & 0.1884 & 0.3042 & 0.7371 \\
\bottomrule
\end{tabular}
\end{table}

\subsection{Restricted optimization dynamics}
\label{app:optimizer-dynamics}

The family has one fixed edge
$c\sim\operatorname{Uniform}[0.55,0.90]$ and $q$ completion edges of scale
$t$, giving $R(t)=ct^q$.  With $L=q+1$ and $P=c^2t^{2q}$, NOTEARS equals
\begin{equation}
h_{\exp}(t)=\sum_{m=1}^{\infty}\frac{L P^m}{(mL)!},
\qquad
\tr(A^k)=0\ \text{unless }L\mid k.
\label{eq:isolated-exp-series}
\end{equation}
Twelve terms match the JAX matrix exponential through $q=6$ within
$2\times10^{-14}$.  Every method uses 1,000 projected-gradient steps with
Armijo parameter $10^{-4}$ and maximum step $0.25$.  ALM uses 40 updates,
25 primal steps per update, $\rho_0=1$, and growth $1.2$ capped at $10^6$;
only the seeded variant sets $\mu_0=1$.  Every trajectory starts at $t=0.8$.
The experiment therefore follows the full restricted correction path into the
local regime, but it is not evidence about global parameter learning.

The coefficient-controlled protocol rescales each objective to unit initial
gradient.  At $q=6$, its medians are $0.134/0.266/0.357/0.357/0.376/0.357$
for $p=1/p=1.5/p=2$/EXP/cold/seeded; full records are in the released CSV.

\begin{table}[ht]
\centering
\small
\caption{Median completion-edge scale after 1,000 raw-objective projected-gradient steps over 20 seeds. Lower is better; all runs start at $0.8$.}
\label{tab:optimizer-dynamics}
\begin{tabular}{rcccccc}
\toprule
$q$ & $p=1$ & $p=1.5$ & $p=2$ & EXP & cold ALM & seeded ALM \\
\midrule
2 & 0.000 & 0.002 & 0.044 & 0.044 & 0.094 & 0.044 \\
4 & 0.025 & 0.125 & 0.256 & 0.387 & 0.469 & 0.386 \\
6 & 0.122 & 0.268 & 0.406 & 0.707 & 0.790 & 0.706 \\
8 & 0.219 & 0.372 & 0.501 & 0.798 & 0.800 & 0.798 \\
\bottomrule
\end{tabular}
\end{table}

\begin{figure}[H]
    \centering
    \includegraphics[width=0.82\textwidth]{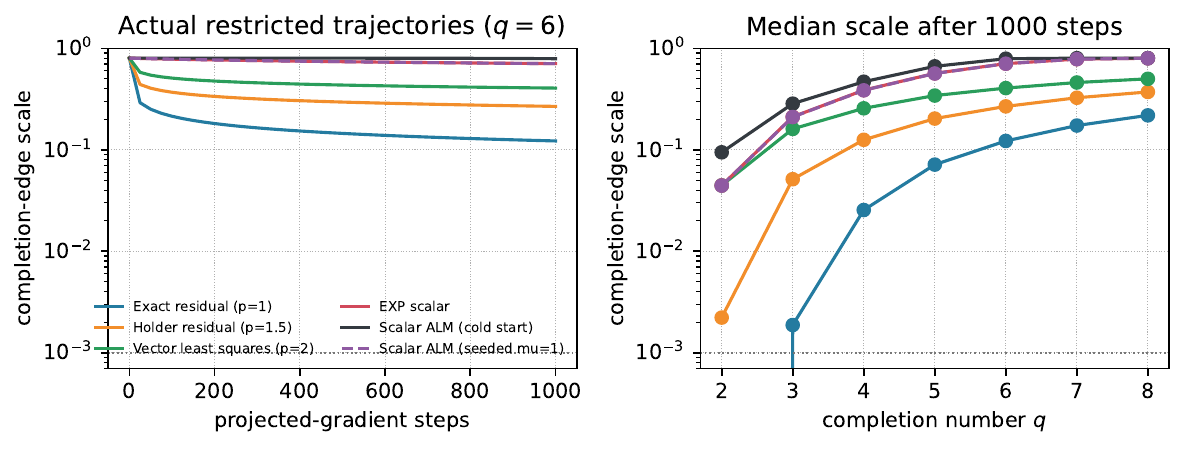}
    \caption{Restricted feasibility-correction trajectories under the fixed
    protocol. Lower is better. The exact $p=1$ residual retains the strongest
    signal, while smooth scalarization and cold-start ALM become stiffer as the
    completion number grows.}
    \label{fig:dynamics}
\end{figure}

\subsection{Full-matrix selection and SEM-score protocol}
\label{app:full-matrix-selection}

The experiment supporting Proposition~\ref{prop:symmetry-time} uses directed
cycles of lengths \(L\in\{4,6,8\}\).  Every cycle edge is optimized.  The two
tracked edges start at \(a\pm\varepsilon\), the remaining \(L-2\) edges start
at \(0.8\), and we vary \(a\in\{0.3,0.5\}\),
\(r\in\{0.1,0.25\}\), and
\(\varepsilon\in[10^{-5},10^{-2}]\).  We evaluate the exact full-matrix
NOTEARS and DAGMA constraints under
\[
\Psi(h)=h,\qquad
\Psi(h)=\frac12h^2,\qquad
\Psi(h)=h+\frac12h^2.
\]
These are the linear feasibility force, a cold quadratic penalty, and an ALM
inner objective with unit multiplier and penalty.  Selection times are
computed by quadrature of the exact matrix-flow reduction, not by fitting a
surrogate.  We additionally compare the closed values and every cycle-edge
gradient with direct matrix exponential and inverse evaluations at 120 random
cycle matrices.

For the score experiment, the population data distribution is the linear
Gaussian SEM whose true graph is the path
\[
0\longrightarrow1\longrightarrow\cdots\longrightarrow L-1
\]
with every coefficient equal to \(0.5\) and unit noise variance.  If
\(W_\star\) is this path matrix, its covariance under the row-vector
convention is
\[
\Sigma=(I-W_\star)^{-T}(I-W_\star)^{-1}.
\]
Initialization adds the false back edge \((L-1)\to0\), closing the cycle.
The optimized smooth objective is
\begin{equation}
F_\Sigma(W)=
\frac{0.2}{2}\operatorname{tr}
\!\left((I-W)^T\Sigma(I-W)\right)
+h(W)+\frac12h(W)^2.
\label{eq:population-score-objective}
\end{equation}
We run both a cycle-support flow and a dense flow in which every off-diagonal
entry may change.  ``Aligned'' initialization makes the false edge smaller by
\(2\varepsilon\); the adversarial ``reversed'' initialization makes it larger.
Correct selection is declared when the false-to-true tracked-edge ratio first
reaches \(1/4\).  The opposite \(1/4\) event is recorded as an incorrect
selection.  No ground-truth quantity enters the gradient or stopping rule
other than this post-run correctness label.

All 216 population flows reach the false-edge ratio event before its opposite,
as do explicit gradient-descent runs with step \(0.01\) on the predeclared
subset of \(\varepsilon\) values.  The fitted selection-time slopes range from
\(-0.0041\) to \(0.0039\), and the minimum total deletion margin in every
population trajectory is positive.  Dense trajectories are genuinely
different from the cycle restriction: their off-cycle Frobenius norm reaches
\(0.384\).  We enforce the DAGMA domain \(\rho(W\circ W)<1\) at every
gradient evaluation; the largest value over all population and finite-sample
DAGMA trajectories is \(0.453\).

For the finite-sample check, we replace \(\Sigma\) by the empirical
covariance from \(n\in\{128,512,2048\}\) observations and use 20 fixed seeds.
All 720 runs use reversed initialization.  Overall correctness is \(91.25\%\);
the per-setting range is \(65\%\) to \(100\%\), increasing toward the
population result with sample size.  These failures are retained because
Proposition~\ref{prop:symmetry-time} is conditional on a deterministic
positive margin.  It does not claim that observational samples always supply
that margin.

\subsection{Official optimizer trajectories}
\label{app:predictive-bridge}

The isolated-cycle experiments verify the asymptotic calculation in
Proposition~\ref{prop:symmetry-time}.  They do not show, by themselves, that
the same geometry is visible along an optimizer trajectory on an ordinary
random graph.  We therefore record the unthresholded iterates of the official
linear NOTEARS and DAGMA updates.  The update equations, continuation
schedules, stopping rules, and regularization parameters are unchanged.
Read-only callbacks store every fifth L-BFGS-B iterate for NOTEARS and every
250th Adam iterate for DAGMA.  On small deterministic tests, the instrumented
and uninstrumented implementations agree to relative tolerances
$10^{-9}$ and $10^{-12}$, respectively.

\subsubsection{Truth-free early diagnostics}

For an unthresholded matrix \(W\), define its critical DAG threshold
\begin{equation}
\tau(W)=\min\{t\ge0:\supp(\lvert W\rvert>t)\text{ is a DAG}\}.
\label{eq:critical-dag-threshold}
\end{equation}
This quantity is used as a diagnostic, not as the reported graph.  At the
first completed continuation stage whose nonzero support is cyclic, let
\(B=\supp(\lvert W\rvert>2\tau)\), and treat entries with magnitude in
\([\tau,2\tau]\) as completion coordinates.  The base \(B\) is acyclic.  We
compute the smallest number \(q_2\) of band coordinates needed to complete a
cycle and choose a deterministic shortest witnessing cycle \(C\).  Thus
\(q_2\) is a two-scale operational version of the completion statistic, not a
claim that the random trajectory lies in the exact local model of
Section~\ref{sec:completion-geometry}.

Three dimensionless quantities describe the early ranking on \(C\).  Write
\[
a_e=|W_e|,\qquad
b_e=|\partial_eh(W)|,\qquad
c_e=\mathcal L(W\text{ with }W_e=0)-\mathcal L(W).
\]
For the weights and deletion costs, the separation is the relative gap
between the smallest and second-smallest values.  For the constraint force it
is the relative gap between the largest and second-largest values.  A zero
gap denotes a first-order tie.  The outcome is the first later snapshot at
which the smallest cycle weight is at most one tenth of the second-smallest.
Time is divided by the remaining optimizer steps, separately for each method;
an unresolved trajectory is right-censored at one.  No generating edge,
coefficient, or SHD enters the anchor, cycle, predictor, or outcome.

Development seeds 100--104 cover ER2/ER4/SF2/SF4 graphs under standard
Gaussian and 40\%-weak-edge Gaussian SEMs.  They fix one ridge model per
optimizer with penalty one.  The baseline uses
\(\log h(W)\), \(\log(\tau(W)/\|W\|_{\max})\), and
\(\log\|W\|_{\max}\).  The geometry model adds \(q_2\) and the three
separation statistics.  Formal seeds 0--9 use the same two regimes plus
exponential and heteroscedastic noise.  There are 160 formal trajectories per
optimizer.  Spearman correlations, bootstrap intervals, and permutation
tests are computed only after this freeze.

\subsubsection{Formal results}

\begin{table}[H]
\centering
\small
\caption{Predictive bridge on 320 formal trajectories.  The correlation is
between early constraint separation and normalized selection time.  Models
are fitted only on disjoint development seeds.}
\label{tab:predictive-bridge}
\begin{tabular}{llcccc}
\toprule
Method & Early predictor & $\rho_s$ [95\% CI] & $p_{\rm perm}$ & Base MAE & Geom. MAE \\
\midrule
NOTEARS & constraint separation & -0.52 [-0.62,-0.40] & 0.0001 & 0.041 & 0.041 \\
DAGMA & constraint separation & -0.66 [-0.75,-0.56] & 0.0001 & 0.108 & 0.088 \\
\bottomrule
\end{tabular}
\end{table}

Constraint separation predicts faster selection for both optimizers:
\(\rho_s=-0.521\) for NOTEARS and \(-0.664\) for DAGMA.  The 95\% bootstrap
intervals exclude zero, and both permutation tests give \(p<10^{-4}\).
Ranking observations within each graph--noise stratum gives
\(-0.564\) and \(-0.649\), respectively, so the result is not explained by
pooling sparse and dense regimes.  Completion depth has the opposite sign,
with correlations \(0.435\) for NOTEARS and \(0.386\) for DAGMA.  The
development-frozen geometry model lowers DAGMA's formal MAE by \(18.5\%\),
from \(0.1078\) to \(0.0879\), and raises its prediction rank correlation from
\(0.316\) to \(0.648\).  For NOTEARS the MAE change is negligible
(\(0.04112\) to \(0.04100\)); a local association need not make a richer
cross-regime predictor useful.

\begin{figure}[H]
    \centering
    \includegraphics[width=0.94\textwidth]{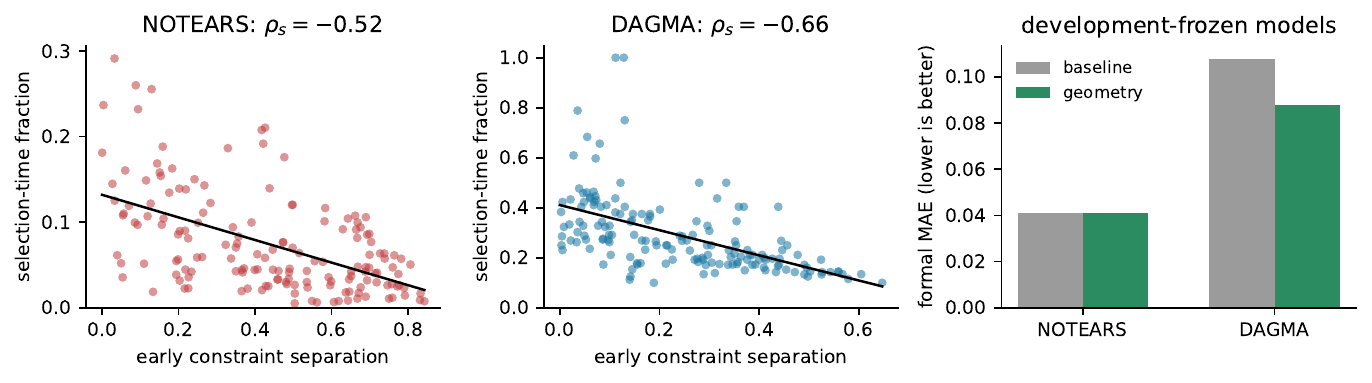}
    \caption{A truth-free early diagnostic predicts a later support event.
    Left and center: formal trajectories and pooled rank correlations.
    Right: MAE of development-frozen baseline and geometry models.  Prediction
    improves materially for DAGMA, but not for NOTEARS.}
    \label{fig:predictive-bridge}
\end{figure}

The experiment also locates a limit of the random-graph diagnostic.  Of the
320 critical cycles, 319 are directed two-cycles and one is a three-cycle;
\(q_2=1\) in 78 runs and \(q_2=2\) in 242.  Dense signed frontends therefore
expose mostly pairwise orientation conflicts at their first continuation
stage.  The higher completion orders in the controlled experiments should not
be described as typical of these trajectories.  Score separation is
method-dependent as well: its formal correlation is \(-0.304\) for NOTEARS
but \(0.118\) for DAGMA.  What survives the move from the controlled model is
the narrower statement that an early feasibility tie and a larger
two-scale completion depth forecast a slower support decision.  This is
evidence that the representation geometry is operationally visible, not a
global iteration lower bound.

\section{Modular Certificate: Proof and Scaling Audit}
\label{app:completion-margin}

\subsection{The algebraic bridge}

\begin{lemma}[Completion-repair duality]
\label{lem:completion-repair-duality}
Let
\[
I_E=\langle x^C:C\in\mathcal C_{\min}\rangle
\]
be the squarefree cycle-completion ideal on an active directed support \(E\).
If \(\mathcal T_{\min}\) is the family of inclusion-minimal deletion sets that
leave a DAG, then
\begin{equation}
I_E^\vee
=\bigcap_{C\in\mathcal C_{\min}}\langle x_e:e\in C\rangle
=\langle x^D:D\in\mathcal T_{\min}\rangle .
\label{eq:completion-repair-duality}
\end{equation}
\end{lemma}

\begin{proof}
The Alexander dual of a squarefree monomial ideal is the intersection of the
coordinate primes indexed by its minimal generators
\citep{miller2005combinatorial}.  Hence
\[
I_E^\vee
=\bigcap_{C\in\mathcal C_{\min}}\langle x_e:e\in C\rangle.
\]
A squarefree monomial \(x^D\) belongs to this intersection exactly when
\(D\cap C\ne\varnothing\) for every \(C\in\mathcal C_{\min}\).  This is the
transversal condition, and divisibility-minimal monomials correspond exactly
to inclusion-minimal transversals.  These transversals are the minimal edge
deletions that make \(E\) acyclic.
\end{proof}

We record the edge-modular specialization used below.  After a frozen training
stage, let each active edge have bounded holdout deletion loss
\(Z_e\in[-B,B]\) and population mean \(\mu_e\).  Write
\(\mu(D)=\sum_{e\in D}\mu_e\) for a feasible transversal
\(D\in\mathcal T\), and let \(D^*\) be a minimum-cost transversal.  When this
minimizer is unique, its normalized completion-exchange margin is
\begin{equation}
\kappa_{\mathcal C}(\mu)=
\min_{D\in\mathcal T,\,D\ne D^*}
\frac{\mu(D)-\mu(D^*)}{|D\mathbin{\triangle}D^*|}.
\label{eq:completion-exchange-margin}
\end{equation}
Given simultaneous intervals \([\ell_e,u_e]\) and an empirical minimizer
\(\widehat D\), define
\begin{equation}
\underline\Delta(D;\widehat D)=
\sum_{e\in D\setminus\widehat D}\ell_e
-\sum_{e\in\widehat D\setminus D}u_e,\qquad
\mathcal A_\alpha(\widehat D)=
\{D\in\mathcal T:\underline\Delta(D;\widehat D)\le0\}.
\label{eq:completion-robust-gap}
\end{equation}
Let \(E_{\rm del}=\cap_{D\in\mathcal A_\alpha}D\),
\(E_{\rm keep}=E\setminus\cup_{D\in\mathcal A_\alpha}D\),
\(L=\min_{D\in\mathcal T}\underline\Delta(D;\widehat D)\), and
\(L_-=\min_{D\ne\widehat D}\underline\Delta(D;\widehat D)\).

\begin{theorem}[Selective completion-repair confidence set]
\label{thm:completion-margin}
If all intervals cover simultaneously, every population-optimal repair belongs
to \(\mathcal A_\alpha(\widehat D)\).  Consequently, every population optimum
deletes \(E_{\rm del}\) and retains \(E_{\rm keep}\).  On the same event,
\begin{equation}
\mu(\widehat D)-\min_{D\in\mathcal T}\mu(D)\le[-L]_+.
\label{eq:completion-regret}
\end{equation}
Moreover, \(\mathcal A_\alpha=\{\widehat D\}\) exactly when \(L_->0\).  If
the intervals have common radius \(r\) and
\(\kappa_{\mathcal C}(\mu)>0\), then
\(r<\kappa_{\mathcal C}(\mu)\) implies \(\widehat D=D^*\), while
\(r<\kappa_{\mathcal C}(\mu)/2\) makes the confidence set a singleton.
\end{theorem}

\begin{proposition}[Conservative cycle-cover certificate]
\label{prop:relaxed-selective}
For any feasible reference \(D_0\), relax repair indicators to
\(P=\{z\in[0,1]^E:\sum_{e\in C}z_e\ge1\ \forall C\}\).  If the robust
forced-opposite LP for edge \(e\) has lower bound above the robust score of
\(D_0\), every population optimum has the \(D_0\) label at \(e\).  These
labels are a subset of those returned by exact forced-opposite optimization
for the same reference.  Minimum-weight cycle separation gives a
polynomial-time separation oracle for the LP.
\end{proposition}

\subsection{Proof of Theorem~\ref{thm:completion-margin}}

A deletion set leaves an acyclic support if and only if it intersects every
simple directed cycle.  The simple cycles are precisely the inclusion-minimal
nonfaces of the acyclic edge complex, so the feasible deletion sets are the
transversals of \(\mathcal C_{\min}\).

Let \(m=|E|\).  For the common-radius statement, Hoeffding's inequality and a
union bound give the event
\begin{equation}
\mathcal E_n=\left\{
\max_{e\in E}|\widehat\mu_e-\mu_e|\le r_n
\right\},
\qquad r_n=B\sqrt{2\log(2m/\alpha)/n},
\qquad \Pr(\mathcal E_n)\ge1-\alpha.
\label{eq:completion-uniform-event}
\end{equation}
The first part of the theorem only needs the more general simultaneous event
\(\mathcal E=\{\ell_e\le\mu_e\le u_e\ \forall e\}\).  Dependence among the
coordinates of one holdout vector is allowed.  On \(\mathcal E\), every
\(D\in\mathcal T\) satisfies
\begin{align}
\mu(D)-\mu(\widehat D)
&=\sum_{e\in D\setminus\widehat D}\mu_e
  -\sum_{e\in\widehat D\setminus D}\mu_e \\
&\ge \underline\Delta(D;\widehat D).
\label{eq:completion-relative-lower}
\end{align}
Let \(D^*\) be any population-optimal repair.  Since
\(\mu(D^*)-\mu(\widehat D)\le0\), Eq.~\eqref{eq:completion-relative-lower}
implies \(\underline\Delta(D^*;\widehat D)\le0\), and hence
\(D^*\in\mathcal A_\alpha(\widehat D)\).  This holds for every population
optimum.  An edge in the intersection of the confidence set is therefore
deleted by every optimum; an edge outside its union is retained by every
optimum.

Minimizing Eq.~\eqref{eq:completion-relative-lower} over \(D\) proves
Eq.~\eqref{eq:completion-regret}.  By definition,
\(\mathcal A_\alpha=\{\widehat D\}\) if and only if
\(\underline\Delta(D;\widehat D)>0\) for every \(D\ne\widehat D\), which is
equivalent to \(L_->0\).  Equation~\eqref{eq:completion-relative-lower} then
makes \(\widehat D\) the unique population optimum.

Suppose \(D^*\) is unique.  For every other transversal,
\begin{equation}
\widehat\mu(D)-\widehat\mu(D^*)
\ge \left(\kappa_{\mathcal C}(\mu)-r_n\right)
|D\mathbin{\triangle}D^*|.
\label{eq:completion-empirical-gap}
\end{equation}
Hence \(r_n<\kappa_{\mathcal C}(\mu)\) makes \(D^*\) the unique empirical
minimizer.  Applying the confidence subtraction once more gives
\begin{equation}
\underline\Delta(D;D^*)
\ge \left(\kappa_{\mathcal C}(\mu)-2r_n\right)
|D\mathbin{\triangle}D^*|,
\label{eq:completion-certificate-gap}
\end{equation}
which is positive for all competitors when
\(r_n<\kappa_{\mathcal C}(\mu)/2\).
This proves the theorem.

\subsection{Computing the selective labels}

The confidence set need not be enumerated.  Define
\begin{equation}
c_e=\begin{cases}
u_e,&e\in\widehat D,\\
\ell_e,&e\notin\widehat D.
\end{cases}
\qquad
\underline\Delta(D;\widehat D)
=\sum_{e\in D}c_e-\sum_{e\in\widehat D}u_e.
\label{eq:selective-relative-cost}
\end{equation}
Membership in \(\mathcal A_\alpha\) is therefore a threshold test on a
weighted feedback-edge objective.  For each edge, one solve with its label
forced opposite to \(\widehat D\) determines whether any repair in the set can
reverse that label.  The whole selective output takes one unconstrained and at
most \(|E|\) forced-opposite solves.  Weighted feedback-edge optimization is
NP-hard in the worst case, so this identity removes explicit set enumeration
but not combinatorial complexity.  The formal main protocol gives each forced
query one second; a timeout remains computationally unresolved rather than
being treated as evidence.

\subsection{Proof of Proposition~\ref{prop:relaxed-selective}}

The reference \(D_0\) need not minimize either the empirical or population
score.  Define
\[
\mathcal A_\alpha(D_0)=
\{D\in\mathcal T:\underline\Delta(D;D_0)\le0\}.
\]
Let
\begin{align*}
P&=\left\{z\in[0,1]^E:
\sum_{e\in C}z_e\ge1\text{ for every directed cycle }C\right\},\\
c_e&=\begin{cases}u_e,&e\in D_0,\\ \ell_e,&e\notin D_0,\end{cases}
\qquad
\tau_0=\sum_{e\in D_0}u_e.
\end{align*}
For an active edge \(e\), let \(P_e^{\rm opp}\) impose \(z_e=0\) when
\(e\in D_0\), and \(z_e=1\) otherwise.  The relaxation reports the \(D_0\)
label precisely when
\begin{equation}
\min_{z\in P_e^{\rm opp}}c^\top z>\tau_0.
\label{eq:relaxed-selective}
\end{equation}
On the simultaneous coverage event, Eq.~\eqref{eq:completion-relative-lower}
holds with \(D_0\) in place of \(\widehat D\).  If \(D^*\) is any population
optimum, then \(\mu(D^*)-\mu(D_0)\le0\), and hence
\(\underline\Delta(D^*;D_0)\le0\).  Thus every population optimum belongs to
\(\mathcal A_\alpha(D_0)\), even when \(D_0\) is obtained by a greedy repair.

For the costs in Proposition~\ref{prop:relaxed-selective}, direct expansion
gives
\begin{equation}
\underline\Delta(D;D_0)
=\sum_{e\in D}c_e-\tau_0.
\label{eq:relaxed-cost-identity}
\end{equation}
Every incidence vector \(\mathbf 1_D\) of a feasible repair intersects every
directed cycle, so \(\mathbf 1_D\in P\).  Suppose that \(D\in
\mathcal A_\alpha(D_0)\) gives edge \(e\) the label opposite to \(D_0\).
Then \(\mathbf 1_D\in P_e^{\rm opp}\), while
Eq.~\eqref{eq:relaxed-cost-identity} gives
\(c^\top\mathbf 1_D\le\tau_0\).  This contradicts
Eq.~\eqref{eq:relaxed-selective}.  The same argument covers an empty forced
face by assigning it optimum \(+\infty\).  Therefore every repair in the
confidence set, and in particular every population optimum, shares the
reported label.

The exact forced-opposite problem minimizes the same objective over the
integer points of \(P_e^{\rm opp}\).  Its LP optimum is no larger than its
integer optimum.  Any label certified by the relaxation is consequently also
certified by an exact forced-opposite query using the same reference \(D_0\);
this need not be the confidence set centered at a different empirical
minimizer.  The relaxation can lose labels only by abstaining.

It remains to justify the computational claim.  For a candidate point
\(z\in[0,1]^E\), a cycle inequality is violated exactly when the directed
graph contains a cycle of total \(z\)-weight below one.  Since the weights are
nonnegative, a minimum-weight directed cycle can be found by shortest paths,
which supplies a polynomial-time separation oracle.  The equivalence of
separation and optimization therefore gives a polynomial-time algorithm for
the rational LP.  Our implementation uses repeated shortest-cycle separation
with a standard LP solver.  This practical cutting-plane loop is not itself
claimed to have a strongly polynomial iteration bound.

The implementation does not use a primal feasible objective as a lower bound.
For each LP it reconstructs a weak-duality bound from the row multipliers and
box constraints; for each integer query it uses the MILP dual bound.  A
floating-point slack is subtracted before either quantity can certify a label.

\subsection{Two-point rate calibration}

The following calculation checks the scale of the sufficient condition.  It
is not a DAG-specific lower bound and does not constrain an optimizer with a
different output rule.

Take a directed two-cycle with edge set
\(E=\{e_1,e_2\}\).  Its minimum transversals are \(\{e_1\}\) and
\(\{e_2\}\).  Fix \(a>\kappa\) and consider independent Gaussian score
coordinates with mean vectors
\begin{equation}
\mu_+=(a-\kappa,a+\kappa),\qquad
\mu_-=(a+\kappa,a-\kappa)
\end{equation}
and covariance \(\sigma^2I_2\).  The optimal deletion is different under the
two models, and Eq.~\eqref{eq:completion-exchange-margin} equals \(\kappa\)
under both.  Their \(n\)-sample Kullback--Leibler divergence is
\begin{equation}
\operatorname{KL}(P_+^n,P_-^n)=\frac{4n\kappa^2}{\sigma^2}.
\end{equation}
For \(\delta\in(0,1/4)\), any selector that succeeds under both models induces
a test with both error probabilities at most \(\delta\).  The standard
two-point testing inequality
gives
\(2\delta\ge\tfrac12\exp\{-\operatorname{KL}(P_+^n,P_-^n)\}\), which gives
\begin{equation}
n\ge \frac{\sigma^2}{4\kappa^2}\log\frac1{4\delta}.
\label{eq:completion-sample-lower}
\end{equation}
General instance-dependent best-set bounds are substantially sharper
\citep{chen2017combinatorial}.  We use this familiar rate only to calibrate the
score margin.

\subsection{Controlled and random-graph audits}

The two-cycle calculation has an exact error curve.  Under \(P_+\), the
empirical score difference is Gaussian with mean \(2\kappa\) and variance
\(2\sigma^2/n\), so the wrong-deletion probability is
\begin{equation}
\Phi\!\left(-\sqrt{2n}\,\kappa/\sigma\right).
\end{equation}
We use four margins, six values of \(n\kappa^2/\sigma^2\), and 200,000
replications per cell.  The simulated curve differs from the formula by at
most \(0.0016\); the fitted log--log slope of the sample size required for
error \(0.05\) is \(-2.000\).

\begin{figure}[H]
    \centering
    \includegraphics[width=0.82\textwidth]{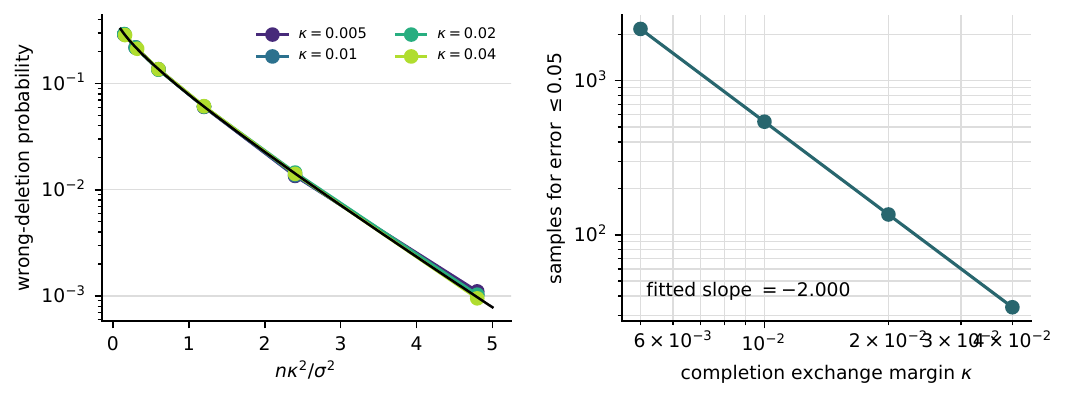}
    \caption{Statistical selection on the directed two-cycle.  Error curves
    collapse under the scaled budget \(n\kappa^2/\sigma^2\) (left), and the
    sample size for fixed error has slope \(-2\) in \(\kappa\) (right).}
    \label{fig:completion-margin-law}
\end{figure}

The random-graph audit uses the 160 frozen all-node Lasso screens from the
modular score experiment.  For each screen, an exact weighted feedback-edge
solver finds the best and second-best deletion transversals under an
independent 20,000-sample oracle proxy.  Their objective gap divided by their
symmetric-difference size estimates \(\kappa_{\mathcal C}\).  This proxy is
used only for the audit, never for selection or certification.

\begin{table}[H]
\centering
\small
\caption{Completion-margin audit over 160 formal datasets per holdout size.
The margin is measured with an independent oracle proxy.}
\label{tab:completion-margin-audit}
\begin{tabular}{rrrrr}
\toprule
$n$ & median $\kappa$ & median $2r/\kappa$ & $2r<\kappa$ & unique cert. \\
\midrule
500 & 1.16e-05 & 809.8 & 1.2\% & 2.5\% \\
1,000 & 1.16e-05 & 522.3 & 1.9\% & 2.5\% \\
5,000 & 1.16e-05 & 204.0 & 3.1\% & 4.4\% \\
20,000 & 1.16e-05 & 95.9 & 5.0\% & 7.5\% \\
\bottomrule
\end{tabular}
\end{table}

\begin{figure}[H]
    \centering
    \includegraphics[width=0.82\textwidth]{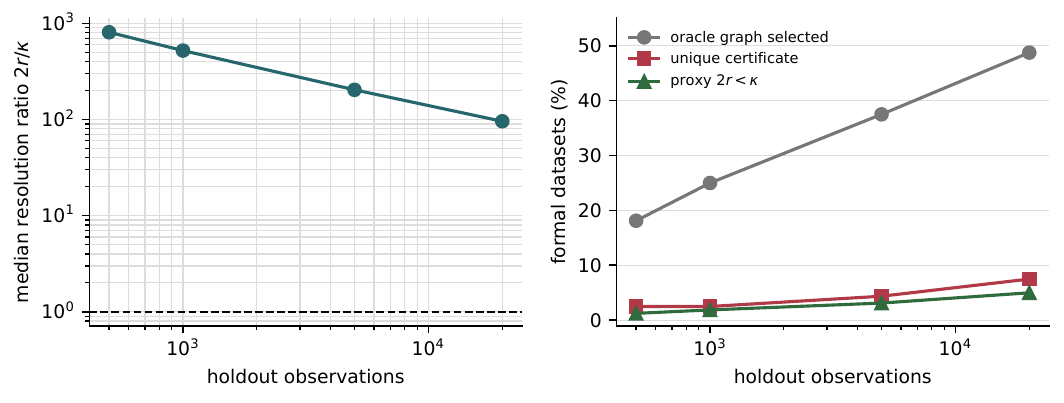}
    \caption{Why unique certificates remain rare.  The median confidence
    radius is far larger than the oracle-proxy exchange margin (left), even as
    graph selection and certification improve with holdout size (right).}
    \label{fig:completion-margin-audit}
\end{figure}

All 640 regret bounds cover their independent oracle audit.  The median
exchange margin is \(1.16\times10^{-5}\) at every nested holdout size.  At
\(n=20{,}000\), the empirical optimizer selects the oracle deletion set in
\(48.75\%\) of runs, but only \(7.5\%\) are uniquely certified.  The median
resolution ratio remains \(95.9\), so the gap between selection and proof is
consistent with near-tied transversals.  Because the oracle margin is itself
estimated, the column \(2r<\kappa\) is a diagnostic rather than a certificate.

\subsection{Selective output and score dependence}

The selective audit uses the same simultaneous empirical-Bernstein intervals
as the regret experiment.  It records the fraction of active edges with a
common label across \(\mathcal A_\alpha\), along with singleton and
oracle-proxy checks.

\begin{table}[H]
\centering
\small
\caption{Selective repair over 160 formal datasets per holdout size under the
one-second main protocol.  ``Statistical'' means that an optimally solved
forced-opposite query remains inside the confidence set; ``computational''
means that the query timed out.  ``Proxy retained'' checks the independent
oracle-proxy optimum.}
\label{tab:selective-repair-audit}
\begin{tabular}{rrrrrr}
\toprule
$n$ & decided & statistical & computational & singleton & proxy retained \\
\midrule
500 & 3.8\% & 96.2\% & 0.0\% & 2.5\% & 100.0\% \\
1,000 & 6.0\% & 94.0\% & 0.0\% & 2.5\% & 100.0\% \\
5,000 & 16.5\% & 83.5\% & 0.0\% & 4.4\% & 100.0\% \\
20,000 & 36.8\% & 63.2\% & 0.0\% & 7.5\% & 100.0\% \\
\bottomrule
\end{tabular}
\end{table}

\begin{figure}[H]
    \centering
    \includegraphics[width=0.82\textwidth]{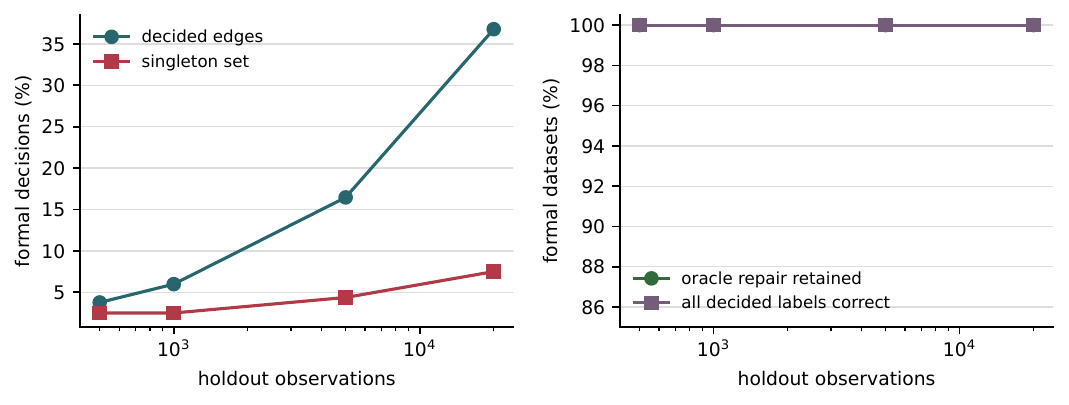}
    \caption{Selective rather than forced output.  More edges become decidable
    as the holdout grows (left); the independent oracle audit checks set and
    label coverage (right).}
    \label{fig:selective-repair-audit}
\end{figure}

\begin{table}[H]
\centering
\small
\caption{Compute-budget sensitivity at \(n=20{,}000\) over the same 160
datasets.  Solved queries are reused as the cap increases; only previous
timeouts are retried.}
\label{tab:selective-timeout-sensitivity}
\begin{tabular}{rrrrr}
\toprule
cap & decided & statistical & timeout & singleton \\
\midrule
0.01s & 34.6\% & 48.0\% & 17.3\% & 7.5\% \\
0.1s & 36.8\% & 62.9\% & 0.3\% & 7.5\% \\
1s & 36.8\% & 63.2\% & 0.0\% & 7.5\% \\
10s & 36.8\% & 63.2\% & 0.0\% & 7.5\% \\
\bottomrule
\end{tabular}
\end{table}

\begin{figure}[H]
    \centering
    \includegraphics[width=0.82\textwidth]{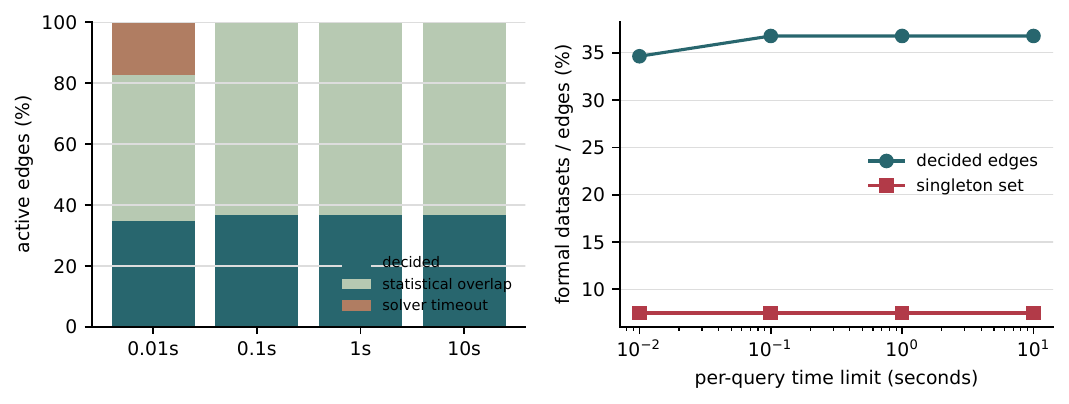}
    \caption{Statistical and computational abstention are reported separately.
    A 0.01-second cap leaves a visible timeout component; one second eliminates
    it on the formal suite.}
    \label{fig:selective-timeout-sensitivity}
\end{figure}

\subsection{Scaling audit for exact and relaxed queries}

The \(d=20\) main certificate suite was chosen to keep the statistical audit
from being obscured by a failed combinatorial solve.  We separately test the
computational boundary on \(d\in\{20,40,60,80,100\}\).  The frozen grid has
ER2 and SF2 graphs, standard and mixed-strength Gaussian SEMs, three seeds,
500 training observations, and 20,000 independent holdout observations:
60 datasets in total.  A Lasso screen fixes the active directed support.  A
deterministic greedy repair supplies the feasible \(D_0\), which is permitted
by Proposition~\ref{prop:relaxed-selective}; the experiment does not require
an exact solve to enter the confidence set.

The uncapped LP relaxation audits every active edge.  For comparison, each dataset has
20 evenly spaced active-edge indices fixed before any score is evaluated.
Their forced-opposite integer queries receive one second each.  Ground truth
is not used by either path.

\begin{table}[H]
\centering
\small
\caption{Scaling audit.  Percentages in the upper panel are edge- or
query-weighted within dimension.  The lower panel separates solved
statistical overlap from exact-solver timeout by candidate-support size.
LP sec. is the median time to audit every active edge.}
\label{tab:relaxed-certificate-scaling}
\begin{tabular}{rrrrrrr}
\toprule
$d$ & sets & edges & LP decided & LP sec. & exact decided & timeout \\
\midrule
20 & 12 & 85.7 & 6.7\% & 0.32 & 15.9\% & 0.0\% \\
40 & 12 & 101.5 & 5.9\% & 0.38 & 9.2\% & 0.0\% \\
60 & 12 & 203.8 & 1.3\% & 0.68 & 7.3\% & 16.7\% \\
80 & 12 & 122.2 & 5.5\% & 0.53 & 22.7\% & 0.0\% \\
100 & 12 & 199.1 & 0.7\% & 0.74 & 0.8\% & 0.0\% \\
\bottomrule
\end{tabular}
\vspace{3pt}
\begin{tabular}{lrrrrr}
\toprule
candidate edges & sets & LP sec. & exact decided & statistical & timeout \\
\midrule
$\leq100$ & 23 & 0.17 & 26.7\% & 73.3\% & 0.0\% \\
101--200 & 26 & 0.58 & 2.7\% & 97.3\% & 0.0\% \\
201--400 & 9 & 2.03 & 0.0\% & 100.0\% & 0.0\% \\
$>400$ & 2 & 16.13 & 0.0\% & 2.5\% & 97.5\% \\
\bottomrule
\end{tabular}
\end{table}

\begin{figure}[H]
    \centering
    \includegraphics[width=0.92\textwidth]{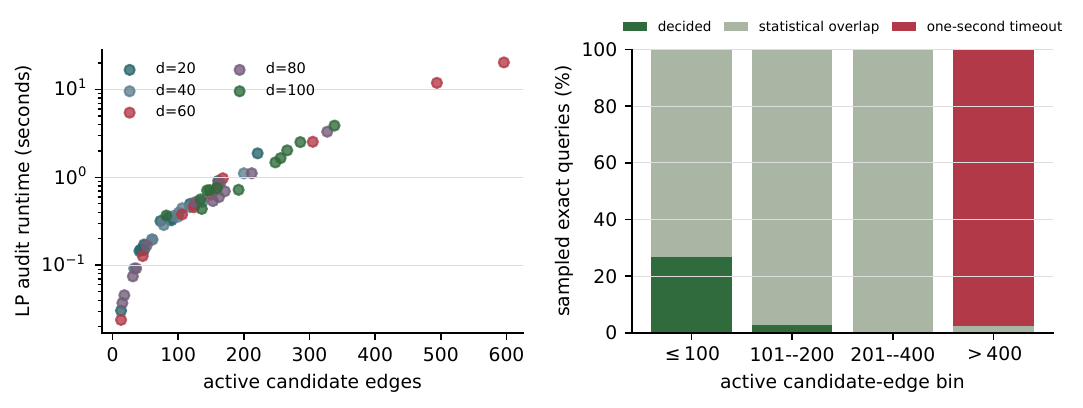}
    \caption{Computational audit.  The relaxed path audits every active edge
    (left).  Exact-query outcomes (right) show statistical overlap up to 400
    candidate edges; the 39 one-second timeouts all occur in the two supports
    above 490 edges.}
    \label{fig:relaxed-certificate-scaling}
\end{figure}

\subsection{Frontend-agnostic certificate audit}
\label{app:frontend-certificate}

The certificate is conditional on a fixed candidate support and score; it
does not require that the support come from continuous optimization.  We test
this point directly using NOTEARS, DAGMA, SDCD, and GOLEM as four interchangeable
screens.  Each frontend contributes its top \(2d\) off-diagonal magnitudes
with deterministic ties at \(d=20\).  The downstream procedure is otherwise
identical: regressors and clipping are frozen on \(n=500\) training samples,
simultaneous empirical-Bernstein intervals use an independent holdout of
5,000 samples, and a separate 20,000-sample draw is consulted only after all
labels have been fixed.  The grid contains ER2/ER4/SF2/SF4, Gaussian,
non-Gaussian, mixed-strength, and heteroscedastic SEMs, and five seeds, for
320 frontend--dataset pairs.

\begin{table}[H]
    \centering
    \small
    \caption{The same selective certificate after four frontend screens.
    Coverage refers to simultaneous coverage of all screened deletion costs;
    errors compare LP-certified labels with the independent oracle proxy.}
    \label{tab:frontend-certificate}
\begin{tabular}{lrrrr}
\toprule
Frontend & Screen recall & Labels (\%) & Coverage (\%) & Errors \\
\midrule
NOTEARS & 67.3 & 29.5 & 100.0 & 0 \\
DAGMA & 95.3 & 51.2 & 98.8 & 0 \\
SDCD & 63.1 & 13.2 & 100.0 & 0 \\
GOLEM & 94.8 & 46.1 & 98.8 & 0 \\
\bottomrule
\end{tabular}
\end{table}

\begin{figure}[H]
    \centering
    \includegraphics[width=0.86\textwidth]{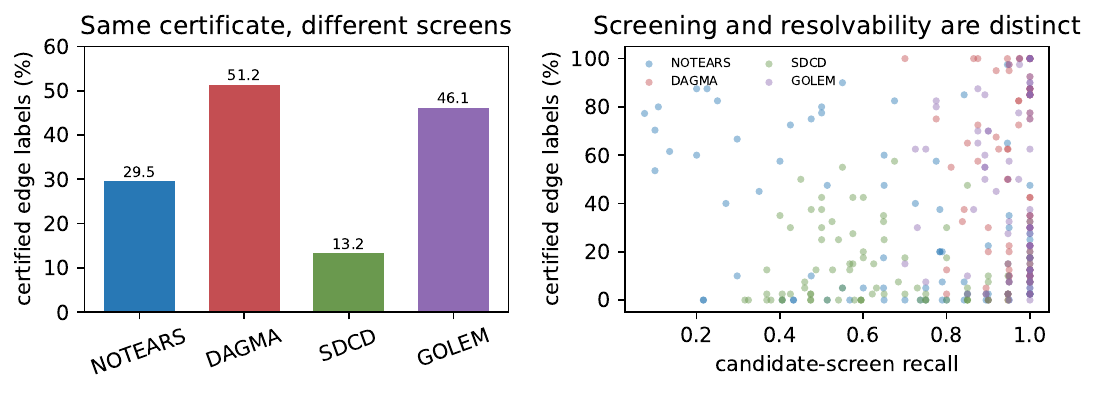}
    \caption{Frontend-agnostic audit.  The common certificate resolves
    different fractions of the four screened supports (left).  Candidate
    recall and selective resolvability are distinct: a broad screen may retain
    true edges without creating a large holdout margin between feasible
    repairs (right).}
    \label{fig:frontend-certificate}
\end{figure}

Regret coverage is \(100\%\) for every frontend and no LP-certified label
disagrees with the oracle proxy.  Simultaneous interval coverage is
\(98.75\%\) for DAGMA and GOLEM and \(100\%\) for NOTEARS and SDCD.  The
decided fraction is \(51.2\%\), \(46.1\%\), \(29.5\%\), and \(13.3\%\) for
DAGMA, GOLEM, NOTEARS, and SDCD, respectively.  These numbers are not an
end-to-end causal ranking: the top-\(2d\) screen itself has different recall.
They isolate the claim needed here---the certificate applies after different
frontends and exposes how much of each frontend's proposed support the frozen
score can actually distinguish.

The relaxed labels are conservative by construction.  On the 1,179 sampled
integer queries, they also coincide empirically with all 131 labels certified
by completed MILP solves and never certify an edge left unresolved by a
completed solve.  Of the
remaining queries, 1,009 are statistically unresolved and 39 time out.  The
LP path was not assigned a timeout and completed every audit; its median
all-edge runtime ranges from 0.32 to 0.74
seconds across the five dimensions, although the two supports above 490 edges
take 11.90 and 20.35 seconds.  At \(d=100\), only \(0.7\%\) of active edges
are certified.  The relaxation therefore addresses computation, not weak
statistical separation.  In this suite, active-support density is a better
warning sign for exact-query failure than ambient dimension alone; this is an
empirical boundary, not an average-case complexity claim.

The exchange margin is not a property of the causal graph alone.  It is
indexed by the candidate support, deletion score, and fitting protocol.  We
therefore repeat the 160-dataset oracle audit with four dimensionless scores:
the frozen clipped score used by the certificate, the same frozen regressors
without clipping, oracle-refitted normalized quadratic loss, and
oracle-refitted Gaussian profile loss.  The oracle-refitted variants are
diagnostics and are not used for selection.

\begin{table}[H]
\centering
\small
\caption{Score sensitivity on the same 160 frozen Lasso supports.  Relative
\(\kappa\) divides by the median absolute edge cost.}
\label{tab:score-margin-sensitivity}
\begin{tabular}{lrrrr}
\toprule
Score & median $\kappa$ & relative $\kappa$ & zero margin & SHD \\
\midrule
Frozen clipped & 1.16e-05 & 4.31e-03 & 0.0\% & 34.90 \\
Frozen raw & 2.09e-05 & 8.38e-04 & 0.0\% & 37.33 \\
Refit quadratic & 2.94e-07 & 2.07e-05 & 15.6\% & 42.49 \\
Refit profile & 7.88e-08 & 1.92e-06 & 43.8\% & 38.79 \\
\bottomrule
\end{tabular}
\end{table}

\begin{figure}[H]
    \centering
    \includegraphics[width=0.82\textwidth]{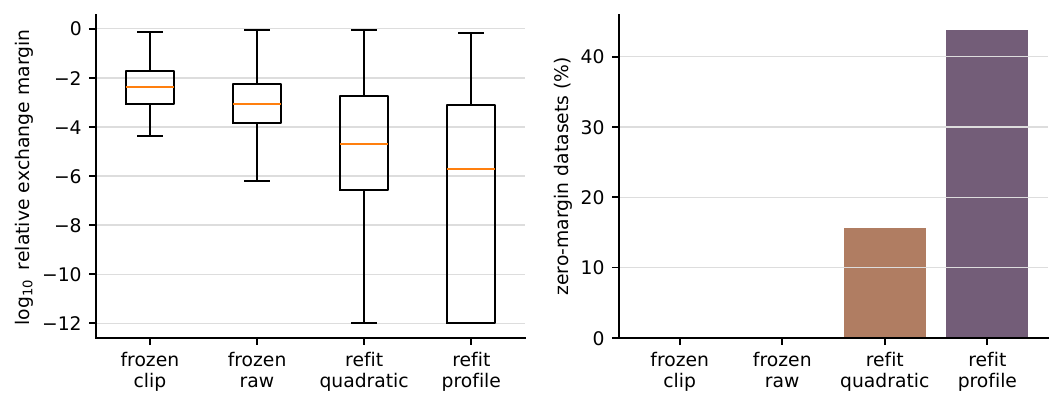}
    \caption{The completion margin is score-conditional.  Replacing clipping,
    refitting regressors, or changing the loss moves the margin distribution
    and can create exact ties.}
    \label{fig:score-margin-sensitivity}
\end{figure}

The one-second protocol has no query timeouts.  It decides progressively more
labels without requiring a unique graph: the mean decided fraction rises from
\(3.8\%\) to \(36.8\%\), while the singleton rate rises from \(2.5\%\) to
\(7.5\%\).  All 640 oracle-proxy repairs remain in the reported sets, and every
reported label agrees with the corresponding proxy.  At \(n=20{,}000\), a
0.01-second cap classifies \(34.6\%\) of edges, leaves \(48.0\%\) statistically
unresolved after completed solves, and times out on \(17.3\%\).  At 0.1
seconds, the timeout fraction falls to \(0.3\%\); at one second it is zero,
with \(36.8\%\) decided and \(63.2\%\) statistically unresolved.  The low
singleton rate therefore persists after computational abstention is removed.

The median absolute margin ranges from \(7.88\times10^{-8}\) for refitted
profile loss to \(2.09\times10^{-5}\) for the frozen unclipped score.  Exact
proxy ties occur in \(43.75\%\) of profile-score datasets and \(15.63\%\) of
refitted-quadratic datasets, but not for either frozen score.  The associated
SHD does not improve monotonically with the margin.  Thus the small
\(1.16\times10^{-5}\) value in the primary audit diagnoses one frozen
score/frontend pair; it is not evidence for an intrinsic sample complexity of
causal discovery.

\section{Interaction-Aware Certificate: Proof and Unified Audit}
\label{app:global-certificate}

An edge-modular deletion model does not represent interactions created by
refitting several parents jointly.  We therefore treat the entire parent set
as the local statistical unit.  This retains collider-forming interactions
and supports selective statements about skeleton and V-structure features;
the resulting exact search is combinatorial.

For each compatible DAG \(G\), define
\begin{equation}
 Q(G)=\sum_j(R_{j,\operatorname{pa}_G(j)}+a_{j,\operatorname{pa}_G(j)}),
 \qquad G_0\in\mathcal G(\mathcal P).
\label{eq:global-score}
\end{equation}
The reference $G_0$ is arbitrary and may be data-dependent.  The benchmark
chooses the empirical holdout minimizer, but the theorem also covers a
compatible graph supplied by a continuous, constraint-based, or score-based
frontend.

For completeness, we record the data-dependent quantities suppressed in the
main statement.  Write
\(R_{j,S}=\mathbb E\ell_{j,S}(X)\),
\(\widehat R_{j,S}=n^{-1}\sum_i\ell_{j,S}(X_i)\), and
\(M=\sum_j|\mathcal P_j|(|\mathcal P_j|-1)\).  For each ordered pair
\(S,T\in\mathcal P_j\), let \(\widehat V_{j,S,T}\) be the sample variance of
\(\ell_{j,S}(X_i)-\ell_{j,T}(X_i)\), and set
\begin{equation}
r_{j,S,T}=\sqrt{\frac{2\widehat V_{j,S,T}\log(2M/\alpha)}{n}}
+\frac{28B\log(2M/\alpha)}{3(n-1)}.
\label{eq:global-radius}
\end{equation}
For \(S_{0j}=\operatorname{pa}_{G_0}(j)\), define
\begin{equation}
g_{j,S_{0j}}=0,\qquad
g_{j,S}=\widehat R_{j,S}-\widehat R_{j,S_{0j}}
+a_{j,S}-a_{j,S_{0j}}-r_{j,S,S_{0j}}
\quad(S\ne S_{0j}).
\label{eq:robust-local-gap}
\end{equation}

The main text defines \(L\), \(L_-\),
\(\mathcal A_\alpha(G_0)\), and the forced-opposite value
\(L_\varphi^{\rm opp}\).  We now prove the simultaneous regret, confidence
family, and feature statements in Theorem~\ref{thm:global-certificate}.

\subsection{Proof of Theorem~\ref{thm:global-certificate}}

Condition on the training sample.  The parent-set families, fitted predictors,
losses, and penalties are then fixed, while the validation observations remain
i.i.d.  For an ordered pair $S,T\in\mathcal P_j$, put
\[
D_i^{j,S,T}=\ell_{j,S}(X_i)-\ell_{j,T}(X_i).
\]
Because each loss lies in $[-B,B]$, the paired difference lies in
$[-2B,2B]$, an interval of length $4B$.  The empirical Bernstein inequality
of \citet{maurer2009empirical}, with failure probability $\alpha/M$, gives
\begin{equation}
R_{j,S}-R_{j,T}
\ge \widehat R_{j,S}-\widehat R_{j,T}-r_{j,S,T}
\label{eq:global-pair-event}
\end{equation}
with the radius in Eq.~\eqref{eq:global-radius}.  A union bound over all $M$
ordered local pairs makes Eq.~\eqref{eq:global-pair-event} simultaneous with
probability at least $1-\alpha$.  Covering the entire ordered-pair family is
essential when $G_0$ depends on the same holdout.

Work on this simultaneous event and set $T=S_{0j}$.  Adding the
deterministic penalty difference to Eq.~\eqref{eq:global-pair-event} yields
\[
R_{j,S}+a_{j,S}-R_{j,S_{0j}}-a_{j,S_{0j}}
\ge g_{j,S}
\]
for every local alternative, including equality at $S=S_{0j}$.  Hence,
for every candidate DAG $G$,
\begin{equation}
Q(G)-Q(G_0)
\ge \sum_j g_{j,\operatorname{pa}_G(j)}.
\label{eq:global-dag-lower}
\end{equation}
Minimizing both sides over $G\in\mathcal G(\mathcal P)$ gives
$\min_GQ(G)-Q(G_0)\ge L$.  Rearrangement proves
Eq.~\eqref{eq:global-regret-certificate}.  Notice that $L\le0$ because
$G_0$ is feasible and has robust relative cost zero.

Let $G^*$ be any population minimizer.  Since
$Q(G^*)-Q(G_0)\le0$, Eq.~\eqref{eq:global-dag-lower} gives
$\sum_jg_{j,\operatorname{pa}_{G^*}(j)}\le0$.  Hence every such $G^*$ lies
in Eq.~\eqref{eq:parent-set-confidence-family}.  If $L_->0$,
Eq.~\eqref{eq:global-dag-lower} is strictly positive for every
$G\ne G_0$; the reference graph is therefore the unique minimizer of
$Q$.  Finally, if a population minimizer had the feature label opposite to
$G_0$, its robust relative cost would be at least
$L_\varphi^{\rm opp}>0$, contradicting its membership in
$\mathcal A_\alpha(G_0)$.  This proves the feature statement.

\subsection{Optimization and target of the certificate}

For each child, the implementation screens at most eight candidate parents on
the training sample and fits every subset of those parents.  The validation
loss for $S\in\mathcal P_j$ is
\begin{equation}
\ell_{j,S}(x)=\operatorname{clip}_{[-B,B]}\!\left\{
\frac{(x_j-\widehat f_{j,S}(x))^2
      -(x_j-\widehat f_{j,F_j}(x))^2}{2\widehat v_j}\right\},
\label{eq:parent-set-loss}
\end{equation}
where $F_j$ is the full screened parent set.  We use
$a_{j,S}=\lambda_{\mathrm{sp}}|S|$.  The subtraction of the full-model loss
does not affect the graph minimizer but reduces paired variance, and division
by the training response variance makes the loss dimensionless.

Both minimizations in Eq.~\eqref{eq:robust-dag-oracles} are standard exact
Bayesian-network structure-learning problems over a restricted parent-set
family.  We use the parent-set integer program and add cluster inequalities
until the selected support is acyclic \citep{cussens2011cutting}.  A no-good
constraint excludes $G_0$ when computing $L_-$.  These optimization
devices, including parent-set screening and cluster cuts, are established
tools \citep{cussens2020gobnilp}; they are not claimed as contributions.  The
new object is the robust objective in Eq.~\eqref{eq:robust-local-gap} and its
conversion of simultaneous local score uncertainty into the global statement
of Theorem~\ref{thm:global-certificate}.

The feature queries use the same parent-set variables.  If $x_{j,S}$ selects
parent set $S$ for child $j$, write
\begin{equation}
a_{uv}=\sum_{S\in\mathcal P_v:u\in S}x_{v,S},\qquad
s_{uv}=a_{uv}+a_{vu}\quad(u<v).
\label{eq:parent-set-feature-variables}
\end{equation}
On a DAG, $s_{uv}$ is the skeleton-adjacency indicator.  Thus an opposite
skeleton or directed-edge label is one linear equality.  An unshielded
collider $u\to v\leftarrow w$ is present exactly when
$a_{uv}=a_{wv}=1$ and $s_{uw}=0$.  Its presence is imposed by these three
equalities; its absence is the single inequality
\begin{equation}
a_{uv}+a_{wv}-s_{uw}\le1.
\label{eq:collider-opposite-row}
\end{equation}
One forced-opposite solve therefore computes
Eq.~\eqref{eq:forced-opposite-feature}.  Skeletons and unshielded colliders
characterize a Markov equivalence class, so this projection does not require
the data to resolve arbitrary orientations inside the class
\citep{chickering2002optimal}.

The target is deliberately conditional.  Screening can omit a true parent,
and even the population minimizer of the clipped regularized score can differ
from the causal DAG.  The theorem certifies population regret only within
$\mathcal G(\mathcal P)$.  It neither assumes nor concludes causal support
recovery.

\subsection{Frozen protocol and results}

The formal suite uses $d=20$ and 500 training observations.  A standardized
all-node Lasso frontend uses $0.002\lambda_{\max}$; at most eight incoming
parents per node are retained, and all of their subsets are enumerated.  We
set $B=0.02$ and $\lambda_{\mathrm{sp}}=0.001$ using disjoint development
seeds 100--102, then freeze them.  Formal seeds are 0--9.  We test ER2, ER4,
SF2, and SF4 graphs under standard Gaussian, exponential-noise, 40\%-weak-edge
Gaussian, and heteroscedastic Gaussian SEMs.  Independent holdouts contain
$500,1000,5000$, or $20{,}000$ observations; a further 20,000-sample oracle is
used only after selection and certification.  This gives 160 datasets at each
holdout size and 640 records in total.

\begin{figure}[H]
    \centering
    \includegraphics[width=0.92\textwidth]{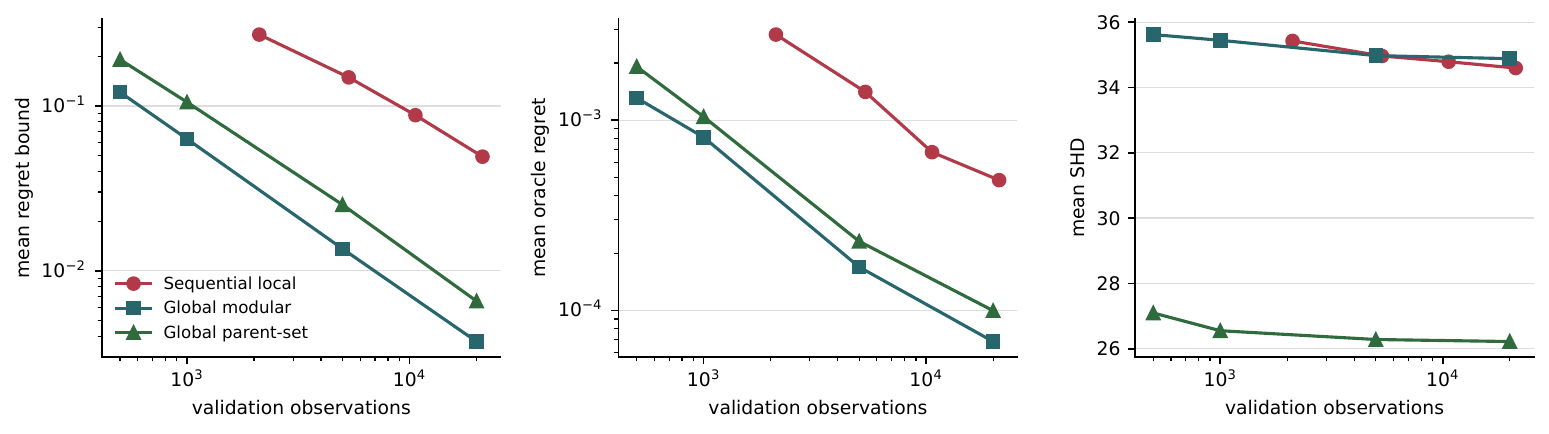}
    \caption{Matched-budget comparison of the original sequential certificate,
    a globally optimized modular deletion score, and the parent-set certificate.
    The methods certify different population losses, so regret values should be
    compared as tightening curves rather than as a common performance metric.
    Every reported bound covered its independent oracle audit.}
    \label{fig:global-certificate}
\end{figure}

\begin{table}[t]
\centering
\small
\setlength{\tabcolsep}{3.4pt}
\caption{Matched validation-budget audit over 160 formal datasets. The three rows certify different targets; their regret values are not losses on a common scale.}
\label{tab:global-certificate}
\begin{tabular}{@{}lrrrrr@{}}
\toprule
Method & Budget & Bound & Oracle reg. & Exact path/graph & SHD \\
\midrule
Sequential local & 21,275 & 0.0492 & 0.00048 & 0.0\% & 34.59 \\
Global modular & 20,000 & 0.0037 & 0.00007 & 7.5\% & 34.88 \\
Global parent-set & 20,000 & 0.0065 & 0.00010 & 0.0\% & 26.22 \\
\bottomrule
\end{tabular}
\end{table}

For the parent-set target, all 640 bounds cover the independent oracle proxy.
The mean bound decreases from $0.1911$ at $n=500$ to $0.00655$ at
$n=20{,}000$, while the mean oracle regret decreases from $0.00191$ to
$0.00010$.  The selected graph matches the oracle graph in $3.75\%$ and
$47.5\%$ of runs, respectively, but $L_-$ is never positive.  Thus finite
samples often select the oracle graph before the simultaneous confidence box
is narrow enough to prove uniqueness.

The parent-set construction retains deletion interactions and reduces mean
SHD from $34.88$ for the modular certificate to $26.22$ at the largest matched
budget.  This remains far behind the strongest frozen frontends in the broad
benchmark, so the experiment supports the certificate rather than a recovery
claim.  Mean screen recall is $0.692$, and the difficult ER4/SF4 cells account
for most remaining errors.  The result also identifies the operative boundary:
no validation certificate can recover parents removed by the training screen.

The modular companion fixes one bounded loss per active-edge deletion and
solves an exact weighted feedback-edge problem.  At $n=20{,}000$, its mean
bound is $0.00372$, its oracle regret is $0.000069$, and $7.5\%$ of graphs are
uniquely certified.  Its empirical modular objective is lower than that of a
same-cost greedy cycle repair in every formal run.  This verifies the value of
global optimization for its stated target, while the parent-set result shows
why an additive single-edge target is statistically incomplete.

\subsection{Unified frontend-to-feature audit}

We next keep the full pipeline fixed within one experiment.  NOTEARS, DAGMA,
SDCD, and GOLEM each supply their top-$2d$ weighted entries.  A deterministic
four-parent-per-child cap fixes the parent-set family before holdout scores are
evaluated; every subset is refitted on 500 training observations.  The
interaction-aware score uses a 5,000-sample holdout, and an independent
20,000-sample draw is consulted only after the graph and feature labels have
been fixed.  The grid contains ER2/ER4/SF2/SF4, Gaussian, non-Gaussian,
mixed-strength, and heteroscedastic SEMs, and five seeds, for 320
frontend--dataset pairs.  The cap, clipping level, and sparsity penalty are
fixed across all cells.

\begin{table}[H]
\centering
\small
\caption{Unified pipeline audit.  ``Front.'' is the frontend's standard
reported DAG; ``Modular'' and ``Parent'' apply the two repair scores to the
same top-$2d$ screen.  ``Oracle parent'' uses 20,000 independent observations;
label columns are certified candidate-feature fractions.}
\label{tab:interaction-aware-benchmark}
\begin{tabular}{lrrrrrrr}
\toprule
Frontend & Screen & Front. & Modular & Parent & Oracle & Adj. & Collider \\
 & recall & SHD & SHD & SHD & parent SHD & labels & labels \\
\midrule
NOTEARS & 57.4 & 15.6 & 19.9 & 20.4 & 20.1 & 19.5 & 16.8 \\
DAGMA & 77.6 & 4.5 & 3.6 & 8.5 & 8.6 & 44.5 & 34.6 \\
SDCD & 53.4 & 39.3 & 24.1 & 21.1 & 21.0 & 10.5 & 4.4 \\
GOLEM & 77.2 & 8.1 & 4.0 & 8.5 & 8.5 & 44.4 & 36.3 \\
\bottomrule
\end{tabular}
\end{table}

\begin{figure}[H]
\centering
\includegraphics[width=0.88\textwidth]{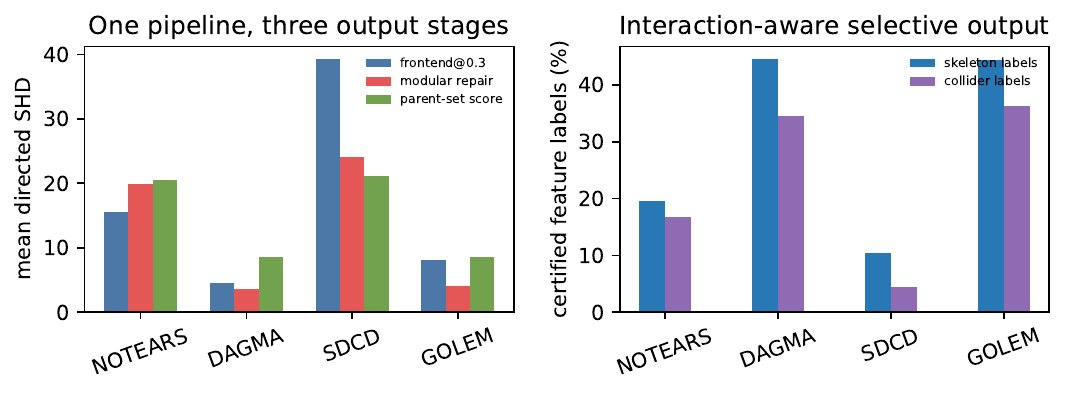}
\caption{One frozen pipeline, reported end to end.  Joint parent-set refitting
does not uniformly improve SHD over modular repair (left), but it changes the
statistical target and permits selective skeleton and collider labels (right).}
\label{fig:interaction-aware-benchmark}
\end{figure}

All 320 regret bounds cover the independent oracle-score audit.  Among 3,042
certified skeleton labels and 2,396 certified collider labels, none disagrees
with the corresponding oracle-score optimum.  Decision rates vary with the
frontend screen: adjacency labels range from $10.5\%$ to $44.5\%$, and
collider labels from $4.4\%$ to $36.3\%$.  The parent-set score lowers mean SHD
relative to the reported frontend for SDCD, but trails modular repair for
DAGMA and GOLEM.  Its mean SHD is $14.63$, nearly the $14.56$ obtained by
replacing the holdout score with a 20,000-sample oracle score.  Across all 320
cells, parent-set SHD and screen recall have Pearson correlation $-0.886$.
The deterioration is therefore not a finite-holdout effect that tighter
intervals would repair; it reflects the screened family and the score's
population target.  Averaged over frontends, modular and reported frontend
SHDs are $12.91$ and $16.87$.

The score certificate is not a causal-truth certificate.  Final truth-only
evaluation finds disagreement on $4.4\%$ of certified skeleton labels and
$5.5\%$ of certified collider labels.  Those errors can arise from candidate
screen omissions, the finite oracle proxy, or a population score whose
minimizer differs from the generating DAG.  Reporting both comparisons keeps
score uncertainty separate from causal identification.

\subsection{Screen-budget diagnosis}

The top-$2d$ screen in the primary audit is a protocol choice, not an oracle
density estimate.  We rerun all 320 frontend--dataset pairs with predeclared
top-$2d$, top-$4d$, and top-$6d$ budgets.  Every frontend fit, data split,
four-parent cap, loss, and penalty is unchanged; no budget is selected using
truth.  ``Raw recall'' is measured before the per-child cap, whereas ``family
recall'' describes the parent sets actually passed to the certificate.

\begin{table}[H]
\centering
\small
\caption{Frozen screen-budget audit over 960 cells.  Coverage compares the
holdout regret bound with the independent 20,000-sample score optimum.}
\label{tab:interaction-screen-budget}
\begin{tabular}{lrrrrrr}
\toprule
Budget & Raw recall & Family recall & Parent SHD & Oracle SHD & Coverage & Time (s) \\
\midrule
top-$2d$ & 80.1 & 66.4 & 14.63 & 14.56 & 100.0 & 0.009 \\
top-$4d$ & 84.5 & 68.4 & 15.40 & 15.41 & 100.0 & 0.028 \\
top-$6d$ & 85.4 & 68.5 & 15.73 & 15.67 & 100.0 & 0.058 \\
\bottomrule
\end{tabular}
\end{table}

The extra candidates raise raw recall by 5.3 percentage points, but the local
cap passes only 2.1 points to the parent-set family.  More importantly, the
oracle-score solution does not improve: its mean SHD rises from 14.56 to
15.67.  The effect is frontend-dependent.  At top-$4d$, oracle-parent SHD is
8.54 for DAGMA and 8.18 for GOLEM, but 21.65 for NOTEARS and 23.28 for SDCD.
Thus the top-$2d$ result is not explained by screen omissions alone.  Enlarging
the family recovers some omitted edges while admitting alternatives favored by
the frozen score but not by the generating DAG.

\subsection{Bounded-parent scaling}

The unified accuracy audit uses $d=20$ so that statistical and score-target
effects are not mixed with solver failures.  We separately test whether the
interaction-aware implementation itself stops at that dimension.  A
truth-blind absolute-correlation screen supplies at most two candidate parents
per node.  For each instance we solve the global reference certificate and 20
feature queries fixed by index before scores are evaluated: ten skeleton
adjacencies and ten collider candidates.  The grid contains ER2 and SF4,
standard and mixed-strength Gaussian SEMs, and five seeds, giving 20 cases at
each dimension.

\begin{table}[H]
\centering
\small
\caption{Interaction-aware computational scaling.  Times are seconds for the
global certificate (core) and the batch of 20 forced-opposite queries.}
\label{tab:parent-set-scaling}
\begin{tabular}{rrrrrr}
\toprule
$d$ & Cases & Parent sets & Queries & Core med./p95 & Query med./p95 \\
\midrule
20 & 20 & 80 & 20 & 0.051/0.069 & 0.343/0.472 \\
50 & 20 & 200 & 20 & 0.175/0.241 & 1.139/1.674 \\
100 & 20 & 400 & 20 & 0.533/0.742 & 3.675/5.226 \\
200 & 20 & 800 & 20 & 1.800/2.433 & 12.798/17.487 \\
\bottomrule
\end{tabular}
\end{table}

All 80 core solves and all 1,600 feature queries terminate.  At $d=200$, the
median core time is 1.80 seconds and the median time for 20 queries is 12.80
seconds.  This is an empirical bounded-parent result, not a polynomial-time
guarantee: runtimes grow superlinearly, two-parent families capture only local
pair interactions, and denser candidate families remain subject to the
worst-case complexity of exact Bayesian-network structure learning.  The
edge-modular LP in Appendix~\ref{app:completion-margin} is the scalable
conservative option when that restriction is unacceptable.

\end{document}